\documentclass{article}

\usepackage{microtype}
\usepackage{graphicx}
\usepackage{subfigure}
\usepackage{booktabs} 
\usepackage{hyperref}
\usepackage{multirow} 

\usepackage[accepted]{icml2024}

\usepackage{amsmath}
\usepackage{amssymb}
\usepackage{mathtools}
\usepackage{amsthm}

\usepackage[capitalize,noabbrev]{cleveref}

\def\argmin{\mathop{\rm argmin}}

\def\arg{\mathop{\rm arg}}

\def\bfX{\boldsymbol{X}}

\newcommand{\RNum}[1]{\uppercase\expandafter{\romannumeral #1\relax}}

\theoremstyle{plain}
\newtheorem{theorem}{Theorem}[section]
\newtheorem{proposition}[theorem]{Proposition}
\newtheorem{lemma}[theorem]{Lemma}

\theoremstyle{definition}

\newtheorem{assumption}[theorem]{Assumption}

\theoremstyle{remark}
\newtheorem{remark}[theorem]{Remark}
\usepackage[textsize=tiny]{todonotes}

\icmltitlerunning{Distributed High-Dimensional Quantile Regression: Estimation Efficiency and Support Recovery}

\begin{document}

\twocolumn[
\icmltitle{Distributed High-Dimensional Quantile Regression: Estimation Efficiency and Support Recovery}


\icmlsetsymbol{equal}{*}

\begin{icmlauthorlist}
\icmlauthor{Caixing Wang}{sch}
\icmlauthor{Ziliang Shen}{sch}
\end{icmlauthorlist}

\icmlaffiliation{sch}{School of Statistics and Management, Shanghai University of Finance and Economics}

\icmlcorrespondingauthor{Caixing Wang}{wang.caixing@stu.sufe.edu.cn}

\icmlkeywords{Distributed estimation, quantile regression, high-dimensional, support recovery.}
\vskip 0.3in 
]



\printAffiliationsAndNotice{}  

\begin{abstract}
    In this paper, we focus on distributed estimation and support recovery for high-dimensional linear quantile regression. 
    Quantile regression is a popular alternative tool to the least squares regression for robustness against outliers and data heterogeneity. 
    However, the non-smoothness of the check loss function poses big challenges to both computation and theory in the distributed setting.
    To tackle these problems, we transform the original quantile regression into the least-squares optimization. 
    By applying a double-smoothing approach, we extend a previous Newton-type distributed approach without the restrictive 
    independent assumption between the error term and covariates. 
    An efficient algorithm is developed, which enjoys high computation and 
    communication efficiency. Theoretically, the proposed distributed estimator achieves a near-oracle convergence rate and high support 
    recovery accuracy after a constant number of iterations. 
    Extensive experiments on synthetic examples and a real data application further demonstrate the effectiveness of the proposed method.
\end{abstract}

\section{Introduction}
\label{sec:intro}
With the development of modern technology, the proliferation of massive data has 
garnered significant attention from researchers and practitioners \citep{li2020federated,gao2022review}.
For example, financial institutions leverage big data to scrutinize customer 
preferences and fine-tune their marketing approaches, while manufacturers and 
transportation departments lean heavily on extensive datasets to streamline 
supply chain management and enhance delivery route optimization. 
However, these large-scale datasets are usually distributed across different 
machines due to storage and privacy concerns, making the direct application of existing statistical methods infeasible. On the other hand, another challenge arises from the high dimensionality of modern data. In the literature, a sparse assumption is often adopted \citep{zhao2006model,wainwright2009sharp,hastie2015statistical}, and support recovery is an essential problem for high-dimensional analysis. Despite its importance in practice, the support recovery in a distributed system is underexplored in theory, compared to the well-studied statistical estimation of the interested parameters \citep{zhang13,lee2017communication,jordan2018communication}.

Since the seminar work of \citet{koenker1978regression}, 
quantile regression has gained increasing attention across various fields, 
including economics, biomedicine, and environmental studies. Compared to the 
least squares regression that only estimates the conditional means, quantile 
regression models the entire conditional quantiles of the response, making it 
more robust against outliers in the response measurements \citep{koenker2005quantile,feng2023lack}. 
Although quantile regression can better handle data heterogeneity, computational challenges arise when both sample size and dimension are large due to the non-smooth check loss function \citep{he2021smoothed}.  Consequently, it is natural
to consider a distributed estimation procedure to address scalability concerns. 

In this paper, we focus on distributed estimation and support recovery for high-dimensional linear quantile regression. We aim to bridge the theoretical and practical gap by addressing several fundamental questions regarding distributed high-dimensional linear quantile regression. \textit{First, what is the statistical limit of estimation in the presence of distributed data? And how does this limit depend on the number of local machines in the distributed system? Second, can distributed high-dimensional linear quantile regression achieve the same convergence rate of the parameters and support recovery rate as those in a single machine setting? Third, can the proposed method be applied to various practical settings, such as homoscedastic and heteroscedastic data structures?}

We address the aforementioned theoretical and practical questions by designing a distributed high-dimensional linear quantile regression algorithm via a double-smoothing transformation. To our limited knowledge, our work is one of the pioneering works in studying distributed high-dimensional linear quantile regression with the least practical constraint and solid theoretical guarantees involving the estimation efficiency and support recovery. The specific contributions can be concluded as follows.

\textbf{Methodology novelty.}
Our paper extends the idea of \citet{chen2020distributed} resulting in a novel method for \textbf{D}istributed \textbf{H}igh-dimensional \textbf{S}parse \textbf{Q}uantile \textbf{R}egression (\textit{DHSQR}) without the stringent assumption that the error term is independent of the covariates.  Specifically, we start by transforming the covariate and the response, which recast the quantile regression into the least squares framework. Next, we introduce an iterative distributed algorithm based on an approximate Newton method that uses a double-smoothing approach applied to the global and local loss functions, respectively. In the distributed system, the local machines only need to broadcast the $(p+1)$-dimensional gradient vectors (instead of the $(p+1)\times (p+1)$ Hessian matrix), where $p$ is the dimension of the covariates. This optimization problem can be efficiently addressed on the central machine due to its simplified least squares formulation.

\textbf{Theoretical assessments.}
Theoretically, we not only establish the convergence rate of our DHSQR estimator in the $\ell_2$-norm (Theorems \ref{thm1} and \ref{thm2}), but also characterize the \textit{beta-min} condition for the exact support recovery (Theorems \ref{thm3} and \ref{thm4}) that is novel for distributed high-dimensional sparse estimation \citep{neykov2016l1}. After a constant number of iterations, the convergence rate and the \textit{beta-min} condition of the distributed estimator align with the classical theoretical results derived for a single machine setup \citep{wainwright2009sharp,belloni2011}. 

\textbf{Numerical verification.} 
Another contribution of this work is the comprehensive studies on the validity and effectiveness of the proposed algorithm in various synthetic and real-life examples, which further support the theoretical findings in this paper.

\subsection{Related Work}

\textbf{Distributed methods.}
Significant efforts have been dedicated to the development of distributed statistical learning methods, broadly categorized into two main streams.
The first class is known as the divide-and-conquer (DC) methods \citep{zhang13,li2013statistical,chen2014split,zhang2015divide}. 
These one-shot methods usually compute the relevant estimates based on local samples in the first step and then send these local
estimates to a central machine where the final estimate is obtained by simply averaging the local estimates. These methods offer 
computational efficiency with just one round of communication, 
but they have the theoretical constraint on the number of local machines to guarantee the global optimal rate 
\citep{huang2019distributed,gao2022review}. The second class comprises
multi-round distributed methods designed to improve estimation efficiency and
relax restrictions on the number of local machines \citep{shamir2014communication,lee2017communication,wang2017efficient}. 
 \citet{jordan2018communication} and \citet{fan2021communication} proposed a communication-efficient surrogate likelihood (CSL) framework that can be applied to low-dimensional estimation, high-dimensional regularized estimation, and Bayesian inference. Notably, the CSL method eliminates the need to transfer local Hessian matrices to the central machine, resulting in significantly reduced communication costs. It's worth noting that most of the aforementioned methods only focus on homogeneous data, which can be less practical in the context of big data analysis. 

\textbf{Distributed linear quantile regression.}
In the existing literature, distributed linear quantile regression has been widely investigated using the traditional divide-and-conquer method \citep{zhao2014general,xu2020block,chen2020quantile}. However, when dealing with high-dimensional settings, where sparsity assumptions are commonly applied, the DC estimator is no longer sparse due to de-biasing and averaging processes, resulting in poor support recovery \citep{bradic2017uniform,chen2020distributed}. In addition, their methods require a condition on the number of distributed machines to ensure the global convergence rate. To alleviate the restriction on the number of machines, \citet{chen2020distributed} transform the check loss to the square loss via a kernel smoothing approach and propose a Newton-type distributed estimator. The theoretical results offered insights into estimation errors and support recovery. Nevertheless, their method and theory require the error term to be independent of the covariates, which is not very common and hard to verify in practice.  Inspired by the ideas in \citet{jordan2018communication} and \citet{fernandes2021smoothing}, \citet{tan2022communication} studied a distributed convolution-type smoothing quantile regression whose loss function is twice continuously differentiable in both low-dimensional and high-dimensional regime. Note that most of the previous works mainly focus on the convergence rate of their respective estimator, while we further establish the distributed support recovery theory.
\subsection{Notations}
For two sequences $\{a_n\}$ and $\{b_n\}$, we denote $a_n \lesssim b_n$ if $a_n \leq C b_n$, where $C$ is a constant. And $a_n \asymp b_n$ if and only if $a_n \lesssim b_n$ and $b_n \lesssim a_n$. For a vector $\boldsymbol{u}=(u_1,\ldots,u_p)^{\mathrm{T}}$, we define  supp$(\boldsymbol{u})=\{j  : u_j \neq 0\}$. We use $|\cdot|_q$ to denote the $\ell_q$-norm in $\mathbb{R}_p$: $|\boldsymbol{u}|_q=(\sum_{i=1}^{p}|u_i|^q)^{1/q}$, for $1\leq q< \infty$ and $|\boldsymbol{u}|_{\infty}=\max_{1\leq i \leq p}|u_i|$. For $S  \subseteq\{1, \ldots, p\} $ with length $|S|$, let $\boldsymbol{u}_S=(u_i, i \in S) \subseteq \mathbb{R}^{|S|}$. For a matrix $\boldsymbol{A}=(a_{ij}) \in \mathbb{R}^{p \times q}$, define $|\boldsymbol{A}|_{\infty}=\max_{1 \leq i \leq p, 1 \leq j \leq q}|a_{ij}|$, $\|\boldsymbol{A}\|_{\infty}=\max_{1 \leq i \leq p}\sum_{1 \leq j \leq q}|a_{ij}|$, and $\|\boldsymbol{A}\|_{op}=\max_{|\boldsymbol{v}|_2=1}|\boldsymbol{A}\boldsymbol{v}|_2$. For two subsets $S_1 \in \{1,\ldots,p\} $ and $ S_2 \in \{1,\ldots,q\}$, we define the submatrix $\boldsymbol{A}_{S_1 \times S_2}=(a_{ij}, i \in S_1, j \in S_2)$, $\Lambda_{\max }(\boldsymbol{A})$ and $\Lambda_{\min }(\boldsymbol{A})$ to be the largest and smallest eigenvalues of $\boldsymbol{A}$, respectively. In this paper, we use $C$ to denote a universal constant that may vary from line to line.

\section{Methodology}\label{sec:metho}
In this section, we introduce the proposed distributed estimator. Inspired by the Newton-Raphson iteration, we construct a kernel-based estimator that establishes a connection between quantile regression and ordinary least squares regression in a single machine setting. Based on this estimator, we design a distributed algorithm to minimize a double-smoothing shifted loss function. 

\subsection{The Linear Quantile Model}
For a given quantile level $\tau \in(0,1)$, we consider to construct the conditional $\tau$-th quantile function $Q_\tau(Y|\boldsymbol{X})$  with a linear model:
$$
Q_\tau(Y|\boldsymbol{X}) = \boldsymbol{X}^{\mathrm{T}} \boldsymbol{\beta}^*(\tau) = \sum_{j=0}^p x_j \beta_j^*(\tau),
$$
where $Y \in \mathbb{R}$ is a univariate response and ${\bfX}=(x_0,x_1,\ldots,x_p)^{\mathrm{T}} \in \mathbb{R}^{p+1}$ is $p+1$-dimensional covariate vector with $x_0 \equiv 1$. Here, $\boldsymbol{\beta}^*=\boldsymbol{\beta}^*(\tau) = (\beta_0^*(\tau),\beta_1^*(\tau),\ldots,\beta_p^*(\tau))$ is the true coefficient vector that can be obtained by minimizing the following stochastic optimization problem,
\begin{equation}\label{quantile_loss}
\mathcal{Q}(\boldsymbol{\beta})=\mathbb{E}\left[\rho_\tau(Y-\boldsymbol{X}^{\mathrm{T}}\boldsymbol{\beta})\right] ,   
\end{equation}
where $\rho_\tau(u)=u\{\tau-I(u\leq0)\}$ is the standard check loss function \citep{koenker1978regression} with $I(\cdot)$ is the indicator function. 

\subsection{Newton-type Transformation on Quantile Regression}

To solve the stochastic optimization problem in \eqref{quantile_loss}, we use the Newton-Raphson method. Given an initial estimator $\boldsymbol{\beta}_0$, the population form of the Newton-Raphson iteration is
\begin{equation}\label{N-R_iter}
\boldsymbol{\beta}_{1}=\boldsymbol{\beta}_{0}-\boldsymbol{H}^{-1}(\boldsymbol{\beta}_{0})\mathbb{E}\left[\partial \mathcal{Q}(\boldsymbol{\beta}_{0})\right],
\end{equation}
where $\partial \mathcal{Q}(\boldsymbol{\beta})=\boldsymbol{X}\{I(Y-\boldsymbol{X}^{\mathrm{T}}\boldsymbol{\beta}\leq 0)-\tau\}$ is the subgradient of the check loss function with respect to $\boldsymbol{\beta}$, and $\boldsymbol{H}(\boldsymbol{\beta})=\partial\mathbb{E}[ \partial \mathcal{Q}(\boldsymbol{\beta})]/\partial \boldsymbol{\beta} = \mathbb{E}[\boldsymbol{X}\boldsymbol{X}^{\mathrm{T}}f_{\varepsilon|\boldsymbol{X}}(\boldsymbol{X}^{\mathrm{T}}(\boldsymbol{\beta}-\boldsymbol{\beta}^*))]$ denotes the population Hessian matrix of $\mathbb{E}[\mathcal{Q}(\boldsymbol{\beta})]$. Here, we denote the error term as $\varepsilon=Y-\boldsymbol{X}^{\mathrm{T}} \boldsymbol{\beta}^*$, and $f_{\varepsilon|\boldsymbol{X}}(\cdot)$ is the conditional density of $\varepsilon$ given $\boldsymbol{X}$. 

When the initial estimator $\boldsymbol{\beta}_0$ is close to the true parameter $\boldsymbol{\beta}^*$, $\boldsymbol{H}(\boldsymbol{\beta}_0)$ will be close to $\boldsymbol{H}(\boldsymbol{\beta}^*)=\mathbb{E}[\boldsymbol{X}\boldsymbol{X}^{\mathrm{T}}f_{\varepsilon|\boldsymbol{X}}(0)]$. Motivated by this, we further approximate $\boldsymbol{H}(\boldsymbol{\beta}^*)$ with $\boldsymbol{D}_h(\boldsymbol{\beta}_{0})$ such that 
\begin{align}\label{appro_matrix}
 \boldsymbol{H}(\boldsymbol{\beta}_{0})\approx \boldsymbol{H}(\boldsymbol{\beta}^*)\approx\boldsymbol{D}_h(\boldsymbol{\beta}_{0})=\mathbb{E}(\bfX\bfX^{\mathrm{T}}K_h(e_0)),    
\end{align}
where $ e_0 = Y-{\bfX}^{\mathrm{T}}\boldsymbol{\beta}_0$, and $K_h(\cdot)=K(\cdot/h)/h$ with $K(\cdot)$ denoting a symmetrix and non-negative kernel function, $h \rightarrow 0$ is the bandwidth. For simplicity of notation, we denote a pseudo covariate as $\widetilde{\boldsymbol{X}}_h=\sqrt{K_h(e_0)}\boldsymbol{X}$. Hence, we can rewrite $\boldsymbol{D}_h(\boldsymbol{\beta}_{0})=\mathbb{E}(\widetilde{\boldsymbol{X}}_h\widetilde{\boldsymbol{X}}_h^{\mathrm{T}})$, which is the covariance matrix of $\widetilde{\boldsymbol{X}}_h$. Replacing  $\boldsymbol{H}(\boldsymbol{\beta}_0)$ with $\boldsymbol{D}_h(\boldsymbol{\beta}_{0})$ in \eqref{N-R_iter} leads to the following iteration,
\begin{align}\label{app_iter}
\boldsymbol{\beta}_{1}=\boldsymbol{\beta}_{0}-\boldsymbol{D}_h^{-1}(\boldsymbol{\beta}_{0})\mathbb{E}[\partial \mathcal{Q}(\boldsymbol{\beta}_{0})]
\end{align}
This iteration together with the Taylor expansion of $\mathbb{E}[\partial \mathcal{Q}(\boldsymbol{\beta}_{0})]$ at $\boldsymbol{\beta}^*$,
$$
\mathbb{E}[\partial \mathcal{Q}(\boldsymbol{\beta}_{0})]=\boldsymbol{H}(\boldsymbol{\beta}^{*})(\boldsymbol{\beta}_{0}-\boldsymbol{\beta}^{*})+\mathcal{O}(|\boldsymbol{\beta}_{0}-\boldsymbol{\beta}^{*}|_2^2),
$$
guarantee an improved convergence rate of $\boldsymbol{\beta}_{1}$ in $\ell_2$-norm,
\begin{align*}
|\boldsymbol{\beta}_{1}-\boldsymbol{\beta}^{*}|_2 &= 
\left| \boldsymbol{\beta}_{0}-\boldsymbol{D}_h^{-1}(\boldsymbol{\beta}_{0,h}) (\boldsymbol{H}(\boldsymbol{\beta}^{*})(\boldsymbol{\beta}_{0}-\boldsymbol{\beta}^{*})\right.\\
&\left.+ \mathcal{O}(|\boldsymbol{\beta}_{0}-\boldsymbol{\beta}^{*}|_2^2) ) - \boldsymbol{\beta}^{*} \right|_2 = \mathcal{O}\left(|\boldsymbol{\beta}_{0}-\boldsymbol{\beta}^{*}|_2^2\right).
\end{align*}
Consequently, if we have a consistent estimator $\boldsymbol{\beta}_{0}$, we can refine it by \eqref{app_iter}.

Now we show how to transform the Newton-Raphson iteration into a least squares problem. According to \eqref{app_iter}, we have 
\begin{small}
\begin{align*}
&\boldsymbol{\beta}_{1}
=\boldsymbol{D}_h^{-1}(\boldsymbol{\beta}_{0})\left\{\boldsymbol{D}_h(\boldsymbol{\beta}_{0})\boldsymbol{\beta}_{0}-\mathbb{E}\left[\bfX(I(e_0 \leq 0)-\tau)\right]\right\} \\
&=\boldsymbol{D}_h^{-1}(\boldsymbol{\beta}_{0})\mathbb{E}\left\{\widetilde{\bfX}_h\left[\widetilde{\bfX}_h^{\mathrm{T}}\boldsymbol{\beta}_{0}-\frac{1}{\sqrt{K_h(e_0)}}\left(I(e_0\leq 0)-\tau\right)\right]\right\}.
\end{align*}    
\end{small}
If we further define a new pseudo response as 
\begin{align*}
\widetilde{Y}_h=\widetilde{\bfX}_h^{\mathrm{T}}\boldsymbol{\beta}_{0}-\frac{1}{\sqrt{K_h(e_0)}}\left(I(e_0\leq 0)-\tau\right),
\end{align*}
then $\boldsymbol{\beta}_{1}=\boldsymbol{D}_h^{-1}(\boldsymbol{\beta}_{0})\mathbb{E}(\widetilde{\bfX}_h\widetilde{Y}_h)=\argmin_{\boldsymbol{\beta} \in \mathbb{R}^{p+1}}\mathbb{E}(\widetilde{Y}_h-\widetilde{\bfX}_h^{\mathrm{T}}\boldsymbol{\beta})^2$ is the least squares regression coefficient of $\widetilde{Y}_h$ on $\widetilde{\bfX}_h$. To further encourage the sparsity of the coefficient vector, we consider the following $\ell_1$-penalized least squares problem,
\begin{equation}\label{l1_LS}
\boldsymbol{\beta}_{1,\ell_1}= \underset{\boldsymbol{\beta} \in \mathbb{R}^{p+1}}{\operatorname{argmin}}\frac{1}{2}\mathbb{E}\left(\widetilde{Y}_h-\widetilde{\bfX}_h^{\mathrm{T}}\boldsymbol{\beta}\right)^2 + \lambda |\boldsymbol{\beta}|_1,
\end{equation}
where $\lambda>0$ is the regularization parameter. We can also consider other forms of penalties including the smoothly clipped absolute deviations penalty (SCAD, \citet{fan2001variable}) and the minimax concave penalty (MCP, \citet{zhang2010nearly}). We refer the reader to \citet{hastie2015statistical} for comprehensive reviews on the recent developments.

Now we are ready to define the empirical form of $\boldsymbol{\beta}_{1}$ in a single machine. Let $\widehat{\boldsymbol{\beta}}_{0}$ be an initial estimate based on  random samples ${\cal Z}^N =  \{(\bfX_i, Y_{i})  \}_{i=1}^{N}$, then we can transform the origin covariates and responses by 
\begin{equation}\label{T_x}
\widetilde{\bfX}_{i,h}=\sqrt{K_h(\widehat{e}_{0,i})}\bfX_i ,
\end{equation}
\begin{equation}
\begin{aligned}
\label{T_y}
\widetilde{Y}_{i,h}=\widetilde{\bfX}_{i,h}^{\mathrm{T}}\widehat{\boldsymbol{\beta}}_{0}-\frac{1}{\sqrt{K_h(\widehat{e}_{0,i})}}\left(I(\widehat{e}_{0,i}\leq 0)-\tau\right) ,
\end{aligned}    
\end{equation}
for $i=1,\ldots,N$, where $\widehat{e}_{0,i}=Y_i-\boldsymbol{X}_i^{\mathrm{T}}\widehat{\boldsymbol{\beta}}_{0}$. Thus, we estimate $\boldsymbol{\beta}^*$ by the empirical version of \eqref{l1_LS}:
\begin{equation}\label{empircal_l2}
\widehat{\boldsymbol{\beta}}_{pool}=\underset{\boldsymbol{\beta} \in \mathbb{R}^{p+1}}{\operatorname{argmin}}\frac{1}{2N}\sum_{i=1}^N(\widetilde{Y}_{i,h}-\widetilde{\bfX}_{i,h}^{\mathrm{T}}\boldsymbol{\beta})^2 + \lambda |\boldsymbol{\beta}|_1.
\end{equation}

Given a consistent initial estimator $\widehat{\boldsymbol{\beta}}_{0}$, we introduce a reasonable estimator $\widehat{\boldsymbol{\beta}}_{pool}$ by pooling all data into a single machine. Compared to the standard $\ell_1$-penalized quantile regression, the least squares problem plus a lasso penalty is much more computationally efficient.  Moreover, the smoothness and strong convexity of the quadratic loss function facilitate the development of a distributed estimator in the next section. 

\begin{remark}\label{compare_rem}
Different from the work in \citet{chen2020distributed}, we remove the stringent restriction that the error term $\epsilon$ should be independent of the covariate $\bfX$. Therefore, we cannot simply take the conditional density $f_{\varepsilon|\boldsymbol{X}}(0)$ from $\boldsymbol{H}(\boldsymbol{\beta}^*)$ in \eqref{appro_matrix}.
To consider such a dependence, we further define a pseudo covariate $\widetilde{\bfX}_{h}$ that can be regarded as a density-scaled surrogate of $\bfX$. As indicated in \citet{chen2020distributed}, the extension to the dependent case seems relatively straightforward in a single machine setting. However, it is nontrivial for distributed implementation both in methodology and theory due to the \textit{curse of dimensionality}, and we succeed in solving it by leveraging a double-smoothing approach (see details in the next section). 
\end{remark}

\subsection{Distributed Estimation with a Double-smoothing Shifted Loss Function}
Suppose the random samples ${\cal Z}^N =  \{(\bfX_i, Y_{i})  \}_{i=1}^{N}$ are randomly stored in $m$ machines $\mathcal{M}_{1}, \ldots, \mathcal{M}_{m}$ with the equal local sample size that $n=N/m$.  Without loss of generality,  we assume that $\mathcal{M}_{1}$ is the central machine and denote those samples in the $k$-th machine as $\{ (\bfX_{i}, Y_{i}) \}_{i\in \mathcal{M}_k}$ with $|\mathcal{M}_{k}|=n$, for $k=1,\ldots,m$. Based on the initial estimator $\widehat{\boldsymbol{\beta}}_{0}$, every local machines can compute the transformed samples as $\{ (\widetilde{\bfX}_{i,h}, \widetilde{Y}_{i,h}) \}_{i\in \mathcal{M}_k}$ according to \eqref{T_x} and \eqref{T_y}. For ease of notation, let
\begin{equation}
\begin{aligned}
&\widehat{\boldsymbol{D}}_{k,h} = \frac{1}{n} \sum_{i \in \mathcal{M}_k} \widetilde{\boldsymbol{X}}_{i,h} \widetilde{\boldsymbol{X}}_{i,h}^{\mathrm{T}}, \\
&\widehat{\boldsymbol{D}}_{h} = \frac{1}{m} \sum_{k=1}^m \widehat{\boldsymbol{D}}_{k,h} =\frac{1}{N} \sum_{i=1}^N \widetilde{\boldsymbol{X}}_{i,h} \widetilde{\boldsymbol{X}}_{i,h}^{\mathrm{T}},
\end{aligned}    
\end{equation}
as the $k$-th local sample covariance matrix and total sample covariance matrix, respectively. It is worth noting that our algorithm does not need the local machine to explicitly calculate and broadcast $\widehat{\boldsymbol{D}}_{k,h}$ for $k \neq 1$ (see in Algorithm \ref{alg:2}). We further define the pseudo local and global loss functions, respectively, as 
\begin{equation}\label{pseudo_loss}
\begin{aligned}
    &\mathcal{L}_k(\boldsymbol{\beta})=\frac{1}{2n}\sum_{i \in \mathcal{M}_k}(\widetilde{Y}_{i,h}-\widetilde{\bfX}_{i,h}^{\mathrm{T}}\boldsymbol{\beta})^2,\\
    &\mathcal{L}_N(\boldsymbol{\beta})=\frac{1}{2N}\sum_{i=1}^N(\widetilde{Y}_{i,h}-\widetilde{\bfX}_{i,h}^{\mathrm{T}}\boldsymbol{\beta})^2.
\end{aligned}
\end{equation}
According to the Taylor expansion of $\mathcal{L}_N(\boldsymbol{\beta})$ around $\widehat{\boldsymbol{\beta}}_{0}$, we have
\begin{equation}\label{Taylor_expansion}
\begin{aligned}
\mathcal{L}_N(\boldsymbol{\beta})&=
\mathcal{L}_N(\widehat{\boldsymbol{\beta}}_0)
+\{\partial \mathcal{L}_N(\widehat{\boldsymbol{\beta}}_0)\}^{\mathrm{T}}(\boldsymbol{\beta}-\widehat{\boldsymbol{\beta}}_0) \\
&+\frac{1}{2}(\boldsymbol{\beta}-\widehat{\boldsymbol{\beta}}_0)^{\mathrm{T}} \widehat{\boldsymbol{D}}_{h}(\boldsymbol{\beta}-\widehat{\boldsymbol{\beta}}_0).   
\end{aligned}
\end{equation}
It is easy to see that $\partial \mathcal{L}_N(\widehat{\boldsymbol{\beta}}_0)$ and $\widehat{\boldsymbol{D}}_{h}$ can be simply calculated by averaging the local ones. However, the burden of transmitting the local $(p+1)\times (p+1)$ covariance matrix $\widehat{\boldsymbol{D}}_{k,h}$ is heavy when $p$ is large. To save the communication cost, we replace the global Hessian $\widehat{\boldsymbol{D}}_{h}$ with the local Hessian $\widehat{\boldsymbol{D}}_{1,b}$. Here, $h$ and $b$ denote the \emph{global bandwidth} and \emph{local bandwidth}, respectively, and we assume $b\geq h\geq 0$. Thus we can rewrite \eqref{Taylor_expansion} as 
\begin{equation}\label{substitude_loss}
\begin{aligned}
\mathcal{L}_N(\boldsymbol{\beta},\widehat{\boldsymbol{D}}_{h})
&=\underbrace{\mathcal{L}_N(\boldsymbol{\beta},\widehat{\boldsymbol{D}}_{1,b})}_{(\romannumeral1) \text{ Shifted loss}}\\
&+\underbrace{\mathcal{O}_{\mathbb{P}}\left\{\|\widehat{\boldsymbol{D}}_{h}-\widehat{\boldsymbol{D}}_{1,b}\|_{op}\cdot|\boldsymbol{\beta}-\widehat{\boldsymbol{\beta}}_0|_2^2\right\}}_{(\romannumeral2) 
 \text{ Approximation error}}.
\end{aligned}
\end{equation}
The second term in \eqref{substitude_loss} is from the Cauchy–Schwarz inequality. Note that the substituted local Hessian matrix is flexibly controlled by a local bandwidth $b$ instead of the global bandwidth $h$, which ensures that $\|\widehat{\boldsymbol{D}}_{h}-\widehat{\boldsymbol{D}}_{1,b}\|_{op}=o_\mathbb{P}(1)$ (Detailed proof can be referred in Appendix \ref{secB.4}). Remove the terms that are independent of $\boldsymbol{\beta}$ in (\romannumeral1)  and the negligible approximation error (\romannumeral2) in \eqref{substitude_loss}, the shifted loss function
can be simplified to
\begin{equation}\label{shift_loss}
\begin{aligned}
\widetilde{\mathcal{L}}(\boldsymbol{\beta})=&\frac{1}{2 n} \sum_{i \in \mathcal{M}_1}(\widetilde{\boldsymbol{X}}_{i,b}^{\mathrm{T}} \boldsymbol{\beta})^2-\boldsymbol{\beta}^{\mathrm{T}}\left\{\boldsymbol{z}_N+(\widehat{\boldsymbol{D}}_{1,b}-\widehat{\boldsymbol{D}}_{h}) \widehat{\boldsymbol{\beta}}_{0}\right\},
\end{aligned}    
\end{equation}
where $z_N=\frac{1}{N} \sum_{i=1}^N \widetilde{\boldsymbol{X}}_{i,h} \widetilde{Y}_{i,h}$. Up to now, we only need to focus on the shifted loss function in \eqref{shift_loss} instead of the pseudo global loss function in \eqref{pseudo_loss} for higher communication efficiency. Specifically, we define the one-step distributed estimator as
\begin{align}\label{shift_loss_pen}
\widehat{\boldsymbol{\beta}}_{1,h}=\underset{\boldsymbol{\beta} \in \mathbb{R}^{p+1}}{\arg \min } 
\ \widetilde{\mathcal{L}}(\boldsymbol{\beta})+ \lambda_N|\boldsymbol{\beta}|_1.   
\end{align}
Note that the local machines only need to compute two vectors $z_{n,k}=\frac{1}{n} \sum_{i\in \mathcal{M}_k} \widetilde{\boldsymbol{X}}_{i,h} \widetilde{Y}_{i,h}$  and $\widehat{\boldsymbol{D}}_{k,h}\widehat{\boldsymbol{\beta}}_{0}=\frac{1}{n} \sum_{i\in \mathcal{M}_k} \widetilde{\boldsymbol{X}}_{i,h} (\widetilde{\boldsymbol{X}}_{i,h}^{\mathrm{T}}\widehat{\boldsymbol{\beta}}_{0})$ and then broadcast them to the central machine with communication cost of  $\mathcal{O}(mp)$. There is no need to communicate the $(p+1)\times (p+1)$ covariance matrix $\widehat{\boldsymbol{D}}_{k,h}$. The central machine first calculates $z_N$ and $\widehat{\boldsymbol{D}}_h\widehat{\boldsymbol{\beta}}_{0}$ by simple averaging, then solves \eqref{shift_loss_pen} via some well-learned algorithms, e.g., the PSSsp algorithm \citep{schmidt2010graphical}, the active set algorithm \citep{solntsev2015algorithm} and the coordinate descent algorithm \citep{wright2015coordinate}. 

Given $\widehat{\boldsymbol{\beta}}_{1,h}$ as the estimator from the first iteration,  we can similarly construct an iterative distributed estimation procedure. Specifically, in the $t$-th iteration, we update the pseudo covariates and responses in \eqref{T_x} and \eqref{T_y} by substituting $\widehat{\boldsymbol{\beta}}_{0}$ with $\widehat{\boldsymbol{\beta}}_{t-1,h}$,
\begin{equation}\label{T_x_t}
\widetilde{\bfX}_{i,h}^{(t)}=\sqrt{K_h(\widehat{e}^{(t-1)}_{i,h})}\bfX_i,   
\end{equation}
\begin{equation}\label{T_y_t}
\begin{aligned}
   \widetilde{Y}_{i,h}^{(t)}=&(\widetilde{\bfX}_{i,h}^{(t-1)})^{\mathrm{T}}\widehat{\boldsymbol{\beta}}_{t-1,h}\\
   &-\frac{1}{\sqrt{K_h(\widehat{e}^{(t-1)}_{i,h})}}\left(I(\widehat{e}^{(t-1)}_{i,h}\leq 0)-\tau\right),  
\end{aligned}
\end{equation}
for $i=1,\ldots,N$, where $
\widehat{e}^{(t)}_{i,h} =  Y_i -\boldsymbol{X}_i^{\mathrm{T}}\widehat{\boldsymbol{\beta}}_{t,h}$. Similar with \eqref{shift_loss_pen}, the distributed estimator in the $t$-th iteration is 
\begin{equation}\label{dis_DHSQR_iter}
\begin{aligned}
 \widehat{\boldsymbol{\beta}}_{t,h}=&\underset{\boldsymbol{\beta} \in \mathbb{R}^{p+1}}{\arg \min } \frac{1}{2 n} \sum_{i \in \mathcal{M}_1}\left((\widetilde{\boldsymbol{X}}_{i,b}^{(t)})^{\mathrm{T}} \boldsymbol{\beta}\right)^2 - \\
 &\boldsymbol{\beta}^{\mathrm{T}}\left\{\boldsymbol{z}_N^{(t)}+\left(\widehat{\boldsymbol{D}}_{1,b}^{(t)}-\widehat{\boldsymbol{D}}_{h}^{(t)}\right) \widehat{\boldsymbol{\beta}}_{t-1,h}\right\}+\lambda_{N,t}|\boldsymbol{\beta}|_1,   
\end{aligned}
\end{equation}
where 
\begin{align*}
&\boldsymbol{z}_N^{(t)}=\frac{1}{N} \sum_{i=1}^N \widetilde{\boldsymbol{X}}_{i,h}^{(t)} \widetilde{Y}_{i,h}^{(t)}, ~  \widehat{\boldsymbol{D}}_{1,b}^{(t)}=\frac{1}{n}\sum_{i \in \mathcal{M}_1} \widetilde{\boldsymbol{X}}_{i,b}^{(t)} (\widetilde{\boldsymbol{X}}_{i,b}^{(t)})^{\mathrm{T}}, \\
& \widehat{\boldsymbol{D}}_{h}^{(t)}=\frac{1}{N} \sum_{i=1}^N\widetilde{\boldsymbol{X}}_{i,h}^{(t)} (\widetilde{\boldsymbol{X}}_{i,h}^{(t)})^{\mathrm{T}}.   
\end{align*}

\begin{algorithm*}
\caption{Distributed high-dimensional sparse quantile regression (DHSQR). } 
\hspace*{0.02in} 
\begin{algorithmic}[1]\label{alg:2}
\STATE {\bf Input:} 
Samples $\{(\boldsymbol{X}_i, Y_i)\}_{i \in \mathcal{M}_k}, k=1,\ldots,m$, the number of iterations $T$, the quantile level $\tau$, the kernel function $K$, the global and local bandwidth $h$ and $b$, the regularization parameters $\lambda_0$ and $\lambda_{N,t}$ for $t=1,\ldots,T$.
\STATE Compute the initial estimator $\widehat{\boldsymbol{\beta}}_{0,h}=\widehat{\boldsymbol{\beta}}_{0}$ based on  $\{(\boldsymbol{X}_i, Y_i)\}_{i \in \mathcal{M}_1}$ by
$$
\widehat{\boldsymbol{\beta}}_{0}=\underset{\boldsymbol{\beta} \in \mathbb{R}^{p+1}}{\arg \min } \frac{1}{n} \sum_{i \in \mathcal{M}_1}\rho_\tau(Y_i-\boldsymbol{X}_i^{\mathrm{T}}\boldsymbol{\beta})+\lambda_0\left|\boldsymbol{\beta}\right|_1.
$$
\STATE {\bf For} {$t=1,\ldots,T$ {\bf do}:}\\
\STATE{\qquad Broadcast $\widehat{\boldsymbol{\beta}}_{t-1,h}$ to the local machines.}\\
\STATE{\qquad \textbf{for} $k=1,\ldots,m$ \textbf{do:}}\\
\STATE{\qquad \qquad The $k$-th machine update the pseudo covariates $\widetilde{\bfX}_{i,h}^{(t)}$ and responses $\widetilde{Y}_{i,h}^{(t)}$ based on \eqref{T_x_t} and \eqref{T_y_t}, and computes $\widehat{\boldsymbol{D}}_{k,h}^{(t)}\widehat{\boldsymbol{\beta}}_{t-1,h}=\frac{1}{n} \sum_{i\in \mathcal{M}_k} \widetilde{\boldsymbol{X}}_{i,h}^{(t)} ((\widetilde{\boldsymbol{X}}_{i,h}^{(t)})^{\mathrm{T}}\widehat{\boldsymbol{\beta}}_{t-1,h})$ and $\boldsymbol{z}_{n,k}^{(t)}=\frac{1}{n}\sum_{i \in \mathcal{M}_k}\widetilde{\boldsymbol{X}}_{i,h}^{(t)}\widetilde{Y}_{i,h}^{(t)}$ . Then send them back to the first machine.}\\
\STATE{\qquad \textbf{end for}}\\
\STATE{\qquad The first machine computes $(\widehat{\boldsymbol{D}}_{1,b}^{(t)}-\widehat{\boldsymbol{D}}_{h}^{(t)})\widehat{\boldsymbol{\beta}}_{t-1,h}$ and $\boldsymbol{z}_{N}^{(t)}$ based on\\
$$
(\widehat{\boldsymbol{D}}_{1,b}^{(t)}-\widehat{\boldsymbol{D}}_{h}^{(t)})\widehat{\boldsymbol{\beta}}_{t-1,h}=\widehat{\boldsymbol{D}}_{1,b}^{(t)}\widehat{\boldsymbol{\beta}}_{t-1,h}-\frac{1}{m}\sum_{k=1}^{m}\widehat{\boldsymbol{D}}_{k,h}^{(t)}\widehat{\boldsymbol{\beta}}_{t-1,h}, \quad \boldsymbol{z}_{N}^{(t)}=\frac{1}{m}\sum_{k=1}^{m}\boldsymbol{z}_{n,k}^{(t)}.
$$}\\
\STATE{\qquad Compute the estimator $\widehat{\boldsymbol{\beta}}_{t,h}$ on the first machine based on \eqref{dis_DHSQR_iter}.}
\STATE{\textbf{end for}}
\STATE {\bf Return:} $\widehat{\boldsymbol{\beta}}_{T,h}$.
\end{algorithmic}
\end{algorithm*}

In this paper, we adopt the coordinate descent algorithm to solve \eqref{dis_DHSQR_iter}. For the choice of the initial estimator $\widehat{\boldsymbol{\beta}}_{0}$, we take the solution of $\ell_1$-penalized QR regression using the local data on the central machine, which can be solved by the R package "\emph{quantreg}" \citep{koenker2005quantile} or "\emph{conquer}" \citep{he2021smoothed}. Other types of initialization are also available as long as they satisfy Assumption \ref{ass6} in Section \ref{sec:theory}.  We summarize the entire
distributed estimation procedure in Algorithm \ref{alg:2}.

\textbf{Space complexity.} In each local machine $\mathcal{M}_i$, DHSQR method necessitates storing $\{\boldsymbol{X}_i\}_{i \in \mathcal{M}_i} \in \mathbb{R}^{(p+1) \times n}$, $\{Y_i\}_{i \in \mathcal{M}_i} \in \mathbb{R}^n$, $\{\tilde{\boldsymbol{X}}_{i, h}^{(t)}\}_{i \in \mathcal{M}_i} \in \mathbb{R}^{(p+1) \times n}$, and $\{\tilde{Y}_{i, h}^{(t)}\}_{i \in \mathcal{M}_i} \in \mathbb{R}^n$, resulting in a space complexity of order $\mathcal{O}(np)$. Additionally, to solve the lasso problem, we require storing a $(p+1) \times (p+1)$ matrix, resulting in a space complexity not exceeding $\mathcal{O}\left(p^2\right)$. Hence, the overall space complexity is of order $\mathcal{O}\left(np+p^2\right)$ for each local machine. The total space complexity of the total system sums up to $\mathcal{O}\left(Np+mp^2\right)$.

\begin{remark}
Note that the assumption that the samples are randomly and even stored across the local machines is commonly required in literature \citep{jordan2018communication,fan2021communication}. It is worthy pointing out that the proposed algorithm is still effective if the samples in other local machines are not randomly distributed as long as the subsample on the first machine (central machine) is randomly selected from the entire sample, which is also claimed in Remark 1 in \citet{chen2020distributed}.     
\end{remark}

\section{Statistical Guarantees}\label{sec:theory}
In this section, we establish the theoretical results of our proposed method, involving the convergence rate and support recovery accuracy. Firstly, we denote 
$$
S=\{j:\beta_j \neq 0, j \in \mathbb{N}_{+}\} \subseteq\{1, \ldots, p\},
$$ 
as the support of $\boldsymbol{\beta}^*$ and $|S|=s$. We assume the following regular conditions hold.

\begin{assumption}\label{ass1}
Assume that the kernel function $K(\cdot)$ is symmetric, non-negative, bound and integrates to one. In addition, the kernel function satisfies that  $\int_{-\infty}^{\infty} u^2 K(u) \mathrm{d} u<\infty$ and $\min _{|u| \leq 1} K(u)>0$. We further assume $K(\cdot)$ is differentiable and its derivative $K^{\prime}(\cdot)$ is bounded.
\end{assumption}

\begin{assumption}\label{ass2}
There exist $f_2\geq f_1>0$ such that $f_1 \leqslant f_{\varepsilon \mid \boldsymbol{X}}(0) \leqslant f_2$ almost surely (for all $\boldsymbol{X}$). Moreover, there exists some $l_0$ such that 
$$
\left|f_{\varepsilon \mid \boldsymbol{X}}(u)-f_{\varepsilon \mid \boldsymbol{X}}(v)\right| \leq l_0|u-v|,
$$ 
for any $u,v \in \mathbb{R}$ and all $\boldsymbol{X}$.
\end{assumption}

\begin{assumption}\label{ass3}
The random covariate $\boldsymbol{X} \in \mathbb{R}^p$ is sub-Gaussian: there exists some $c_1>0$ such that 
$$
{\mathbb{P}}\left(\left|\boldsymbol{X}^{\mathrm{T}} \boldsymbol{\Sigma}^{-1/2} \delta\right| \geqslant c_1 t\right) \leq 2 e^{-t^2 / 2},
$$ 
for every unit vector $\delta $ and $t>0$, where $\boldsymbol{\Sigma}=\mathbb{E}(\boldsymbol{X}\boldsymbol{X}^{\mathrm{T}})$.
\end{assumption}

\begin{assumption}\label{ass4}
Denote $\boldsymbol{I}=\mathbb{E}\left\{f_{\varepsilon \mid \boldsymbol{X}}(0) \boldsymbol{X} \boldsymbol{X}^{\mathrm{T}}\right\}$, then $\boldsymbol{I}$ satisfies that 
$$
\left\|\boldsymbol{I}_{S^c \times S} \boldsymbol{I}_{S \times S}^{-1}\right\|_{\infty} \leq 1-\alpha,
$$
for some $0<\alpha<1$. Moreover, we assume that $\lambda_{-} \leq \Lambda_{\min }(\boldsymbol{I}) \leq \Lambda_{\max }(\boldsymbol{I}) \leq \lambda_{+}$ for some $\lambda_{-},\lambda_{+}>0$.
\end{assumption}

\begin{assumption}\label{ass5} 
The dimension $p$ satisfies $p=\mathcal{O}\left(N^\nu\right)$ for some $\nu>0$. The local sample size $n$ satisfies $n \geq N^c$ for some $0<c<1$, and the sparsity level $s$ satisfies $s=\mathcal{O}\left(n^r\right)$ for some $0\leq r<1 / 3$.
\end{assumption}

\begin{assumption}\label{ass6} 
We assume the initial estimator $\widehat{\boldsymbol{\beta}}_{0,h}$ satisfies that $|\widehat{\boldsymbol{\beta}}_{0,h}-\boldsymbol{\beta}^{*}|_2=\mathcal{O}_\mathbb{P}(a_N)$, where $a_N=\sqrt{s\log N/n}$. And suppose that $\mathbb{P}(\operatorname{supp}(\widehat{\boldsymbol{\beta}}_{0,h}) \subseteq S) \rightarrow 1$.
\end{assumption}

Assumption \ref{ass1} imposes some regularity conditions on the kernel function $K(\cdot)$, which is satisfied by many popular kernels, including the Gaussian kernel. Assumption \ref{ass2} is a mild condition on the smoothness of the conditional density function of the error term, which is standard in quantile regression \citep{chen2019quantile,chen2020distributed,tan2022communication}. Assumption \ref{ass3} requires that the distribution of $\boldsymbol{X}$ have heavier tails than Gaussian to obtain standard convergence rates for the quantile regression estimates. Assumption \ref{ass4} is known as the \textit{irrepresentable condition}, which is commonly assumed in the sparse high-dimensional estimation literature for the support recovery \citep{belloni2011,chen2020distributed,bao2021one}. Assumption \ref{ass5} is also a common condition in the distributed estimation literature, see also in \citet{wang2017efficient,jordan2018communication,chen2020distributed}. Note that $p=\mathcal{O}(N^\nu)$, we use $\log N$ instead of the commonly used $\log (\max (p,N))$ in the convergence rates. Assumption \ref{ass6} assumes the convergence rate and support recovery accuracy of the initial estimator, which can be satisfied by the estimator using the local sample from a single machine under Assumption \ref{ass2}-\ref{ass5} and some regularity conditions \citep{fan2014adaptive}. We first show the convergence rate of the one-step DHSQR estimator $\widehat{\boldsymbol{\beta}}_{1,h}$.
\begin{theorem}\label{thm1}
Suppose that the initial estimator satisfies that $|\widehat{\boldsymbol{\beta}}_{0,h}-\boldsymbol{\beta}^{*}|_2=\mathcal{O}_\mathbb{P}(a_N)$ and let $h \asymp\left(s \log N / N\right)^{1 / 3} $, $b \asymp\left(s \log n / n\right)^{1 / 3}$ and $a_N \asymp \sqrt{s\log N/n}$. Take
\begin{align*}
\lambda_N=C\left(\sqrt{\frac{\log N}{N}}+a_N\eta\right),
\end{align*}
where $C$ is a sufficient large constant, and $\eta=\max\left\{(s\log n/n)^{1/3},\sqrt{s\log N/n}\right\}$. Then under Assumption \ref{ass1}-\ref{ass6}, we have 
\begin{equation}\label{CR_1}
\left|\widehat{\boldsymbol{\beta}}_{1,h}-\boldsymbol{\beta}^*\right|_2 =\mathcal{O}_\mathbb{P}\left(\sqrt{\frac{s\log N}{N}}+\sqrt{s}a_N\eta\right).    
\end{equation}
\end{theorem}

With a proper choice of the global and local bandwidth $h$ and $b$, we can refine the initial estimator by one iteration of our algorithm. Specifically, the convergence rate reduces from $\mathcal{O}_\mathbb{P}(a_N)$ to $\mathcal{O}_\mathbb{P}(\max \{\sqrt{s\log N/N}, \sqrt{s}a_N\eta\})$ with $\sqrt{s}\eta=o(1)$ by Assumption \ref{ass5}. Now, we can recursively apply Theorem \ref{thm1} to get the convergence rate of the iterative DHSQR estimator. For $1\leq g\leq t$, we denote 
$$
\kappa(a,b,g) = \max\left\{s^{a}\left(\frac{\log n}{n}\right)^{\frac{2g+3}{6}}, s^{b}\left(\frac{\log N}{n}\right)^{\frac{g+1}{2}}\right\}.
$$ 

\begin{theorem}\label{thm2}
Suppose that the initial estimator satisfies that $|\widehat{\boldsymbol{\beta}}_{0,h}-\boldsymbol{\beta}^{*}|_2=\mathcal{O}_\mathbb{P}(\sqrt{s\log N/n})$ and let $h \asymp\left(s \log N / N\right)^{1 / 3} $, $b \asymp\left(s \log n / n\right)^{1 / 3}$.  For $1\leq g\leq t$, take
\begin{align*}
\lambda_{N,g}=C\left(\sqrt{\frac{\log N}{N}} + \kappa(\frac{5g}{6}, g, g) \right),
\end{align*}
where $C$ is a sufficiently large constant. Then under Assumption \ref{ass1}-\ref{ass6}, we have 
\begin{equation}\label{CR_t}
\left|\widehat{\boldsymbol{\beta}}_{t,h}-\boldsymbol{\beta}^*\right|_2=\mathcal{O}_{\mathbb{P}}\left(\sqrt{\frac{s\log N}{N}}+\kappa(\frac{5t+3}{6},\frac{2t+1}{2}, t)\right).    
\end{equation}
\end{theorem}

When the number of iterations $t$ satisfies that both 
\begin{equation}\label{iter}
\begin{aligned}
    t &\geq \frac{3\log(m\log (N/m)/( \log N))}{\log(N^2/(m^2s^5\log^2 (N/m)))}, \\
   \text{and}\quad & t \geq \frac{\log m}{\log(N/(ms^2\log N))} ,   
\end{aligned}
\end{equation}
the second term in \eqref{CR_t} will be dominated by the first term, therefore, we have $|\widehat{\boldsymbol{\beta}}_{t,h}-\boldsymbol{\beta}^*|_2=\mathcal{O}_\mathbb{P}(\sqrt{s\log N/N})$. Under Assumption \ref{ass5}, we can easily verify that the right side of \eqref{iter} is bounded by a constant, which indicates that after a constant number of iterations, the DHSQR estimator can reach the same convergence rate as the traditional $\ell_1$-penalized QR estimators in a single machine \citep{belloni2011}. Interestingly, our algorithm needs the number of iterations to increase logarithmically with the number of machines $m$ to achieve the oracle rate $\sqrt{s/N}$ (up to a logarithmic factor). However, most existing distributed first-order algorithms require the number of iterations to increase polynomially with $m$ \citep{zhang2018communication}.

Next, we provide the support recovery of the one-step and $t$-th iteration DHSQR estimators in the following two theorems. Let $\widehat{\boldsymbol{\beta}}_{t,h}=(\widehat{\beta}_{t,h}^0,\widehat{\beta}_{t,h}^1,\ldots,\widehat{\beta}_{t,h}^p)$ and 
$$
\widehat{S}_t=\{j:\widehat{\beta}_{t,h}^p \neq 0, j \in \mathbb{N}_{+} \},
$$ 
be the support of $\widehat{\boldsymbol{\beta}}_{t,h}$, where $t\ge 1$.

\begin{theorem}\label{thm3}
Under the same conditions of Theorem \ref{thm1}, we have $\mathbb{P}(\widehat{S}_1 \subseteq  S) \rightarrow 1$. Furthermore, if there exists a sufficiently large constant $C>0$ such that 
\begin{equation}\label{min_beta_1}
\min_{j \in S}\left|\beta_j^*\right|\geq C\left\|\boldsymbol{I}_{S \times S}^{-1}\right\|_{\infty}\left(\sqrt{\frac{\log N}{N}}+a_N\eta \right),
\end{equation}
where $\eta=\max\left\{(s\log n/n)^{1/3},\sqrt{s\log N/n}\right\}$. Then we have $\mathbb{P}(\widehat{S}_1 = S) \rightarrow 1.$
\end{theorem}

Based on Theorem \ref{thm3}, we can show a weaker \textit{beta-min} condition for the support recovery result of the $t$-th iteration DHSQR estimator.

\begin{theorem}\label{thm4}
Under the same conditions of Theorem \ref{thm2}, we have 
$\mathbb{P}(\widehat{S}_t \subseteq  S) \rightarrow 1$. Furthermore, if there exists a sufficiently large constant $C>0$ such that 
\begin{align*}
\min_{j \in S}\left|\beta_j^*\right|\geq C\left\|\boldsymbol{I}_{S \times S}^{-1}\right\|_{\infty}\left(\sqrt{\frac{\log N}{N}}+\kappa(\frac{5t}{6}, t;t) \right).
\end{align*}
Then we have $\mathbb{P}(\widehat{S}_t = S) \rightarrow 1.$
\end{theorem}

Theorem \ref{thm3} and \ref{thm4} establish the support recovery results of our one-step and iterative DHSQR estimators by the \textit{beta-min} condition $\min_{j \in S}\left|\beta_j^*\right|$, which is wildly used in the sparse high-dimensional estimation literature.  When the number of iterations $t$ satisfies \eqref{iter}, the \textit{beta-min} condition will reduce to $\min_{j \in S}\left|\beta_j^*\right|\geq C\left\|\boldsymbol{I}_{S \times S}^{-1}\right\|_{\infty}\sqrt{\log N/N}$, which matches the rate of the lower bound for the \textit{beta-min} condition in a single machine \citep{wainwright2009sharp}.

\section{Simulation Studies}\label{sec:sim}
In this section, we provide simulation studies to assess the performance of our DHSQR estimator \footnote{The R code to reproduce our experimental results is available in \url{https://github.com/WangCaixing-96/DHSQR}.}.  We generate synthetic data from the following linear models, corresponding to the homoscedastic error case (model 1) and the heteroscedastic error case (model 2):
\begin{itemize}
    \item Model 1: $Y_i=\boldsymbol{X}_i^{\mathrm{T}} \boldsymbol{\beta}^{*}+\varepsilon_i$;
    \item Model 2:  $
    Y_i=\boldsymbol{X}_i^{\mathrm{T}} \boldsymbol{\beta}^{*}+(1+0.4 x_{i1})\varepsilon_i$,
\end{itemize}
where $\boldsymbol{X}_i=\left(1, x_{i 1}, \ldots, x_{i p}\right)^{\mathrm{T}}$ is a $(p+1)$-dimensional  vector and $\left(x_{i 1}, \ldots, x_{i p}\right)$ is drawn from a multivariate normal distribution $N(\boldsymbol{0},\boldsymbol{\Sigma})$ with covariance matrix $\boldsymbol{\Sigma}_{i j}=0.5^{|i-j|}$ for $1 \leq i, j \leq p$, the true parameter $\boldsymbol{\beta}^{*}=(1,1,2,3,4,5,\boldsymbol{0}_{p-5})^\mathrm{T}$. We fix the dimension $p=500$, and consider three different values of $\tau$, i.e., $\tau \in \{0.3, 0.5, 0.7\}$. We consider the following three noise distributions: 
\begin{enumerate}
    \item Normal distribution: the noise $\varepsilon_i \sim N(0, 1)$;
    \item $t_3$ distribution: the noise $\varepsilon_i\sim t(3)$;
    \item Cauchy distribution: the noise $\varepsilon_i \sim Cauchy(0,1)$.
\end{enumerate}

For the choice of the kernel function, we use the standard Gaussian kernel function that satisfies Assumption \ref{ass1}. For the global and local bandwidth, we set $h= 5 (s\log N/n)^{1/3}$ and $b=0.53 (s\log n/n)^{1/3}$ , respectively, according to the theoretical results in Theorem \ref{thm1} and \ref{thm2}. The regularization parameters $\lambda_{N,g}$ are selected by validation. Specifically, we choose $C_0$ to minimize the check loss on the validation set. 
All the simulation results are the average of 100 independent experiments.

\begin{figure}
    \centering
    \includegraphics[width=0.23\textwidth]{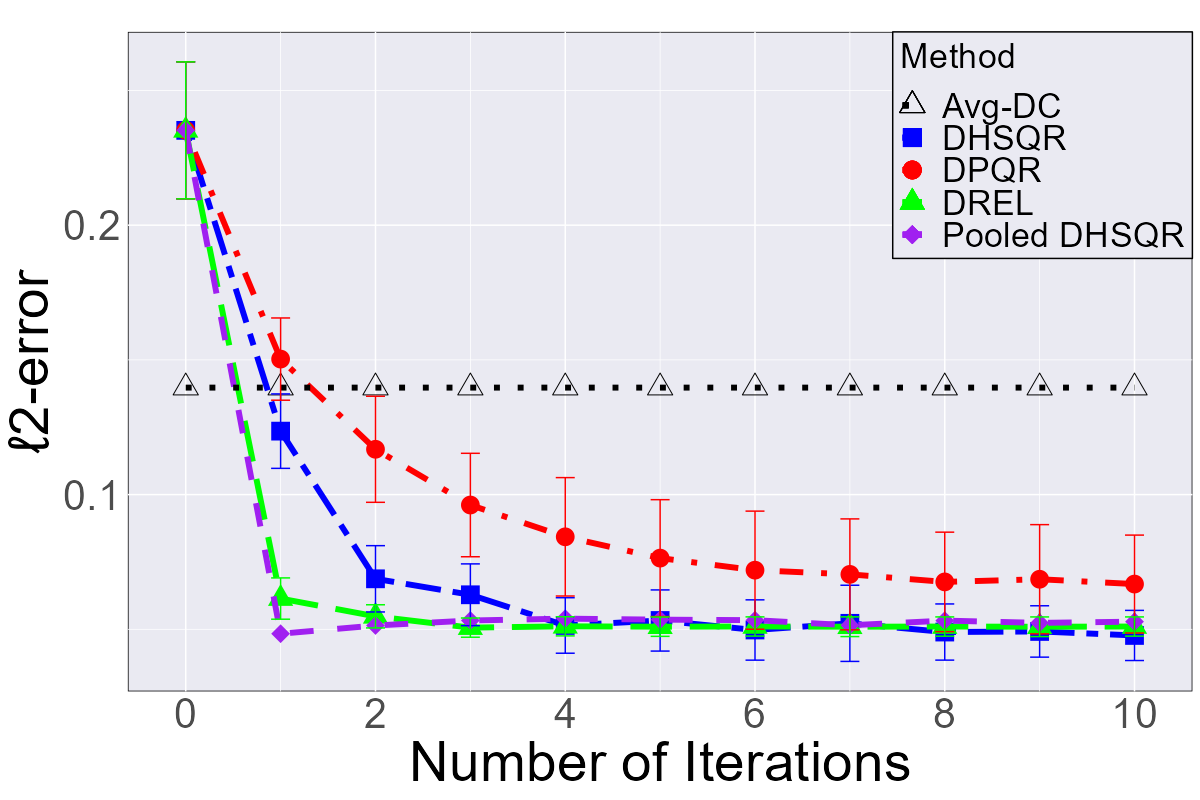}\label{hom_normal}
     \includegraphics[width=0.23\textwidth]{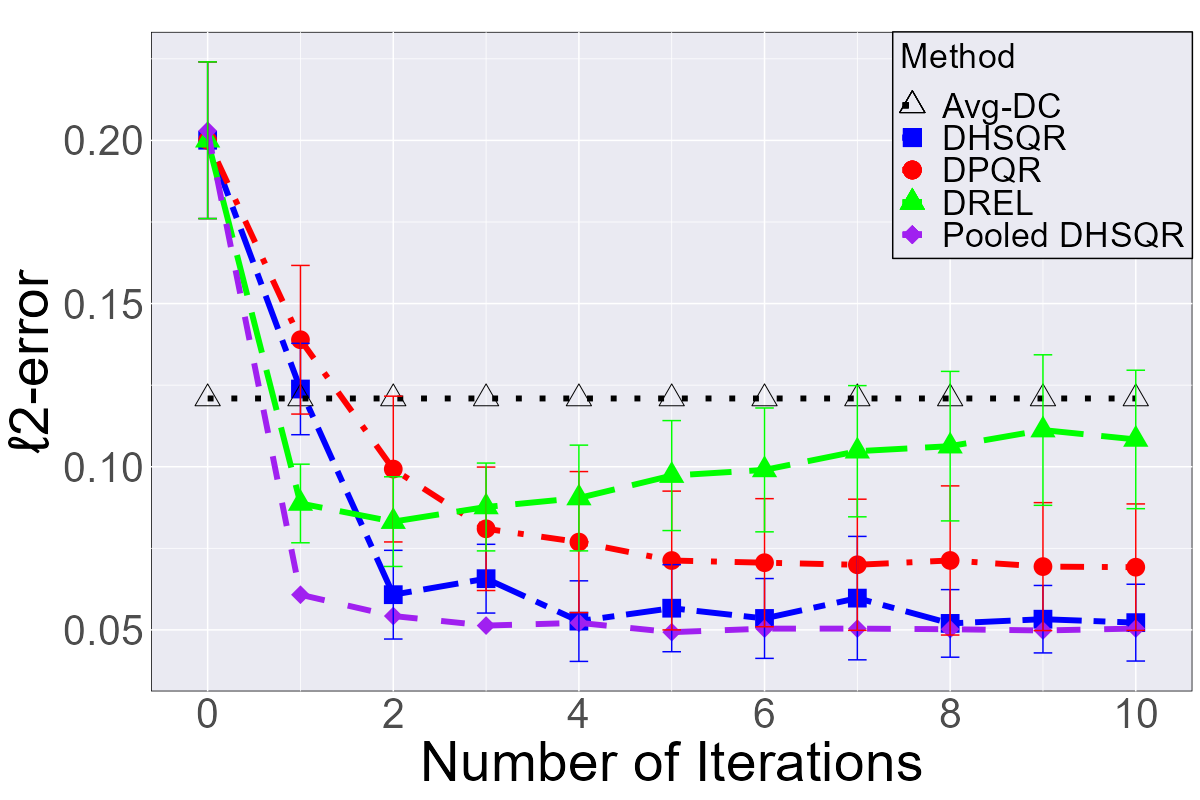}\label{het_normal}
  \includegraphics[width=0.23\textwidth]{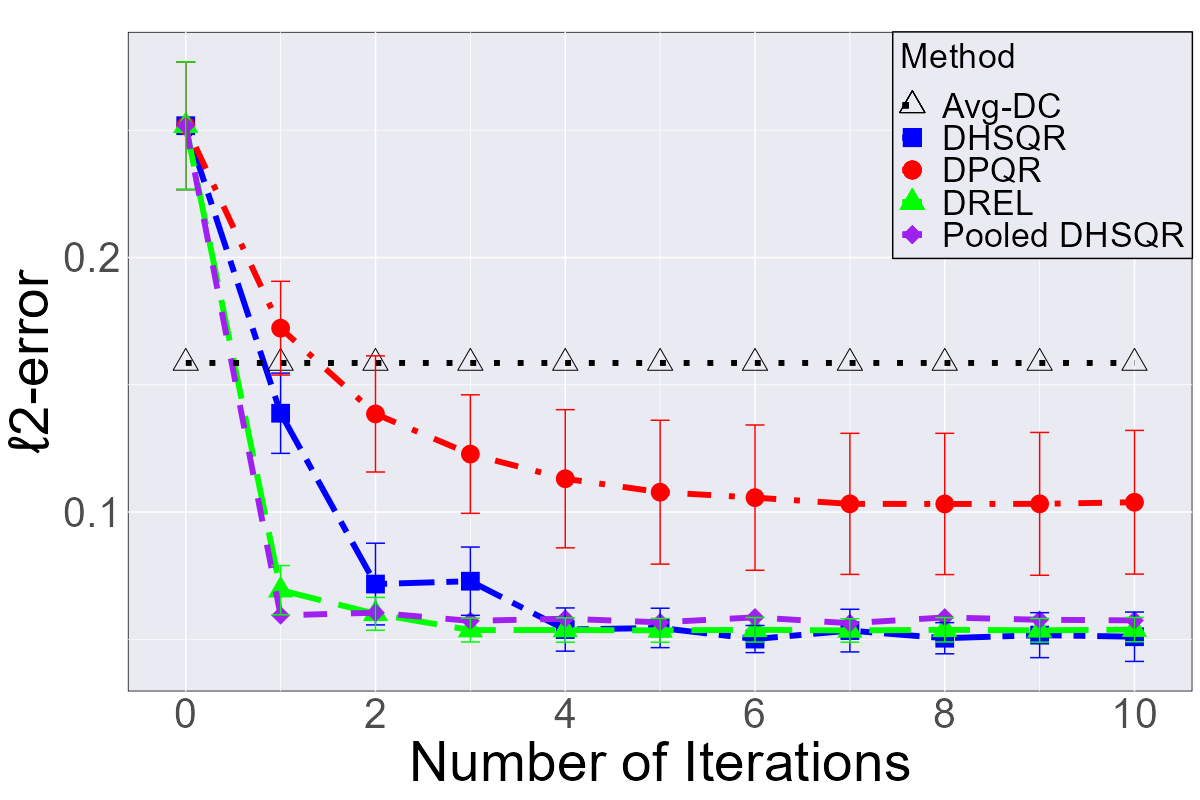}\label{hom_t}
   \includegraphics[width=0.23\textwidth]{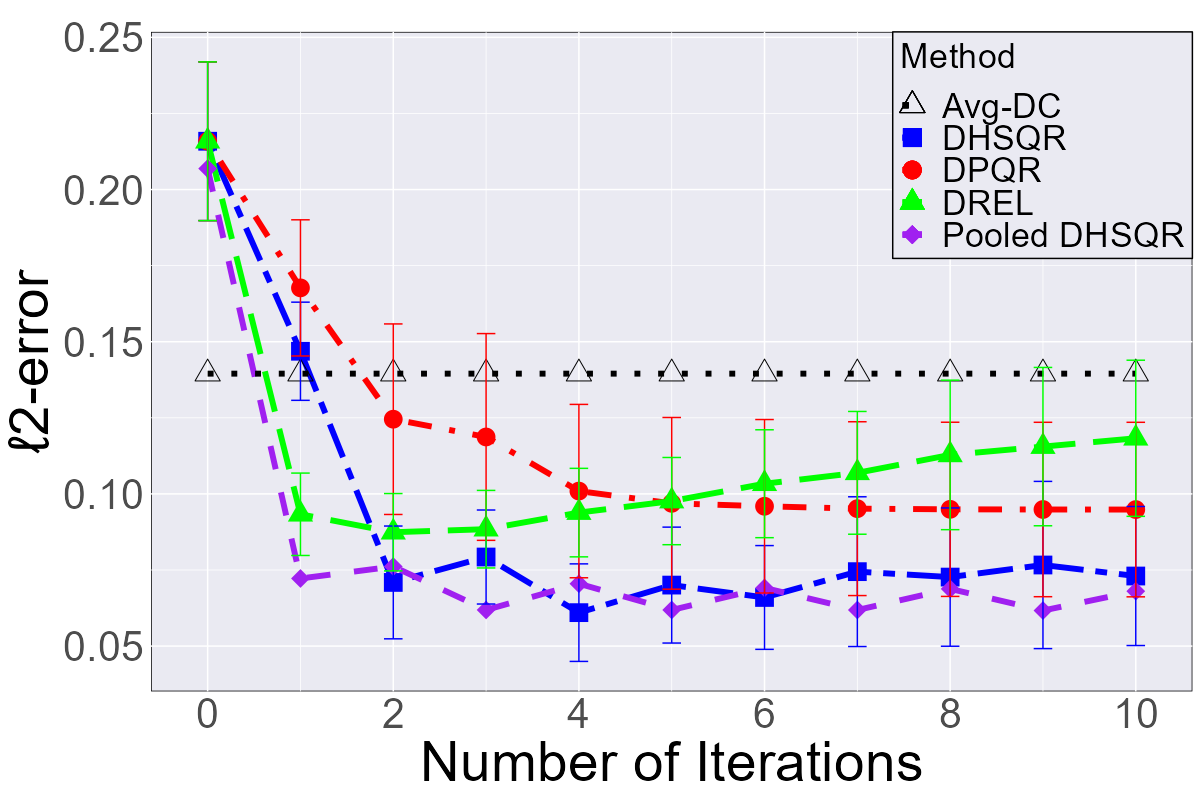}\label{het_t}
   \includegraphics[width=0.23\textwidth]{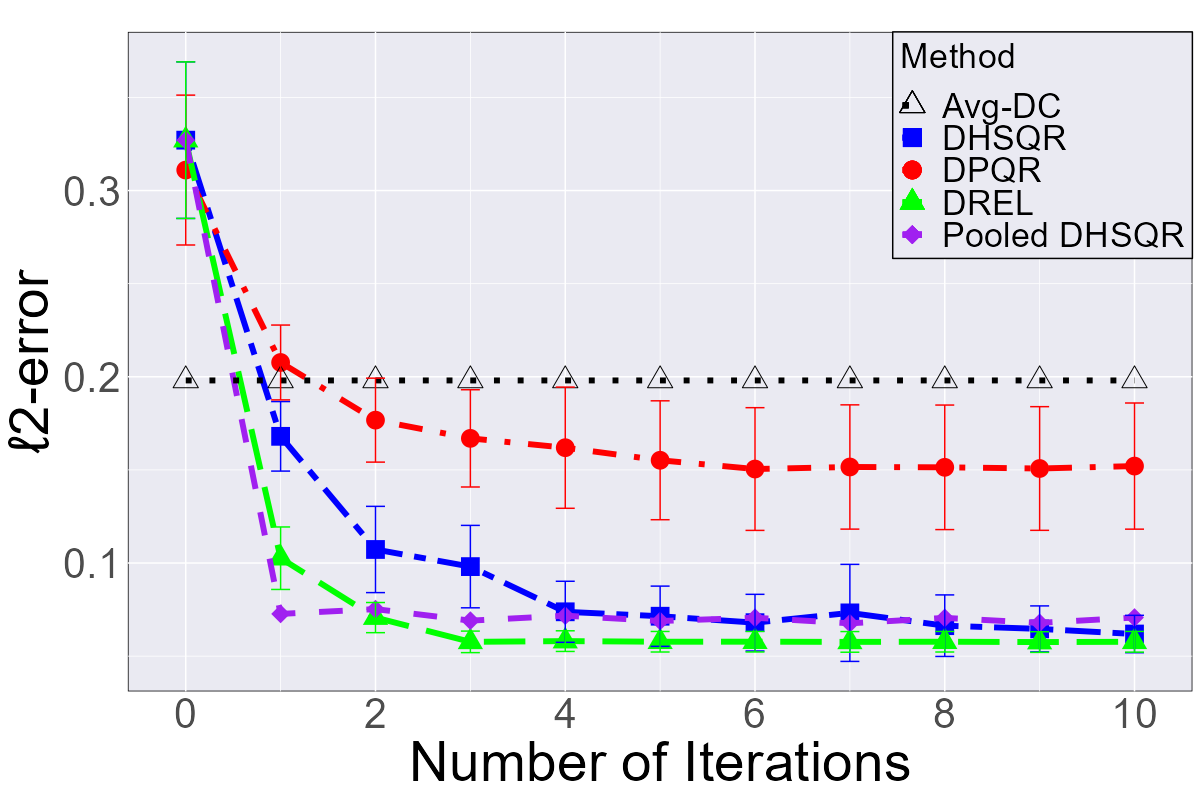}\label{hom_c}
   \includegraphics[width=0.23\textwidth] {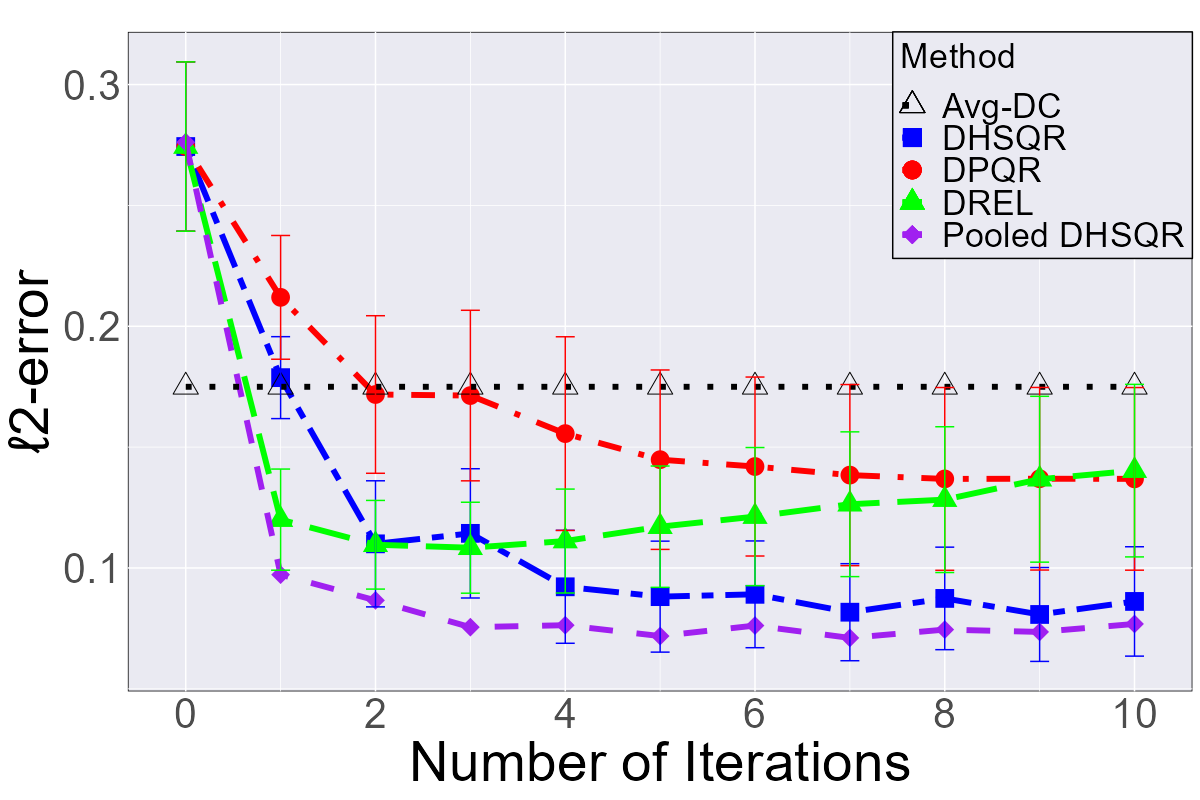}\label{het_c}
  
    \caption{The $\ell_2$-error with an error bound between the true parameter and the estimated parameter versus the number of iterations with a fixed quantile level $\tau=0.5$. In the left panel, from top to bottom represent noise distributions that are Normal, $t_3$, and Cauchy distribution for the homoscedastic error case, respectively. In the right panel, from top to bottom represent noise distributions as Normal, $t_3$, and Cauchy distribution for the heteroscedastic error case, respectively.}
    \label{fig.1}
\end{figure}

To evaluate the performance of our proposed method, we report the $\ell_2$-error between the estimate and the true parameter. In addition, we calculate precision and recall defined as the proportion of correctly estimated positive (TP) in estimated positive (TP+FP) and the proportion of correctly estimated positive (TP) in true positive (TP+FN), which is used to show the support recovery accuracy (i.e.,  precision = 1 and recall = 0 implies perfect support recovery).  We also report the $F_1$-score is defined as
$$
F_1 =\frac{2}{\text { Recall }^{-1}+\text { Precision }^{-1}},
$$
which is commonly used as an evaluation of support recovery. Note that $F_1$-score$=1$ implies perfect support recovery. We compare the finite sample
performance of the \textit{DHSQR} estimator and the other four estimators: 

\begin{itemize}
    \item[(a)]  Averaged DC (\textit{Avg-DC}) estimator which computes the $\ell_1$-penalized QR estimators on the local machine and then combines the local estimators by taking the average;
    \item[(b)] The distributed high-dimensional sparse quantile regression estimator on a single machine with pooled data defined in \eqref{empircal_l2}, which is denoted by \textit{Pooled DHSQR};
    \item[(c)] Distributed robust estimator with Lasso (\textit{DREL}), see in \citet{chen2020distributed};
    \item[(d)] Distributed penalty quantile regression estimator (\textit{DPQR}) with convolution smoothing, see in \citet{jiang2021smoothing,tan2022communication}.
\end{itemize}

\subsection{Effect of the number of iterations Under Heavy-Tailed Noise}
We first show the effect of the number of iterations in our proposed method. We fix the total sample size $N=20000$ and local sample size $n=500$. We plot the $\ell_2$-error from the true QR coefficients versus the number of iterations.
Since the Avg-DC only requires one-shot communication, we use a horizontal
line to show its performance. The results are shown in Figure \ref{fig.1}.

\begin{figure}
    \centering
   \includegraphics[width=0.23\textwidth]{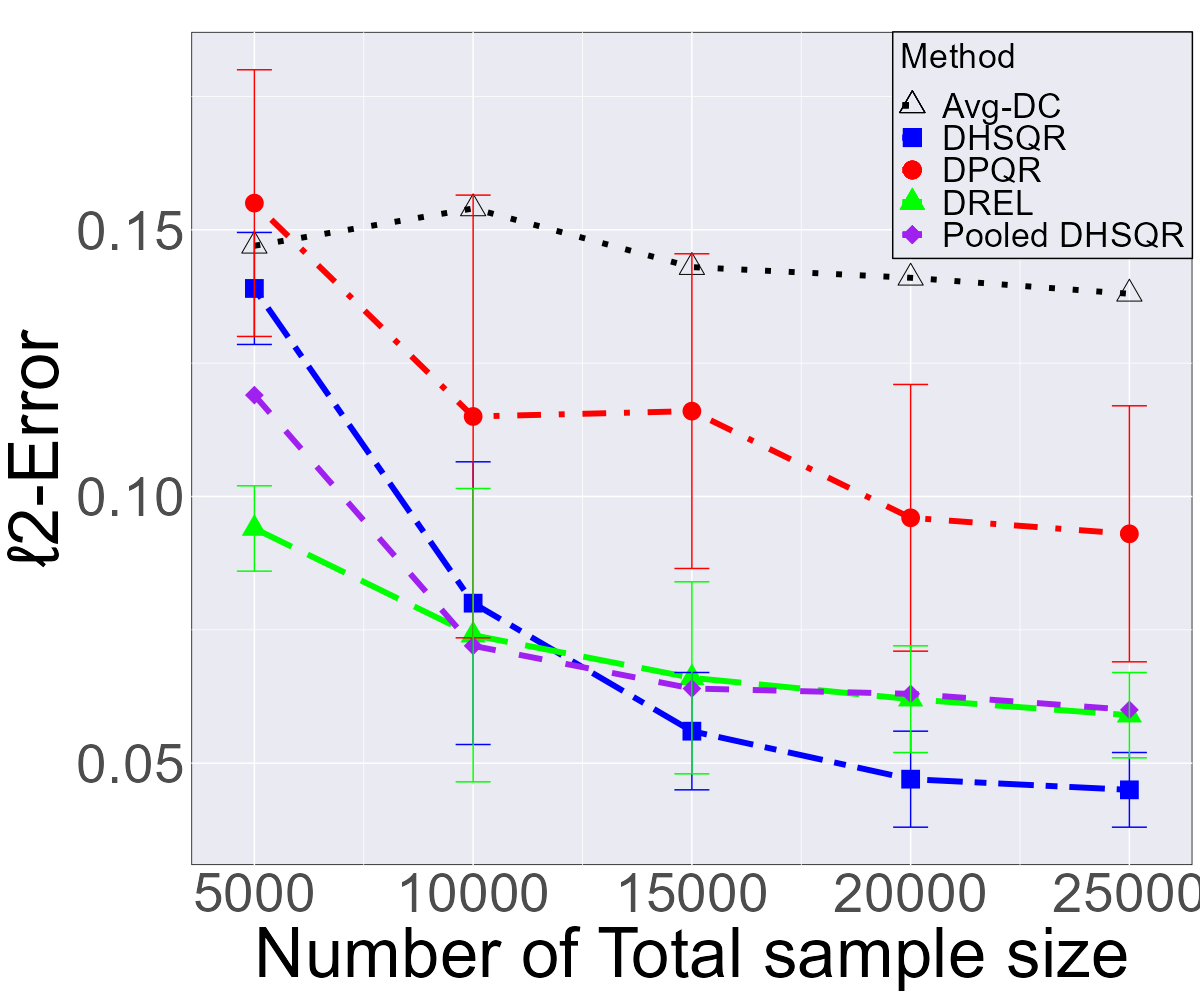}
        \label{label_for_cross_ref_7}
        \includegraphics[width=0.23\textwidth]{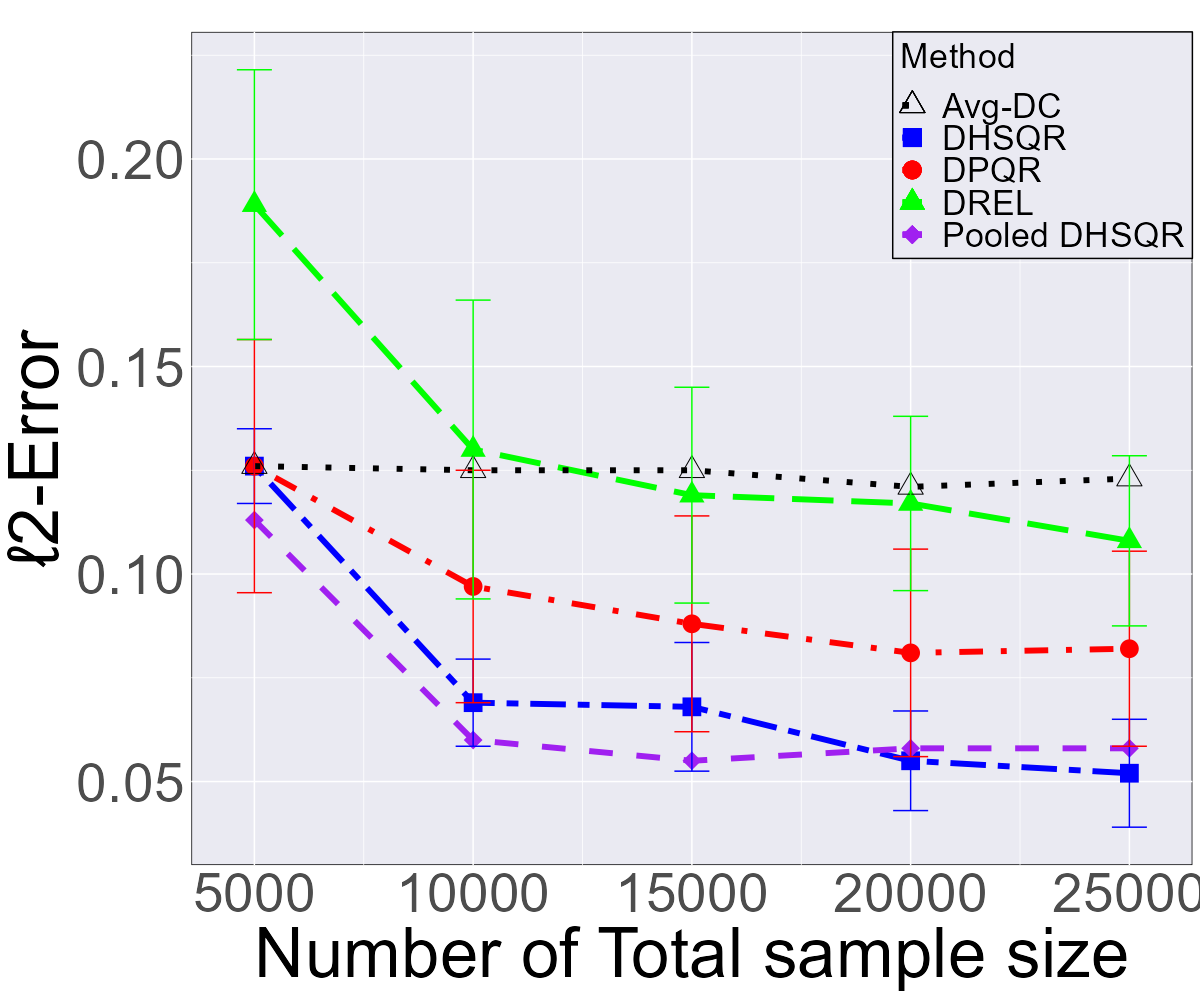}
        \label{label_for_cross_ref_8}
       \includegraphics[width=0.23\textwidth]{plot/iter_n_0.5_hom_normal_l2.png}
            \label{label_for_cross_ref_9}
        \includegraphics[width=0.23\textwidth]{plot/iter_n_0.5_het_normal_l2.png}
            \label{label_for_cross_ref_10}
    \caption{The $\ell_2$-error from the true parameter versus the number of total and local sample size with a fixed quantile level $\tau=0.5$. In the top panel, from left to right show the effect of different total sample sizes $N$ for the homoscedastic and heteroscedastic error cases, respectively. In the bottom panel, from left to right show the effect of different local sample sizes $n$ for the homoscedastic and heteroscedastic error cases, respectively.}
    \label{fig.2}
\end{figure}

 From the result, our DHSQR estimator outperforms the Avg-DC and DPQR estimators in all the cases after a few iterations, and the estimated $\ell_2$-error is very close to that of the pooled DHSQR estimator. Our method is robust for different noise settings, while DPQR performs poorly for heavy-tailed noise. Interestingly, for the heteroscedastic error case, the DREL estimator becomes unstable and fails to converge since it ignores the dependence between the error term and covariates. These results confirm our theoretical findings in section \ref{sec:theory}. For the rest of the experiments in this section, we fix the number of iterations $T=10$.

\subsection{Effect of Total Sample Size and Local Sample Size}
In this section, we investigate the performance of our proposed estimator under varying total and local sample sizes. We only consider the setting with normal error case and quantile level $\tau=0.5$.
For the effect of total sample size, we fix the local sample size $n = 500$, and vary the total sample
size $N \in \{5000, 10000, 15000, 20000, 25000\}$. For the effect of local sample size, we fix the total sample size $N = 20000$, and vary the local sample
size $n \in \{200, 500, 1000, 2000\}$.
We plot the $\ell_2$-error and $F_1$-score of the five estimators under different models in Figure \ref{fig.2} and \ref{fig.3}.

\begin{figure}
    \centering
  \includegraphics[width=0.23\textwidth]{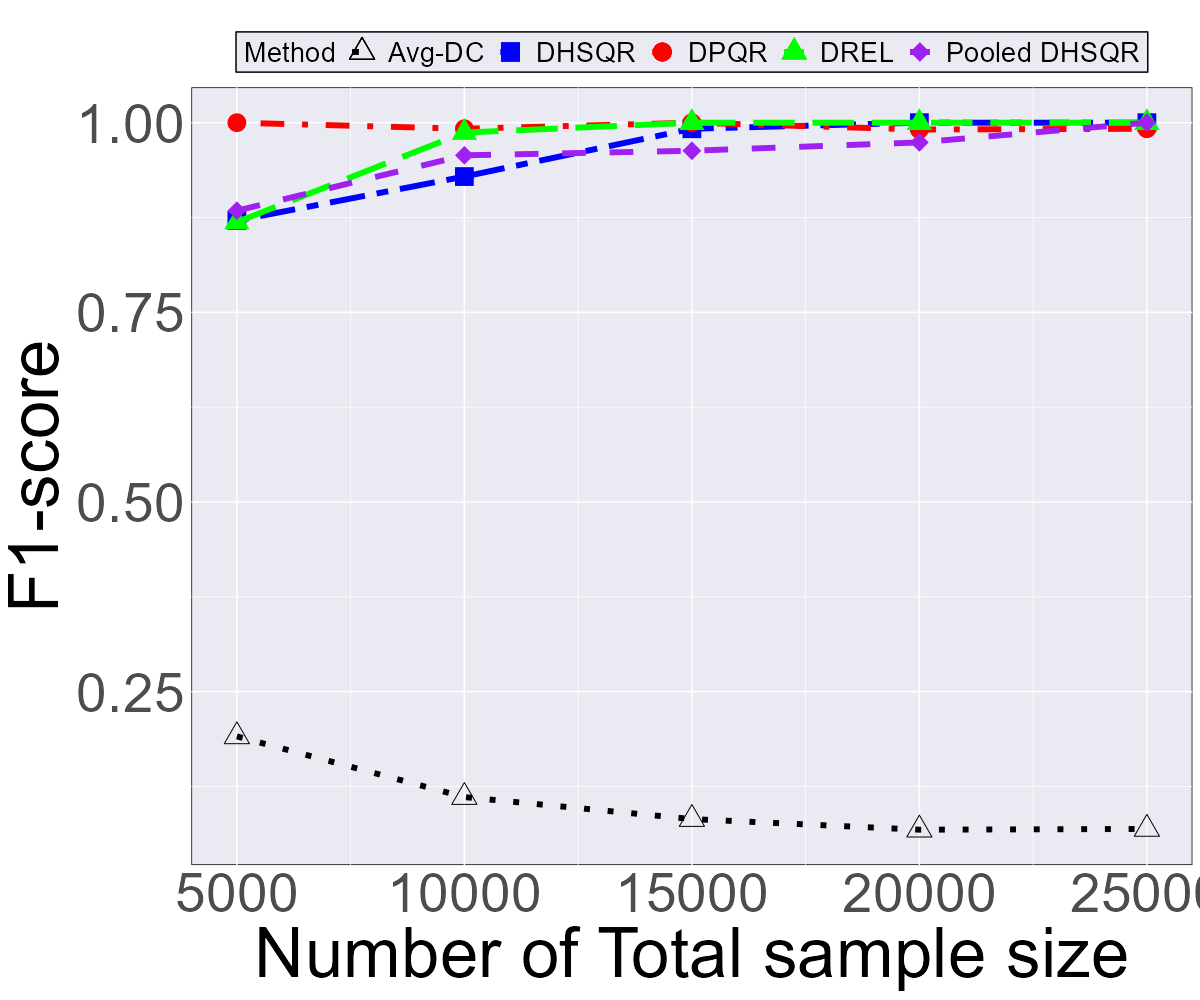}
        \label{label_for_cross_ref_11}
        \includegraphics[width=0.23\textwidth]{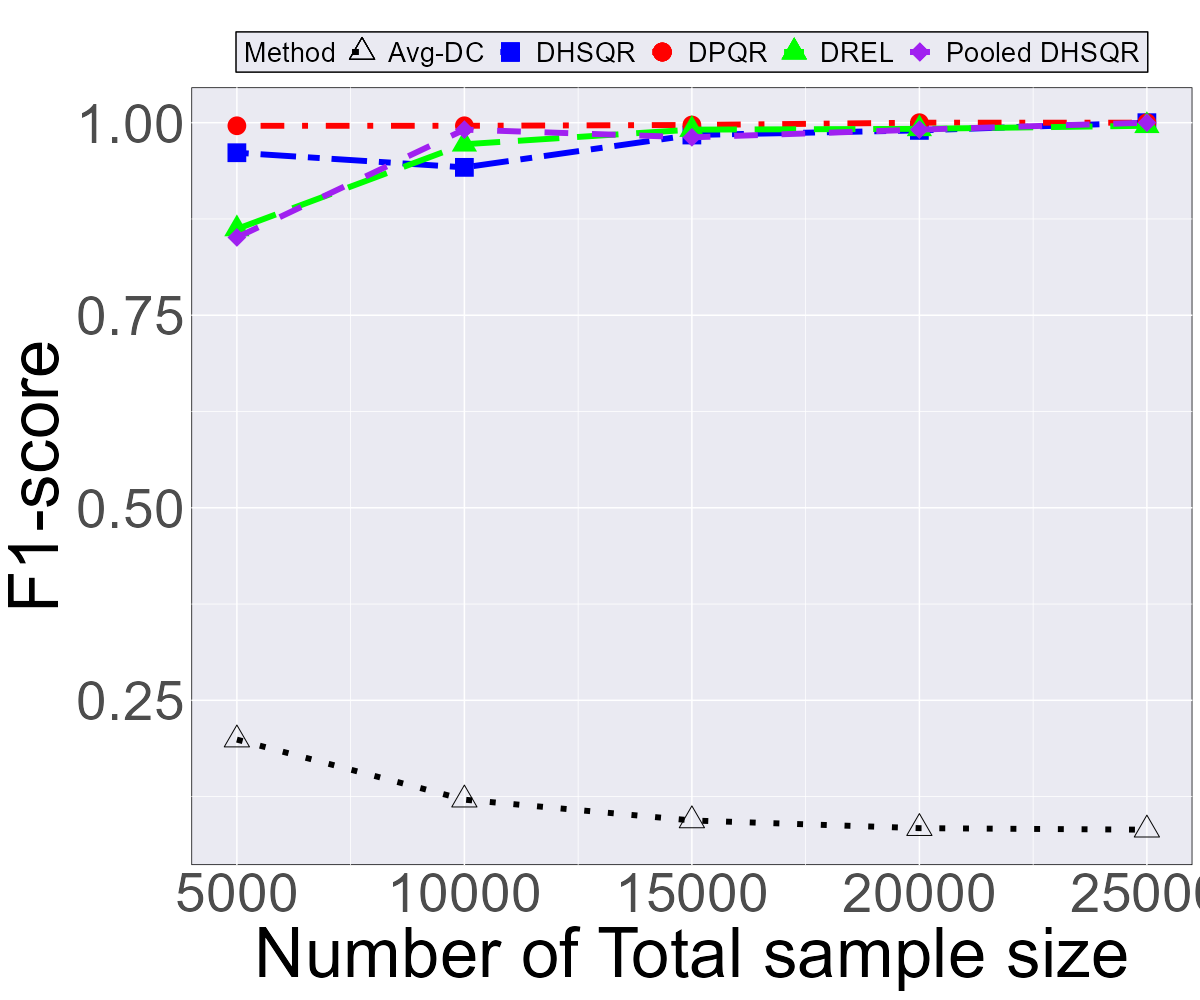}
        \label{label_for_cross_ref_12}
\includegraphics[width=0.23\textwidth]{plot/iter_n_0.5_hom_normal_F1.png}
            \label{label_for_cross_ref_13}
\includegraphics[width=0.23\textwidth]{plot/iter_n_0.5_hom_normal_F1.png}
            \label{label_for_cross_ref_14}
    \caption{The $F_1$ score from the true parameter versus the number of total and local sample size with a fixed quantile level $\tau=0.5$. 
    In the top panel, from left to right show the effect of different total sample sizes $N$ for the homoscedastic and heteroscedastic error cases, respectively. In the bottom panel, from left to right show the effect of different local sample sizes $n$ for the homoscedastic and heteroscedastic error cases, respectively.}
    \label{fig.3}
\end{figure}
The $\ell_2$-error of all methods decreases with an increase in the total sample size $N$ and local sample size $n$, as shown in Figure \ref{fig.2}.
For DHSQR, DPQR, and DREL estimators, as $N$ increases or $n$ increases, the $\ell_2$-error tends to decrease, and both of them outperform the Avg-DC estimator.
We can see that the $\ell_2$-error of DHSQR is very close to that of the Pooled DHSQR estimator in both homoscedastic and heteroscedastic error cases. The DHSQR estimator is better than other three distributed estimators. In the heteroscedastic error case, the DREL estimator's $\ell_2$-error is significantly worse compared to DHSQR and DPQR. The Pooled DHSQR estimator's performance is not significantly affected by variations in $n$, while Avg-DC shows minimal sensitivity to the changes in $N$.

In terms of support recovery, as Figure \ref{fig.3} is shown, both DHSQR, Pooled DHSQR, DPQR, and DREL estimators outperform the Avg-DC estimator in all settings, and their $F_1$ score are nearly equal to 1. It is also noteworthy that the Avg-DC estimator fails in support recovery since it is usually dense after averaging, especially when $N$ is large and $n$ is small.

Due to the space limit, we have also included an additional real data application and simulated experiments for sensitivity analysis for the bandwidth, computation time, different quantile levels, and different settings of parameters $\boldsymbol{\beta}^{*}$ in Appendix \ref{sec2}.

\section{Conclusion}\label{sec:con}

In this paper, we propose an efficient distributed quantile regression method for dealing with heterogeneous data. By constructing a kernel-based pseudo covariate and response, we transform the non-smooth quantile regression problem into a smooth least squares problem. Based on this procedure, we establish a double-smoothing surrogate likelihood framework to facilitate distributed learning. An efficient algorithm is developed, which enjoys computation and communication efficiency. Both theoretical analysis and simulation studies demonstrated the superior performance of the proposed algorithm. As future work, we will investigate the inference problem in distributed high-dimensional quantile regression based on the debiased procedure \citep{zhao2014general,bradic2017uniform}. It is also interesting to develop a decentralized distributed algorithm for quantile regression over a network data structure. 

\section*{Acknowledgments}
The authors thank the area chair and the anonymous referees for their constructive suggestions, which significantly improved this article. This research is supported in part by NSFC-12371270 and Shanghai Science and Technology Development Funds (23JC1402100).  This
research is also supported by Shanghai Research Center for Data Science and Decision Technology.

\section*{Impact Statement}
This paper presents work whose goal is to advance the field of Machine Learning. There are many potential societal consequences of our work, none which we feel must be specifically highlighted here.



\bibliography{main.bib}
\bibliographystyle{icml2024}

\newpage
\appendix
\onecolumn

\newpage

\section{Proof of the Main Results}\label{sec1}

\subsection{Some Technical Lemmas}
In this section, we present some technical lemmas that will be used in our proof of main theorems.
\begin{lemma}\label{lem1}
Let $\xi_1, \ldots, \xi_N$ be independent random variables with zero mean. Suppose that there exists some $t>0$ and $\bar{B}_N$ such that $\sum_{i=1}^N \mathbb{E} \xi_i^2 \exp \left(t\left|\xi_i\right|\right)\leq \bar{B}_N^2$. Then for any $0<x \leq \bar{B}_N$,
\begin{align*}
{\mathbb{P}}\left(\sum_{i=1}^N \xi_i \geq C_t \bar{B}_N x\right) \leq \exp \left(-x^2\right),
\end{align*}
where $C_t=t+t^{-1}$.
\end{lemma}
\noindent Lemma \ref{lem1} is the same as Lemma 1 of \citet{cai2011adaptive} and thus we omit its proof.
\vspace{0.5cm}

\begin{lemma}\label{lem2}
Under Assumption \ref{ass1}-\ref{ass3}, if $|\widehat{\boldsymbol{\beta}}_{0}-\boldsymbol{\beta}^*|_2=\mathcal{O}_\mathbb{P}(a_N)$ and $\mathbb{P}(\operatorname{supp}(\widehat{\boldsymbol{\beta}}_{0}) \subseteq S) \rightarrow 1$, then we have 
\begin{align*}
\left\|\widehat{\boldsymbol{D}}_{h,S \times S}-\boldsymbol{I}_{S \times S}\right\|_{op}=\mathcal{O}_{\mathbb{P}}\left(\sqrt{\frac{s\log N}{Nh}}+a_N+h\right).
\end{align*}
\end{lemma}
\begin{proof}
Denote $\widehat{S}=\operatorname{supp}(\widehat{\boldsymbol{\beta}}_{0})$ and let 
\begin{align*}
\widehat{\boldsymbol{D}}_{h,S \times S}(\boldsymbol{\beta})=\frac{1}{N}\sum_{i=1}^N\boldsymbol{X}_{i,S}\boldsymbol{X}_{i,S}^{\mathrm{T}}K_h(Y_i-\boldsymbol{X}_{i,S}^{\mathrm{T}}\boldsymbol{\beta}_S).
\end{align*}
Since $\mathbb{P}(\operatorname{supp}(\widehat{\boldsymbol{\beta}}_{0}) \subseteq S) \rightarrow 1$, we have $\widehat{\boldsymbol{\beta}}_{0,S^c}=0$ with high probability. And note the fact $\boldsymbol{\beta}^*_{S^c}=0$, so we have $|\widehat{\boldsymbol{\beta}}_{0,S}-\boldsymbol{\beta}^*_{S}|_2=\mathcal{O}_\mathbb{P}(a_N)$. To prove the lemma, without loss of generality, we assume that $|\widehat{\boldsymbol{\beta}}_{0,S}-\boldsymbol{\beta}^*_{S}|_2\leq a_N$ and $\widehat{S} \subseteq S$.

From the proof of Lemma 5 in \citet{cai2010optimal}, there exists $\boldsymbol{v}_1,\ldots,\boldsymbol{v}_{5^s}\in S^{s-1}$ such that
\begin{equation}\label{A.1}
\left\|\widehat{\boldsymbol{D}}_{h,S \times S}(\boldsymbol{\beta})-\boldsymbol{I}_{S \times S}\right\|_{op} \leq 4 \underset{j \leq 5^s}{\operatorname{sup}}\left|\boldsymbol{v}_j^{\mathrm{T}}\{\widehat{\boldsymbol{D}}_{h,S \times S}(\boldsymbol{\beta})-\boldsymbol{I}_{S \times S}\}\boldsymbol{v}_j\right|. 
\end{equation}
So if $|\widehat{\boldsymbol{\beta}}_{0,S}-\boldsymbol{\beta}^*_{S}|_2\leq a_N$, we have 
\begin{equation}\label{A.2}
\underset{j \leq 5^s}{\operatorname{sup}}\left|\boldsymbol{v}_j^{\mathrm{T}}\{\widehat{\boldsymbol{D}}_{h,S \times S}-\boldsymbol{I}_{S \times S}\}\boldsymbol{v}_j\right|\leq \underset{j \leq 5^s}{\operatorname{sup}}\underset{|\boldsymbol{\beta}_{S}-\boldsymbol{\beta}^*_{S}|_2\leq a_N}{\operatorname{sup}}\left|\boldsymbol{v}_j^{\mathrm{T}}\{\widehat{\boldsymbol{D}}_{h,S \times S}(\boldsymbol{\beta})-\boldsymbol{I}_{S \times S}\}\boldsymbol{v}_j\right| . 
\end{equation}

Denote $\boldsymbol{\beta}_S^*=\left(\beta_1^*, \ldots, \beta_s^*\right)^{\mathrm{T}}$, for every $i$, we divide the interval $\left[\beta_i^*-a_N, \beta_i^*+a_N\right]$ into $N^{C_1}$ small subintervals and each has length $2a_N/N^{C_1}$, where $C_1$ is a large positive constant. Therefore, there exists a set of points in $\mathbb{R}^{p+1},\left\{\boldsymbol{\beta}_k, 1 \leq k  \leq N^{C_1s}\right\}$ , such that for any $\boldsymbol{\beta}$ in the ball $\left|\boldsymbol{\beta}_S-\boldsymbol{\beta}_S^*\right|_2 \leq a_N$, we have $\left|\boldsymbol{\beta}_S-\boldsymbol{\beta}_{k, S}\right|_2 \leq 2 \sqrt{s} a_N / N^{C_1}$ for some $1 \leq k \leq N^{C_1s}$ and $\left|\boldsymbol{\beta}_{k, S}-\boldsymbol{\beta}_S^*\right|_2 \leq a_N$. Thus we have 
\begin{align*}
\left|K_h(Y_i-\boldsymbol{X}_{i,S}^{\mathrm{T}}\boldsymbol{\beta}_S)-K_h(Y_i-\boldsymbol{X}_{i,S}^{\mathrm{T}}\boldsymbol{\beta}_{k,S})\right|&=
\left|\frac{1}{h} K\left(\frac{Y_i-\boldsymbol{X}_{i,S}^{\mathrm{T}}\boldsymbol{\beta}_{S}}{h}\right)-\frac{1}{h} K\left(\frac{Y_i-\boldsymbol{X}_{i,S}^{\mathrm{T}}\boldsymbol{\beta}_{k,S}}{h}\right)\right| \\
&\leq C h^{-2}\left|\boldsymbol{X}_{i, S}^{\mathrm{T}}\left(\boldsymbol{\beta}_S-\boldsymbol{\beta}_{k, S}\right)\right|.
\end{align*}
This yields that 
\begin{align*}
&\underset{j \leq 5^s}{\operatorname{sup}}\underset{|\boldsymbol{\beta}_{S}-\boldsymbol{\beta}^*_{S}|_2\leq a_N}{\operatorname{sup}}\left|\boldsymbol{v}_j^{\mathrm{T}}\{\widehat{\boldsymbol{D}}_{h,S \times S}(\boldsymbol{\beta})-\boldsymbol{I}_{S \times S}\}\boldsymbol{v}_j\right|-\underset{j \leq 5^s}{\operatorname{sup}}\underset{1
\leq k \leq N^{C_1s}}{\operatorname{sup}}\left|\boldsymbol{v}_j^{\mathrm{T}}\{\widehat{\boldsymbol{D}}_{h,S \times S}(\boldsymbol{\beta}_k)-\boldsymbol{I}_{S \times S}\}\boldsymbol{v}_j\right|\\
&\leq \frac{Ca_N\sqrt{s}}{N^{C_1+1}h^2}\sum_{i=1}^N|\boldsymbol{X}_{i,S}|_2^3.
\end{align*}
Since $\max_{i,j}\mathbb{E}|X_{i,j}|^3\leq \infty$, then for any $\gamma>0$, by letting $C_1$ large enough, we have
\begin{equation}\label{A.3}
\underset{j \leq 5^s}{\operatorname{sup}}\underset{|\boldsymbol{\beta}_{S}-\boldsymbol{\beta}^*_{S}|_2\leq a_N}{\operatorname{sup}}\left|\boldsymbol{v}_j^{\mathrm{T}}\{\widehat{\boldsymbol{D}}_{h,S \times S}(\boldsymbol{\beta})-\boldsymbol{I}_{S \times S}\}\boldsymbol{v}_j\right|-\underset{j \leq 5^s}{\operatorname{sup}}\underset{1
\leq k \leq N^{c_1s}}{\operatorname{sup}}\left|\boldsymbol{v}_j^{\mathrm{T}}\{\widehat{\boldsymbol{D}}_{h,S \times S}(\boldsymbol{\beta}_k)-\boldsymbol{I}_{S \times S}\}\boldsymbol{v}_j\right|=\mathcal{O}_\mathbb{P}(N^{-\gamma}).
\end{equation}
On the other hand, we have
\begin{align*}
\mathbb{E}\left(\boldsymbol{X}_{i,S}\boldsymbol{X}_{i,S}^{\mathrm{T}}K_h(Y_i-\boldsymbol{X}_{i,S}^{\mathrm{T}}\boldsymbol{\beta}_S)|\boldsymbol{X}_{i}\right)&= \mathbb{E}\left(\left.\boldsymbol{X}_{i,S}\boldsymbol{X}_{i,S}^{\mathrm{T}}\frac{1}{h}K\left(\frac{Y_i-\boldsymbol{X}_{i,S}^{\mathrm{T}}\boldsymbol{\beta}_S}{h}\right)\right|\boldsymbol{X}_{i}\right)  \\
&=\boldsymbol{X}_{i,S}\boldsymbol{X}_{i,S}^{\mathrm{T}}\int_{-\infty}^{\infty} K(x) f_{\varepsilon \mid \boldsymbol{X}}\left\{h x+\boldsymbol{X}_{i, S}^{\mathrm{T}}\left(\boldsymbol{\beta}_S-\boldsymbol{\beta}_S^*\right)\right\} d x\\
&=\boldsymbol{X}_{i,S}\boldsymbol{X}_{i,S}^{\mathrm{T}}\left(f_{\varepsilon \mid \boldsymbol{X}}(0)+\mathcal{O}\left(h+\boldsymbol{X}_{i, S}^{\mathrm{T}}\left(\boldsymbol{\beta}_S-\boldsymbol{\beta}_S^*\right)\right)\right).
\end{align*}
Under Assumption \ref{ass3}, we have $\sup _{|\boldsymbol{\alpha}|_2=1} \mathbb{E}\left|\boldsymbol{\alpha}^{\mathrm{T}} \boldsymbol{X}\right| \leq C$. So we can conclude that 
\begin{equation}\label{A.4}
\left|\mathbb{E}(\boldsymbol{v}_j^{\mathrm{T}}\{\widehat{\boldsymbol{D}}_{h,S \times S}(\boldsymbol{\beta}_k)-\boldsymbol{I}_{S \times S}\}\boldsymbol{v}_j)\right|=\mathcal{O}(h+a_N). 
\end{equation}

It remains to bound $\underset{1
\leq k \leq N^{C_1s}}{\operatorname{sup}}\left|\boldsymbol{v}_j^{\mathrm{T}}\widehat{\boldsymbol{D}}_{h,S \times S}(\boldsymbol{\beta}_k)\boldsymbol{v}_j-\mathbb{E}\{\boldsymbol{v}_j^{\mathrm{T}}\widehat{\boldsymbol{D}}_{h,S \times S}(\boldsymbol{\beta}_k)\boldsymbol{v}_j\}\right|.$ Let 
\begin{align*}
\xi_{i,j,k}=\left(\boldsymbol{v}_j^{\mathrm{T}}\boldsymbol{X}_{i,S}\right)^2K\left(\frac{Y_i-\boldsymbol{X}_{i,S}^{\mathrm{T}}\boldsymbol{\beta}_{k,S}}{h}\right),
\end{align*}
it's easy to verify that
\begin{align*}
\mathbb{E}(\xi_{i,j,k}^2|\boldsymbol{X}_{i})=\left(\boldsymbol{v}_j^{\mathrm{T}}\boldsymbol{X}_{i,S}\right)^4h\int_{-\infty}^{\infty} \{K(x)\}^2 f_{\varepsilon \mid \boldsymbol{X}}\left\{h x+\boldsymbol{X}_{i, S}^{\mathrm{T}}\left(\boldsymbol{\beta}_{k,S}-\boldsymbol{\beta}_S^*\right)\right\}dx.
\end{align*}
Under Assumption \ref{ass1}-\ref{ass3}, so we have $\mathbb{E}(\xi_{i,j,k}^2)\leq Ch$. Then according to the Lemma \ref{lem1}, for any $\gamma>0$, there exist a constant $C$ such that 
\begin{align*}
\underset{1
\leq k \leq N^{C_1s}}{\operatorname{sup}} {\mathbb{P}}\left(\left|\sum_{i=1}^N\left(\xi_{i,j, k}-\mathbb{E} \xi_{i,j, k}\right)\right| \geq C \sqrt{N h s \log N}\right)=\mathcal{O}\left(N^{-\gamma s}\right) .
\end{align*}
By letting $\gamma>C_1$, we can obtain that 
\begin{equation}\label{A.5}
\underset{1
\leq k \leq N^{C_1s}}{\operatorname{sup}}\left|\boldsymbol{v}_j^{\mathrm{T}}\widehat{\boldsymbol{D}}_{h,S \times S}(\boldsymbol{\beta}_k)\boldsymbol{v}_j-\mathbb{E}\{\boldsymbol{v}_j^{\mathrm{T}}\widehat{\boldsymbol{D}}_{h,S \times S}(\boldsymbol{\beta}_k)\boldsymbol{v}_j\}\right|=\mathcal{O}_\mathbb{P}(\sqrt{\frac{s\log N}{Nh}}).  
\end{equation}
Thus by \eqref{A.1}-\eqref{A.5}, we have 
\begin{align*}
\left\|\widehat{\boldsymbol{D}}_{h,S \times S}-\boldsymbol{I}_{S \times S}\right\|_{op}=\mathcal{O}_{\mathbb{P}}\left(\sqrt{\frac{s\log N}{Nh}}+a_N+h\right).
\end{align*}
This completes the proof.
\end{proof}

\begin{lemma}\label{lem3}
$$
\max _{1 \leq j \leq p}\left\|\frac{1}{N} \sum_{i=1}^N\left|X_{i j}\right| \boldsymbol{X}_{i, S} \boldsymbol{X}_{i, S}^{\mathrm{T}}\right\|_{op}=\mathcal{O}_\mathbb{P}(1).
$$
\end{lemma}
\begin{proof}
From the proof of Lemma 5 in \citet{cai2010optimal}, there exists $\boldsymbol{v}_1,\ldots,\boldsymbol{v}_{5^s}\in S^{s-1}$ such that
\begin{align*}
\max _{1 \leq j \leq p}\left\|\frac{1}{N} \sum_{i=1}^N\left|X_{i j}\right| \boldsymbol{X}_{i, S} \boldsymbol{X}_{i, S}^{\mathrm{T}}\right\|_{op}\leq \max _{1 \leq j \leq p}\max _{1 \leq k \leq 5^s}\frac{1}{N} \sum_{i=1}^N\left|X_{i j}\right|\left(\boldsymbol{v}_k^{\mathrm{T}}\boldsymbol{X}_{i, S}\right).
\end{align*}
Denote $\widetilde{X}_{i j} =X_{i j} I\left(\left|X_{i j}\right| \leq \log N\right)$,
under Assumption \ref{ass3}, we can alternatively show that
\begin{align*}
\max _{1 \leq j \leq p}\max _{1 \leq k \leq 5^s}\frac{1}{N} \sum_{i=1}^N\left|\widetilde{X}_{i j}\right|\left(\boldsymbol{v}_k^{\mathrm{T}}\boldsymbol{X}_{i, S}\right)=\mathcal{O}_\mathbb{P}(1).
\end{align*}
Set $Y_{ijk}=\left|\widetilde{X}_{i j}\right|\left(\boldsymbol{v}_k^{\mathrm{T}} \boldsymbol{X}_{i, \mathcal{S}}\right)^2 I\left\{\left|\widetilde{X}_{i j}\right|\left(\boldsymbol{v}_k^{\mathrm{T}} \boldsymbol{X}_{i, \mathcal{S}}\right)^2 \geq(s+1)\log^3 N\right\}$, note that 
\begin{align*}
Np5^s\max _{i, j, k} {\mathbb{P}}\left(\left|\widetilde{X}_{i j}\right|\left(\boldsymbol{v}_k^{\mathrm{T}} \boldsymbol{X}_{i, \mathcal{S}}\right)^2 I\left\{\left|\widetilde{X}_{i j}\right|\left(\boldsymbol{v}_k^{\mathrm{T}} \boldsymbol{X}_{i, \mathcal{S}}\right)^2 \geq(s+1)\log^3 N\right\}\right)=o(1),
\end{align*}
then it suffices to show that
\begin{align*}
\max _{1 \leq j \leq p}\max _{1 \leq k \leq 5^s}\frac{1}{N} \sum_{i=1}^NY_{ijk}=\mathcal{O}_\mathbb{P}(1).
\end{align*}
It is easy to verify that $\mathbb{E} Y_{ i jk} \leq \mathbb{E}\left|X_{i j}\right|\left(\boldsymbol{v}_k^{\mathrm{T}} \boldsymbol{X}_{i, S}\right)^2 \leq C\left(\mathbb{E} X_{i j}^2\right)^{1 / 2} \sup _{|\boldsymbol{v}|_2=1}\sqrt{\mathbb{E}\left(\boldsymbol{v}^{\mathrm{T}} \boldsymbol{X}_{i, S}\right)^4}=\mathcal{O}(1)$ and similarly, $\mathbb{E} Y_{ i jk}^2=\mathcal{O}(1)$, uniformly in $i, j,k$. By Bernstein's inequality,
\begin{align*}
{\mathbb{P}}\left(\left|\frac{1}{N} \sum_{i=1}^N\left(Y_{ i jk}-\mathbb{E} Y_{ i jk}\right)\right| \geq 1\right) \leq e^{-q_1 N}+e^{-q_2 \frac{N}{(s+1)\log^3 N}},
\end{align*}
where $q_1,q_2>0$ are constants. Since $s=O\left(n^r\right)$ for some $0<r<1 / 3$, we have
\begin{align*}
Np 9^s\left(e^{-q_1 N}+e^{-q_2 \frac{N}{(s+1)\log^3 n}}\right)=o(1) .
\end{align*}
This proves $\max _{1 \leq j \leq p} \max _{1\leq k\leq5^s} \frac{1}{N} \sum_{i=1}^N Y_{ i jk}=\mathcal{O}_\mathbb{P}(1)$. This completes the proof.
\end{proof}

\subsection{Proof of Proposition \ref{prop1} and Proposition \ref{prop2}} 
We first give some propositions to prove Theorem \ref{thm1} and Theorem \ref{thm2}. Consider a general convex quadratic optimization as follows,
\begin{equation}\label{opt:lasso}
   \widehat{\boldsymbol{\beta}}=\underset{\boldsymbol{\beta} \in \mathbb{R}^{p+1}}{\arg \min } \frac{1}{2} \boldsymbol{\beta}^{\mathrm{T}} \boldsymbol{A} \boldsymbol{\beta}-\boldsymbol{\beta}^{\mathrm{T}} \boldsymbol{a}+\lambda_N|\boldsymbol{\beta}|_1, 
\end{equation}
where $\boldsymbol{A}$ is a non-negative definite matrix and $\boldsymbol{a}$ is a vector in $\mathbb{R}^{p+1}$, then we can prove the following proposition.

\begin{proposition}\label{prop1}
For the convex quadratic optimization in  \eqref{opt:lasso}, define $\mathcal{A} =\left\{\delta \in \mathbb{R}^p:|\delta|_1 \leq 4 s^{1 / 2}|\delta|_2\right\}$,
if the following conditions 
\begin{align*}
\left|\boldsymbol{A} \boldsymbol{\beta}^*-\boldsymbol{a}\right|_{\infty} \leq \lambda_N / 2, 
\end{align*}
\begin{align*}
\delta^{\mathrm{T}} \boldsymbol{A} \delta \geq \alpha|\delta|_2^2, \quad \forall \delta \in \mathcal{A},
\end{align*}
hold for some constant $\alpha$. Then we have 
\begin{align*}
|\widehat{\boldsymbol{\beta}}-\boldsymbol{\beta}^*|_2 \leq 6 \alpha^{-1} \sqrt{s}  \lambda_N ,
\end{align*}
where $s$ is the sparsity of $\boldsymbol{\beta}^*$.
\end{proposition}
\begin{proof}
We first prove that $|\widehat{\boldsymbol{\beta}}-\boldsymbol{\beta}^*|_1 \leq 4 \sqrt{s}|\widehat{\boldsymbol{\beta}}-\boldsymbol{\beta}^*|_2$, which implies that $\widehat{\boldsymbol{\beta}}-\boldsymbol{\beta}^* \in \mathcal{A}$. Due to the fact that $\widehat{\boldsymbol{\beta}}$ minimizes \eqref{opt:lasso}, we have 
\begin{align*}
\frac{1}{2} \widehat{\boldsymbol{\beta}}^{\mathrm{T}} \boldsymbol{A} \widehat{\boldsymbol{\beta}}-\widehat{\boldsymbol{\beta}}^{\mathrm{T}} \boldsymbol{a}-(\frac{1}{2} \boldsymbol{\beta}^{* \mathrm{~T}} \boldsymbol{A} \boldsymbol{\beta}^*-\boldsymbol{\beta}^{* \mathrm{~T}} \boldsymbol{a}) & \leq \lambda_N\left(\left|\boldsymbol{\beta}^*\right|_1-|\widehat{\boldsymbol{\beta}}|_1\right) \\
&=\lambda_N\left(\left|\boldsymbol{\beta}_S^*\right|_1-|\widehat{\boldsymbol{\beta}}_S|_1-|\widehat{\boldsymbol{\beta}}_{S^c}|_1\right) \\
& \leq \lambda_N|(\boldsymbol{\beta}^*-\widehat{\boldsymbol{\beta}})_S|_1-\lambda_N|(\boldsymbol{\beta}^*-\widehat{\boldsymbol{\beta}})_{S^c}|_1,  
\end{align*}
where $S$ is the support of $\boldsymbol{\beta}^*$. On the other hand, since $\boldsymbol{A}$ is non-negative, we have 
\begin{align*}
\frac{1}{2} \widehat{\boldsymbol{\beta}}^{\mathrm{T}} \boldsymbol{A} \widehat{\boldsymbol{\beta}}-\widehat{\boldsymbol{\beta}}^{\mathrm{T}} \boldsymbol{a}-(\frac{1}{2} \boldsymbol{\beta}^{* \mathrm{~T}} \boldsymbol{A} \boldsymbol{\beta}^*-\boldsymbol{\beta}^{* \mathrm{~T}} \boldsymbol{a}) & \geq(\widehat{\boldsymbol{\beta}}-\boldsymbol{\beta}^*)^{\mathrm{T}}\left(\boldsymbol{A} \boldsymbol{\beta}^*-\boldsymbol{a}\right) \\
& \geq-|\widehat{\boldsymbol{\beta}}-\boldsymbol{\beta}^*|_1\left|\boldsymbol{A} \boldsymbol{\beta}^*-\boldsymbol{a}\right|_{\infty} \\
& \geq-\lambda_N|\widehat{\boldsymbol{\beta}}-\boldsymbol{\beta}^*|_1 / 2 .
\end{align*}
Combine the two inequalities, we can obtain that $|(\widehat{\boldsymbol{\beta}}-\boldsymbol{\beta}^*)_{S^c}|_1 \leq 3|(\widehat{\boldsymbol{\beta}}-\boldsymbol{\beta}^*)_S|_1$, so we have 
\begin{align*}
|\widehat{\boldsymbol{\beta}}-\boldsymbol{\beta}^*|_1 \leq 4|(\widehat{\boldsymbol{\beta}}-\boldsymbol{\beta}^*)_S|_1 \leq 4 \sqrt{s}|(\widehat{\boldsymbol{\beta}}-\boldsymbol{\beta}^*)_S|_2 \leq 4 \sqrt{s}|\widehat{\boldsymbol{\beta}}-\boldsymbol{\beta}^*|_2.
\end{align*}

There exists subgradient $\boldsymbol{Z}$ of the lasso penalty such that $\boldsymbol{Z} \in \partial|\boldsymbol{\beta}|_1$, i.e. $\boldsymbol{Z}_j=\operatorname{sign}\left(\boldsymbol{\beta}_j\right)$ for $\boldsymbol{\beta}_j \neq 0$ and $\boldsymbol{Z}_j \in[-1,1]$ otherwise. Then according to the first order condition of \eqref{opt:lasso}, we have $\boldsymbol{A} \widehat{\boldsymbol{\beta}}-\boldsymbol{a} \in \lambda_N \partial|\widehat{\boldsymbol{\beta}}|_1\leq \lambda_N$. Together with the first condition we have
\begin{align*}
|\boldsymbol{A}(\widehat{\boldsymbol{\beta}}-\boldsymbol{\beta}^*)|_{\infty} \leq 3 \lambda_N/2.
\end{align*}
Since $\widehat{\boldsymbol{\beta}}-\boldsymbol{\beta}^* \in \mathcal{A}$, from second condition, we have 
\begin{align*}
|\widehat{\boldsymbol{\beta}}-\boldsymbol{\beta}^*|_2^2\leq \alpha^{-1}(\widehat{\boldsymbol{\beta}}-\boldsymbol{\beta}^*)^{\mathrm{T}}\boldsymbol{A}(\widehat{\boldsymbol{\beta}}-\boldsymbol{\beta}^*)\leq \frac{3}{2}\alpha^{-1}\lambda_N|\widehat{\boldsymbol{\beta}}-\boldsymbol{\beta}^*|_1\leq 6 \alpha^{-1}\sqrt{s}\lambda_N|\widehat{\boldsymbol{\beta}}-\boldsymbol{\beta}^*|_2,
\end{align*}
which yields that $|\widehat{\boldsymbol{\beta}}-\boldsymbol{\beta}^*|_2\leq6 \alpha^{-1}\sqrt{s}\lambda_N$. This completes the proof.
\end{proof}

\begin{proposition}\label{prop2}
Under Assumption \ref{ass1}-\ref{ass6}, if $|\widehat{\boldsymbol{\beta}}_0-\boldsymbol{\beta}^*|_2=\mathcal{O}_\mathbb{P}(a_N)$ and $ h \asymp\left(s \log N / N\right)^{1 / 3} $, then we have 
\begin{align*}
\left|\boldsymbol{z}_N-\widehat{\boldsymbol{D}}_h\boldsymbol{\beta}^*\right|_{\infty}=\mathcal{O}_{\mathbb{P}}\left(\sqrt{\frac{\log N}{N}}+a_N^2+a_Nh\right).
\end{align*}
\end{proposition}
\begin{proof}
It's easy to verify that
\begin{align*}
\left|\boldsymbol{z}_N-\widehat{\boldsymbol{D}}_h\boldsymbol{\beta}^*\right|_{\infty} &=\left|\frac{1}{N}\sum_{i=1}^N\widetilde{\boldsymbol{X}}_{i,h}^{\mathrm{T}}\widetilde{Y}_{i,h}-\frac{1}{N}\sum_{i=1}^N\widetilde{\boldsymbol{X}}_{i,h}
\widetilde{\boldsymbol{X}}_{i,h}^{\mathrm{T}}\boldsymbol{\beta}^*\right|_{\infty}\\
&=\left|-\frac{1}{N}\sum_{i=1}^N\boldsymbol{X}_i\left\{I\left(Y_i-\boldsymbol{X}_i^{\mathrm{T}}\widehat{\boldsymbol{\beta}}_0\right)-\tau\right\}+\boldsymbol{D}_h\left(\widehat{\boldsymbol{\beta}}_0-\boldsymbol{\beta}^*\right)\right|_{\infty}.
\end{align*}
Denote 
\begin{align*}
\boldsymbol{R}_N(\boldsymbol{\beta})&=\frac{1}{N}\sum_{i=1}^N[\boldsymbol{X}_i I[\varepsilon_i \leq \boldsymbol{X}_{i,S}^{\mathrm{T}}(\boldsymbol{\beta}_S-\boldsymbol{\beta}_S^*)  ]-\boldsymbol{X}_iF_{\varepsilon|\boldsymbol{X}}(\boldsymbol{X}_{i,S}^{\mathrm{T}}(\boldsymbol{\beta}_S-\boldsymbol{\beta}_S^*) )]-\frac{1}{N}\sum_{i=1}^N[\boldsymbol{X}_i I[\varepsilon_i \leq 0  ]-\boldsymbol{X}_iF_{\varepsilon|\boldsymbol{X}}(0)],   
\end{align*}
then we decompose $\left|\boldsymbol{z}_N-\widehat{\boldsymbol{D}}_h\boldsymbol{\beta}^*\right|_{\infty}$ as 
\begin{align*}
\left|\boldsymbol{z}_N-\widehat{\boldsymbol{D}}_h\boldsymbol{\beta}^*\right|_{\infty}&\leq \left|\boldsymbol{R}_N(\widehat{\boldsymbol{\beta}}_0)\right|_{\infty}+\left|\frac{1}{N}\sum_{i=1}^N\boldsymbol{X}_i\left\{F_{\varepsilon|\boldsymbol{X}}(\boldsymbol{X}_{i,S}^{\mathrm{T}}(\widehat{\boldsymbol{\beta}}_{0,S}-\boldsymbol{\beta}_S^* ))-F_{\varepsilon|\boldsymbol{X}}(0)\right\}-\boldsymbol{I}\left(\widehat{\boldsymbol{\beta}}_0-\boldsymbol{\beta}^*\right)\right|_{\infty}\\
&+\left|\frac{1}{N}\sum_{i=1}^N\left[\boldsymbol{X}_iI[\varepsilon_i \leq 0  ]-\boldsymbol{X}_iF_{\varepsilon|\boldsymbol{X}}(0)\right]\right|_{\infty}+\left|\left(\boldsymbol{D}_h-\boldsymbol{I}\right)\left(\widehat{\boldsymbol{\beta}}_0-\boldsymbol{\beta}^*\right)\right|_{\infty}.
\end{align*}
For the first term, from lemma 1 in \citet{chen2020distributed}, we have 
\begin{equation}\label{A.6}
\left|\boldsymbol{R}_N(\widehat{\boldsymbol{\beta}}_0)\right|_{\infty}=\mathcal{O}_{\mathbb{P}}\left(\sqrt{\frac{sa_N\log N}{N}}\right).  
\end{equation}
For the second term, from Taylor expansion, we have 
\begin{align*}
&\frac{1}{N}\sum_{i=1}^N\boldsymbol{X}_i\left\{F_{\varepsilon|\boldsymbol{X}}(\boldsymbol{X}_{i,S}^{\mathrm{T}}(\widehat{\boldsymbol{\beta}}_{0,S}-\boldsymbol{\beta}_S^* ))-F_{\varepsilon|\boldsymbol{X}}(0)\right\}-\boldsymbol{I}\left(\widehat{\boldsymbol{\beta}}_0-\boldsymbol{\beta}^*\right)\\
&=\frac{f_{\varepsilon|\boldsymbol{X}}(0)}{N} \sum_{i=1}^N  \boldsymbol{X}_i \boldsymbol{X}_{i,S}^{\mathrm{T}}(\widehat{\boldsymbol{\beta}}_{0,S}-\boldsymbol{\beta}_S^* )+\frac{C}{N} \sum_{i=1}^N\boldsymbol{X}_i\left\{\boldsymbol{X}_{i,S}^{\mathrm{T}}(\widehat{\boldsymbol{\beta}}_{0,S}-\boldsymbol{\beta}_S^* )\right\}^2-\boldsymbol{I}\left(\widehat{\boldsymbol{\beta}}_0-\boldsymbol{\beta}^*\right)\\
&=\left(f_{\varepsilon|\boldsymbol{X}}(0)\widehat{\boldsymbol{\Sigma}}-\boldsymbol{I}\right)\left(\widehat{\boldsymbol{\beta}}_0-\boldsymbol{\beta}^*\right)+\frac{C}{N} \sum_{i=1}^N\boldsymbol{X}_i\left\{\boldsymbol{X}_{i,S}^{\mathrm{T}}(\widehat{\boldsymbol{\beta}}_{0,S}-\boldsymbol{\beta}_S^* )\right\}^2,
\end{align*}
where the second equality is from the fact that $\boldsymbol{\beta}^*_{S^c}=\boldsymbol{0}$ and $\widehat{\boldsymbol{\beta}}_{0,S^c}=\boldsymbol{0}$.  By Lemma B.3 of \citet{huang2018robust} and Assumption \ref{ass2}, it is easy to verify that 
\begin{align*}
\left|f_{\varepsilon|\boldsymbol{X}}(0)\widehat{\boldsymbol{\Sigma}}-\boldsymbol{I}\right|_{\infty}\leq C\sqrt{\frac{\log p}{N}}.
\end{align*}
holds with probability at least $1-2/p^2$. Together with Lemma \ref{lem3}, we have 
\begin{equation}\label{A.7}
\left|\frac{1}{N}\sum_{i=1}^N\boldsymbol{X}_i\left\{F_{\varepsilon|\boldsymbol{X}}(\boldsymbol{X}_{i,S}^{\mathrm{T}}(\widehat{\boldsymbol{\beta}}_{0,S}-\boldsymbol{\beta}_S^* ))-F_{\varepsilon|\boldsymbol{X}}(0)\right\}-\boldsymbol{I}\left(\widehat{\boldsymbol{\beta}}_0-\boldsymbol{\beta}^*\right)\right|_{\infty}=\mathcal{O}_{\mathbb{P}}\left(\sqrt{\frac{s\log p}{N}}a_N+a_N^2\right).    
\end{equation}
For the third term, from Hoeffding inequality (see, e.g., Proposition 2.5 in \citet{wainwright2019high}), we have 
\begin{equation}\label{A.8}
\left|\frac{1}{N}\sum_{i=1}^N\left[\boldsymbol{X}_iI[\varepsilon_i \leq 0  ]-\boldsymbol{X}_iF_{\varepsilon|\boldsymbol{X}}(0)\right]\right|_{\infty}=\mathcal{O}_{\mathbb{P}}\left(\sqrt{\frac{\log N}{N}}\right). 
\end{equation}
For the last term, recall that we have $\widehat{\boldsymbol{\beta}}_{0, S^c}=0$ with high probability. Due to the fact that $\boldsymbol{\beta}_{S^c}^*=0$, by $|\boldsymbol{\beta}^*-\widehat{\boldsymbol{\beta}}_0|_2=\mathcal{O}_{\mathbb{P}}\left(a_N\right)$ and Lemma \ref{lem2}, we have
\begin{align}\label{A.13}
\left|(\widehat{\boldsymbol{D}}_h-\boldsymbol{I})(\widehat{\boldsymbol{\beta}}_0-\boldsymbol{\beta}^*)\right|_{\infty}&\leq  \left\|(\widehat{\boldsymbol{D}}_{h,S \times S}-\boldsymbol{I}_{S \times S})\right\|_{op}\left|(\widehat{\boldsymbol{\beta}}_{0,S}-\boldsymbol{\beta}^*_S)\right|_2 =\mathcal{O}_{\mathbb{P}}\left(a_N\sqrt{\frac{s\log N}{Nh}}+a_N^2+a_Nh\right) .
\end{align}
Combine \eqref{A.6}-\eqref{A.13}, we have 
\begin{equation}\label{A.14}
\left|\widehat{\boldsymbol{D}}_h\boldsymbol{\beta}^*-\boldsymbol{z}_N\right|_{\infty}=\mathcal{O}_{\mathbb{P}}\left(\sqrt{\frac{\log N}{N}}+\sqrt{\frac{sa_N\log N}{N}}+a_N\left\{\sqrt{\frac{s\log p}{N}}+\sqrt{\frac{s\log N}{Nh}}+a_N+h\right\}\right).  
\end{equation}
Since $ h \asymp\left(s \log N / N\right)^{1 / 3} $ and $sa_N=\mathcal{O}(1)$, we further simplify the result as
\begin{align*}\
\left|\widehat{\boldsymbol{D}}_h\boldsymbol{\beta}^*-\boldsymbol{z}_N\right|_{\infty}=\mathcal{O}_{\mathbb{P}}\left(\sqrt{\frac{\log N}{N}}+a_N^2+a_Nh\right). 
\end{align*}
This completes the proof. 
\end{proof}

\subsection{Proof of Theorem \ref{thm1} and Theorem \ref{thm2}} \label{secB.4}
Now we are ready to prove Theorem \ref{thm1} by using the aforementioned lemmas and propositions.

\begin{proof}
We use proposition \ref{prop1} to prove Theorem \ref{thm1}. Let $\boldsymbol{A}=\widehat{\boldsymbol{D}}_{1,b}$ and $\boldsymbol{a}=\boldsymbol{z}_N+\left(\widehat{\boldsymbol{D}}_{1,b}-\widehat{\boldsymbol{D}}_{h}\right) \widehat{\boldsymbol{\beta}}_{0}$, then we have 
\begin{align}\label{A.15}
\left|\boldsymbol{A}\boldsymbol{\beta}^*-\boldsymbol{a}\right|_{\infty}&=   \left|\widehat{\boldsymbol{D}}_{1,b}\boldsymbol{\beta}^*-\boldsymbol{z}_N-\left(\widehat{\boldsymbol{D}}_{1,b}-\widehat{\boldsymbol{D}}_{h}\right) \widehat{\boldsymbol{\beta}}_{0}\right|_{\infty} \leq \left|\widehat{\boldsymbol{D}}_{h}\boldsymbol{\beta}^*-\boldsymbol{z}_N\right|_{\infty}+\left|\left(\widehat{\boldsymbol{D}}_{1,b}-\widehat{\boldsymbol{D}}_{h}\right) \left(\widehat{\boldsymbol{\beta}}_{0}-\boldsymbol{\beta}^*\right)\right|_{\infty}.
\end{align}

We can bound the first term in \eqref{A.15} by proposition 2 directly. For the second term in \eqref{A.14}, recall that we have $\widehat{\boldsymbol{\beta}}_{0, S^c}=0$ with high probability. Due to the fact that $\boldsymbol{\beta}_{S^c}^*=0$, by $|\boldsymbol{\beta}^*-\widehat{\boldsymbol{\beta}}_0|_2=\mathcal{O}_{\mathbb{P}}\left(a_N\right)$ and Lemma \ref{lem2}, we have
\begin{align*}
\left|\left(\widehat{\boldsymbol{D}}_{1,b}-\widehat{\boldsymbol{D}}_{h}\right) \left(\widehat{\boldsymbol{\beta}}_{0}-\boldsymbol{\beta}^*\right)\right|_{\infty} & \leq \left|\left(\widehat{\boldsymbol{D}}_{1,b}-\boldsymbol{I}\right) \left(\widehat{\boldsymbol{\beta}}_{0}-\boldsymbol{\beta}^*\right)\right|_{\infty} +\left|\left(\widehat{\boldsymbol{D}}_{h}-\boldsymbol{I}\right) \left(\widehat{\boldsymbol{\beta}}_{0}-\boldsymbol{\beta}^*\right)\right|_{\infty}\\
&\leq \left( \left\|(\widehat{\boldsymbol{D}}_{1,b,S \times S}-\boldsymbol{I}_{S \times S})\right\|_{op}+ \left\|(\widehat{\boldsymbol{D}}_{h,S \times S}-\boldsymbol{I}_{S \times S})\right\|_{op}\right)\left|(\widehat{\boldsymbol{\beta}}_{0,S}-\boldsymbol{\beta}^*_S)\right|_2\\
&=\mathcal{O}_{\mathbb{P}}\left(a_N\sqrt{\frac{s\log N}{Nh}}+a_N\sqrt{\frac{s\log n}{nb}}+a_N^2+a_Nh+a_Nb\right).
\end{align*}
If $ h \asymp\left(s \log N / N\right)^{1 / 3} $, $b \asymp\left(s \log n / n\right)^{1 / 3}$ and $a_N \asymp \sqrt{s\log N/n}$, then we have 
\begin{equation}\label{A.16}
\left|\left(\widehat{\boldsymbol{D}}_{1,b}-\widehat{\boldsymbol{D}}_{h}\right) \left(\widehat{\boldsymbol{\beta}}_{0}-\boldsymbol{\beta}^*\right)\right|_{\infty}=\mathcal{O}_{\mathbb{P}}\left(a_N \eta\right), 
\end{equation}
where $\eta=\max\left\{\left(\frac{s\log n}{n}\right)^{1/3},\left(\frac{s\log N}{n}\right)^{1/2}\right\}$. Combine \eqref{A.14}-\eqref{A.16}, we can obtain that 
\begin{align}\label{A.17}
\left|\widehat{\boldsymbol{D}}_{1,b}\boldsymbol{\beta}^*-\boldsymbol{z}_N-\left(\widehat{\boldsymbol{D}}_{1,b}-\widehat{\boldsymbol{D}}_{h}\right) \widehat{\boldsymbol{\beta}}_{0}\right|_{\infty}&=\mathcal{O}_{\mathbb{P}}\left(\sqrt{\frac{\log N}{N}}+a_N\eta\right)=\frac{\mathcal{O}_\mathbb{P}(1)}{C_0}\lambda_N, 
\end{align}
this proves the first condition in Proposition \ref{prop1}.

For the second condition in Proposition \ref{prop1}, with probability tending to one, we have
\begin{align*}
\delta^{\mathrm{T}} \widehat{\boldsymbol{D}}_{1,b} \delta & \geq|\delta|_2^2 \Lambda_{\min }(\boldsymbol{I})-|\delta|_1^2\left|\widehat{\boldsymbol{D}}_{1,b}-\boldsymbol{I}\right|_{\infty} \\
& \geq|\delta|_2^2 \Lambda_{\min }(\boldsymbol{I})-s|\delta|_2^2\left|\widehat{\boldsymbol{D}}_{1,b}-\boldsymbol{I}\right|_{\infty} \\
& \geq \alpha|\delta|_2^2,
\end{align*}
for some $\alpha>0$ as  $s=o\left((n / \log N)^{1 / 2}\right)$. Then from Proposition \ref{prop1}, we have 
\begin{align*}
\left|\widehat{\boldsymbol{\beta}}_{1}-\boldsymbol{\beta}^*\right|_2 =\mathcal{O}_{\mathbb{P}}\left(\sqrt{\frac{s\log N}{N}}+\sqrt{s}a_N\eta\right),
\end{align*}
where $\eta=\max\left\{\left(\frac{s\log n}{n}\right)^{1/3},\left(\frac{s\log N}{n}\right)^{1/2}\right\}$. This completes the proof.
\end{proof}

We use Theorem \ref{thm1} and Theorem \ref{thm3} iteratively to prove Theorem \ref{thm2}.

\emph{Proof of Theorem \ref{thm2}.} By iteratively using Theorem \ref{thm1} and Theorem \ref{thm3}, we have 
\begin{align*}
\left|\widehat{\boldsymbol{\beta}}_{t}-\boldsymbol{\beta}^*\right|_2&=\mathcal{O}_{\mathbb{P}}\left(\sqrt{\frac{s\log N}{N}}+s^{1/2}\eta\left|\widehat{\boldsymbol{\beta}}_{t-1}-\boldsymbol{\beta}^*\right|_2\right)\\
&=\mathcal{O}_{\mathbb{P}}\left(\sqrt{\frac{s\log N}{N}}+\max\left\{s^{(5t+3)/6}\left(\frac{\log n}{n}\right)^{(2t+3)/6}, s^{(2t+1)/2}\left(\frac{\log N}{n}\right)^{(t+1)/2}\right\}\right).
\end{align*}
This completes the proof.

\subsection{Proof of Theorem \ref{thm3} and Theorem \ref{thm4}}
We use the primal dual witness (PWD) construction from \citep{wainwright2009sharp} to prove Theorem \ref{thm3}. By Karush-Kuhn-Tucker's Theorem, there exists subgradient $\widehat{\boldsymbol{Z}}$ such that
\begin{align*}
\widehat{\boldsymbol{D}}_{1,b} \widehat{\boldsymbol{\beta}}_1-\left\{\boldsymbol{z}_N+\left(\widehat{\boldsymbol{D}}_{1,b}-\widehat{\boldsymbol{D}}_{h}\right)\widehat{\boldsymbol{\beta}}_0 \right\}+\lambda_N \widehat{\boldsymbol{Z}}=0.
\end{align*}
The PWD method constructs a pair $(\widetilde{\boldsymbol{\beta}},\widetilde{\boldsymbol{Z}})$ by the following steps:
\begin{itemize}
    \item Denote $\widetilde{\boldsymbol{\beta}}$ to be the solution of the following problem:
    \begin{align*}
    \widetilde{\boldsymbol{\beta}}=\underset{\boldsymbol{\beta} \in \mathbb{R}^{p},\boldsymbol{\beta}_{S^c}=\boldsymbol{0}}{\arg \min } \frac{1}{2}\boldsymbol{\beta}^{\mathrm{T}}\widehat{\boldsymbol{D}}_{1,b}\boldsymbol{\beta} -\boldsymbol{\beta}^{\mathrm{T}}\left\{\boldsymbol{z}_N+\left(\widehat{\boldsymbol{D}}_{1,b}-\widehat{\boldsymbol{D}}_{h}\right) \widehat{\boldsymbol{\beta}}_{h}\right\}+\lambda_N|\boldsymbol{\beta}|_1,
    \end{align*}
    where $\boldsymbol{\beta}_{S^c}$ denotes the subset vector with the coordinates of $\boldsymbol{\beta}$ in $S^c$.
    \item Then there exist sub-gradients $\widetilde{\boldsymbol{Z}}_S$ with $|\widetilde{\boldsymbol{Z}}_S|_{\infty} \leq 1$ such that
    \begin{equation}\label{A.18}
    \widehat{\boldsymbol{D}}_{1,b, S \times S} \widetilde{\boldsymbol{\beta}}_S-\left\{z_N+\left(\widehat{\boldsymbol{D}}_{1,b}-\widehat{\boldsymbol{D}}_{h}\right) \widehat{\boldsymbol{\beta}}_0\right\}_S+\lambda_N \widetilde{\boldsymbol{Z}}_S=0.
    \end{equation}
    \item Set
    \begin{equation}\label{A.19}
     \widetilde{\boldsymbol{Z}}_{S^c}=-\lambda_N^{-1}\left\{\left(\widehat{\boldsymbol{D}}_{1,b} \widetilde{\boldsymbol{\beta}}\right)_{S^c}-\left\{z_N+\left(\widehat{\boldsymbol{D}}_{1,b}-\widehat{\boldsymbol{D}}_{h}\right) \widehat{\boldsymbol{\beta}}_0\right\}_{S^c}\right\} . 
    \end{equation}
    For $j \in S^c$, let $Z_j =\left(\widetilde{\boldsymbol{Z}}_{S^c}\right)_j$. Checking that $\left|Z_j\right|<1$ for all $j \in \mathcal{S}^c$ ensures that there is a unique solution $\widetilde{\boldsymbol{\beta}}$ satisfying $\operatorname{supp}(\widetilde{\boldsymbol{\beta}}) \subseteq S$.
\end{itemize}
The following lemma gives the strict duality of $Z_j$ for all $j \in S^c$.
\vspace{0.5cm}

\begin{lemma}\label{lem4}
Under Assumption \ref{ass1}-\ref{ass6}, we have, with probability tending to one,
\begin{align*}
\left|Z_j\right| \leq v,
\end{align*}
uniformly for $j \in S^c$, for some $0<v<1$.
\end{lemma}
\begin{proof}
From \eqref{A.18}, we can decompose it as 
\begin{align*}
\widetilde{\boldsymbol{Z}}_{S^c}&=-\lambda_N^{-1}\left\{\left(\widehat{\boldsymbol{D}}_{1,b} \widetilde{\boldsymbol{\beta}}\right)_{S^c}-\left\{\boldsymbol{z}_N+\left(\widehat{\boldsymbol{D}}_{1,b}-\widehat{\boldsymbol{D}}_{h}\right) \widehat{\boldsymbol{\beta}}_0\right\}_{S^c}\right\} \\
&=-\lambda_N^{-1} \widehat{\boldsymbol{D}}_{1,b, S^c \times S} \widehat{\boldsymbol{D}}_{1,b, S \times S}^{-1}\left\{\boldsymbol{z}_N+\left(\widehat{\boldsymbol{D}}_{1,b}-\widehat{\boldsymbol{D}}_{h}\right) \widehat{\boldsymbol{\beta}}_0\right\}_S+\widehat{\boldsymbol{D}}_{1,b, S^c \times S} \widehat{\boldsymbol{D}}_{1,b, S \times S}^{-1} \widetilde{\boldsymbol{Z}}_S\\
&+\lambda_N^{-1}\left\{\boldsymbol{z}_N+\left(\widehat{\boldsymbol{D}}_{1,b}-\widehat{\boldsymbol{D}}_{h}\right) \widehat{\boldsymbol{\beta}}_0\right\}_{S^c}\\
&=\boldsymbol{M}_1+\boldsymbol{M}_2+\boldsymbol{M}_3+\boldsymbol{M}_4.
\end{align*}
where
\begin{align*}
\boldsymbol{M}_1=-\lambda_N^{-1} \widehat{\boldsymbol{D}}_{1,b, S^c \times S} \widehat{\boldsymbol{D}}_{1,b, S \times S}^{-1}\left(\left\{\boldsymbol{z}_N-\widehat{\boldsymbol{D}}_{h} \boldsymbol{\beta}^*\right\}_S+\left(\widehat{\boldsymbol{D}}_{h,S \times\{1, \ldots, p+1\}}-\widehat{\boldsymbol{D}}_{1,b, S \times\{1, \ldots, p+1\}}\right)\left(\boldsymbol{\beta}^*-\widehat{\boldsymbol{\beta}}_0\right)\right),
\end{align*}
\begin{align*}
\boldsymbol{M}_2=\widehat{\boldsymbol{D}}_{1,b, S^c \times S} \widehat{\boldsymbol{D}}_{1,b, S \times S}^{-1} \widetilde{\boldsymbol{Z}}_S, 
\quad 
\boldsymbol{M}_3=\lambda_N^{-1}\left\{\boldsymbol{z}_N-\widehat{\boldsymbol{D}}_{h} \boldsymbol{\beta}^*\right\}_{S^c},
\end{align*}
\begin{align*}
\boldsymbol{M}_4=\lambda_N^{-1}\left(\widehat{\boldsymbol{D}}_{h,S^c \times\{1, \ldots, p+1\}}-\widehat{\boldsymbol{D}}_{1,b, S^c \times\{1, \ldots, p+1\}}\right)\left(\boldsymbol{\beta}^*-\widehat{\boldsymbol{\beta}}_0\right) .
\end{align*}
Then we have $|\widetilde{\boldsymbol{Z}}_{S^c}|_{\infty} \leq |\boldsymbol{M}_1|_{\infty}+|\boldsymbol{M}_2|_{\infty}+|\boldsymbol{M}_3|_{\infty}+|\boldsymbol{M}_4|_{\infty}$. In the following, we will bound each $|\boldsymbol{M}_i|_{\infty}$ for $i=1,\ldots,4$.

For the first term, we have 
\begin{align*}
\left|\boldsymbol{M}_1\right|_{\infty}&\leq \left\|\widehat{\boldsymbol{D}}_{1,b, S^c \times S} \widehat{\boldsymbol{D}}_{1,b, S \times S}^{-1}\right\| _{\infty}\left|-\lambda_N^{-1}\left(\left\{\boldsymbol{z}_N-\widehat{\boldsymbol{D}}_{h} \boldsymbol{\beta}^*\right\}_S+\left(\widehat{\boldsymbol{D}}_{h,S \times\{1, \ldots, p+1\}}-\widehat{\boldsymbol{D}}_{1,b, S \times\{1, \ldots, p+1\}}\right)\left(\boldsymbol{\beta}^*-\widehat{\boldsymbol{\beta}}_0\right)\right)\right|_{\infty}.    
\end{align*}
By \eqref{A.17} in the proof of Theorem \ref{thm1}, let $C_0$ be sufficiently large, we have 
\begin{equation}\label{A.20}
\left|-\lambda_N^{-1}\left(\left\{\boldsymbol{z}_N-\widehat{\boldsymbol{D}}_{h} \boldsymbol{\beta}^*\right\}_S+\left(\widehat{\boldsymbol{D}}_{h,S \times\{1, \ldots, p+1\}}-\widehat{\boldsymbol{D}}_{1,b, S \times\{1, \ldots, p+1\}}\right)\left(\boldsymbol{\beta}^*-\widehat{\boldsymbol{\beta}}_0\right)\right)\right|_{\infty}\leq \alpha/12,   
\end{equation}
with probability to one. On the other hand, note that 
\begin{align*}
\widehat{\boldsymbol{D}}_{1,b, S{ }^c \times S} \widehat{\boldsymbol{D}}_{1,b, S \times S}^{-1} 
=&\left(\widehat{\boldsymbol{D}}_{1,b, S^c \times S}-\boldsymbol{I}_{S^c \times S}\right)\left(\widehat{\boldsymbol{D}}_{1,b, S \times S}^{-1}-\boldsymbol{I}_{S \times S}^{-1}\right)+\boldsymbol{I}_{S^c \times S}\left(\widehat{\boldsymbol{D}}_{1,b, S \times S}^{-1}-\boldsymbol{I}_{S \times S}^{-1}\right) \\
&+\left(\widehat{\boldsymbol{D}}_{1,b, S^c \times S}-\boldsymbol{I}_{S^c \times S}\right) \boldsymbol{I}_{S \times S}^{-1}+\boldsymbol{I}_{S^c \times S} \boldsymbol{I}_{S \times S}^{-1}. 
\end{align*}
By the proof of Lemma \ref{lem2}, with $b \asymp\left(s \log n / n\right)^{1 / 3}$ and $a_N \asymp \left(s \log N / n\right)^{1 / 2}$, we have 
\begin{align*}
\left\|\widehat{\boldsymbol{D}}_{1,b,S \times S}-\boldsymbol{I}_{S \times S}\right\|_{op}&=\mathcal{O}_{\mathbb{P}}\left(\sqrt{\frac{s\log n}{nb}}+a_N+b\right)  =\mathcal{O}_{\mathbb{P}}\left(\left(\frac{s \log n}{n}\right)^{1/3}\right).
\end{align*}
Similarly, we can get $\left\|\widehat{\boldsymbol{D}}_{1,b,S^c \times S}-\boldsymbol{I}_{S^c \times S}\right\|_{op}=\mathcal{O}_\mathbb{P}((s\log n/n)^{1/3})$. Due to the fact that 
\begin{align*}
\left\|\widehat{\boldsymbol{D}}_{1,b, S \times S}^{-1}-\boldsymbol{I}_{S \times S}^{-1}\right\|_{op}\leq \left\|\boldsymbol{I}_{S \times S}^{-1}\right\|_{op}^2\left\|\widehat{\boldsymbol{D}}_{1,b,S \times S}-\boldsymbol{I}_{S \times S}\right\|_{op}.
\end{align*}
Note that $\|\boldsymbol{A}\|_{\infty} \leq \sqrt{n}\|\boldsymbol{A}\|_{op}$ and $\|\boldsymbol{A}\|_{\infty} \leq n|\boldsymbol{A}|_{\infty}$ for any matrix $A \in \mathbb{R}^{m \times n}$, then under Assumption \ref{ass4}, we have 
\begin{equation}\label{A.21}
\left\|\left(\widehat{\boldsymbol{D}}_{1,b, S^c \times S}-\boldsymbol{I}_{S^c \times S}\right)\left(\widehat{\boldsymbol{D}}_{1,b, S \times S}^{-1}-\boldsymbol{I}_{S \times S}^{-1}\right)\right\|_{\infty}=\mathcal{O}_{\mathbb{P}}\left(s^{3/2}\left(\frac{s \log n}{n}\right)^{2/3}\right).
\end{equation}
Similarly, we have 
\begin{align}\label{A.22}
\left\|\boldsymbol{I}_{S^c \times S}\left(\widehat{\boldsymbol{D}}_{1,b, S \times S}^{-1}-\boldsymbol{I}_{S \times S}^{-1}\right)\right\|_{\infty}\leq s\left\|\boldsymbol{I}_{S^c \times S}\right\|_{op}\left\|\widehat{\boldsymbol{D}}_{1,b, S \times S}^{-1}-\boldsymbol{I}_{S \times S}^{-1}\right\|_{op}=\mathcal{O}_{\mathbb{P}}\left(s\left(\frac{s\log n}{n}\right)^{1/3}\right), 
\end{align}
and 
\begin{align}\label{A.23}
\left\|\left(\widehat{\boldsymbol{D}}_{1,b, S^c \times S}-\boldsymbol{I}_{S^c \times S}\right) \boldsymbol{I}_{S \times S}^{-1}\right\|_{\infty}\leq s\left\|\left(\widehat{\boldsymbol{D}}_{1,b, S^c \times S}-\boldsymbol{I}_{S^c \times S}\right) \right\|_{op}\left\|\boldsymbol{I}^{-1}_{S \times S}\right\|_{op}=\mathcal{O}_{\mathbb{P}}\left(s\left(\frac{s\log n}{n}\right)^{1/3}\right).
\end{align}
Together with the assumption $s^3\log n/n=\mathcal{O}(1)$, so we can conclude that $\|\widehat{\boldsymbol{D}}_{1,b, S{ }^c \times S} \widehat{\boldsymbol{D}}_{1,b, S \times S}^{-1} \|_{\infty}=\mathcal{O}_\mathbb{P}(1)+\|\boldsymbol{I}_{S^c \times S} \boldsymbol{I}_{S \times S}^{-1}\|_{\infty} $. Under the assumption that $\|\boldsymbol{I}_{S^c \times S} \boldsymbol{I}_{S \times S}^{-1}\|_{\infty}\leq 1-\alpha$, we have that $\|\widehat{\boldsymbol{D}}_{1,b, S{ }^c \times S} \widehat{\boldsymbol{D}}_{1,b, S \times S}^{-1} \|_{\infty}\leq 1-\alpha/2$ with probability to one. Thus combine \eqref{A.20}-\eqref{A.23}, we have
\begin{equation}\label{A.24}
\left|\boldsymbol{M}_1\right|_{\infty}\leq \alpha/12.    
\end{equation}

For the second term, by the fact that $|\widetilde{\boldsymbol{Z}}_S|_{\infty}\leq 1$, we have 
\begin{equation}\label{A.25}
\left|\boldsymbol{M}_2\right|_{\infty}\leq \left\|\widehat{\boldsymbol{D}}_{1,b, S^c \times S} \widehat{\boldsymbol{D}}_{1,b, S \times S}^{-1}\right\| _{\infty} \left|\widetilde{\boldsymbol{Z}}_S\right|_{\infty} \leq 1-\alpha/2 .    
\end{equation}
For the last two terms, by \eqref{A.17} in the proof of Theorem \ref{thm1}, let $C_0$ be sufficient large, we have 
\begin{equation}\label{A.26}
\left|\boldsymbol{M}_3\right|_{\infty}\leq \mathcal{O}_\mathbb{P}(1)/C_0 \leq  \alpha/12 .  
\end{equation}
Similarly, we can obtain that
\begin{equation}\label{A.27}
\left|\boldsymbol{M}_4\right|_{\infty}\leq \mathcal{O}_\mathbb{P}(1)/C_0 \leq  \alpha/12 .  
\end{equation}

Combine \eqref{A.24}-\eqref{A.27}, we can conclude that for any $j \in S^c$, there exist some constant $v=1-\alpha/4<1$ such that
\begin{align*}
\left|Z_j\right| \leq v <1.
\end{align*}
This completes the proof.
\end{proof}

Now we are ready to prove Theorem \ref{thm3}.

\begin{proof}
By Lemma \ref{lem4}, uniformly for $j \in S^c$ and some $0<v<1$, we have $\left|Z_j\right| \leq v <1$ with probability approaching to one. By the PDW construction, we have $\widehat{\boldsymbol{\beta}}_1=\widetilde{\boldsymbol{\beta}}$ with probability approaching to one. Thus, we have 
\begin{equation}\label{A.28}
{\mathbb{P}}\left(\widehat{S}_1 \subseteq  S\right) \rightarrow 1. 
\end{equation}
For the second part of Theorem \ref{thm3}, note that $\widehat{\boldsymbol{\beta}}_1=\widetilde{\boldsymbol{\beta}}$ with probability approaching to one. We alternately verify the same results for $\widetilde{\boldsymbol{\beta}}$. Recall that 
\begin{equation}\label{A.29}
\widehat{\boldsymbol{D}}_{1,b, S \times S} \widetilde{\boldsymbol{\beta}}_S-\left\{z_N+\left(\widehat{\boldsymbol{D}}_{1,b}-\widehat{\boldsymbol{D}}_{h}\right) \widehat{\boldsymbol{\beta}}_0\right\}_S=-\lambda_N \widetilde{\boldsymbol{Z}}_S.  
\end{equation}
Rewrite \eqref{A.29} as 
\begin{align*}
-\lambda_N \widetilde{\boldsymbol{Z}}_S&=\boldsymbol{I}_{S \times S}\left(\widetilde{\boldsymbol{\beta}}_S-\boldsymbol{\beta}^*_S\right)+\left(\widehat{\boldsymbol{D}}_{1,b,S\times S}-\boldsymbol{I}_{S\times S}\right)\left(\widetilde{\boldsymbol{\beta}}_S-\boldsymbol{\beta}^*_S\right)+\widehat{\boldsymbol{D}}_{1,b,S\times S}\boldsymbol{\beta}^*_S-\left\{z_N+\left(\widehat{\boldsymbol{D}}_{1,b}-\widehat{\boldsymbol{D}}_{h}\right)\widehat{\boldsymbol{\beta}}_0\right\}_S.
\end{align*}
This implies that
\begin{align*}\label{A.30}
\widetilde{\boldsymbol{\beta}}_S-\boldsymbol{\beta}^*_S&=\boldsymbol{I}_{S \times S}^{-1}\left\{-\lambda_N\widetilde{\boldsymbol{Z}}_S-\left(\widehat{\boldsymbol{D}}_{1,b,S\times S}-\boldsymbol{I}_{S\times S}\right)\left(\widetilde{\boldsymbol{\beta}}_S-\boldsymbol{\beta}^*_S\right)\right.\\
&\left.-\left(\widehat{\boldsymbol{D}}_{1,b,S\times S}-\widehat{\boldsymbol{D}}_{h,S\times S}\right)\left(\boldsymbol{\beta}^*_S-\widehat{\boldsymbol{\beta}}_{0,S}\right)+\left(z_N-\widehat{\boldsymbol{D}}_{h}\boldsymbol{\beta}^*\right)_S\right\}\\
&=\boldsymbol{I}_{S \times S}^{-1} \left(\boldsymbol{N}_1+\boldsymbol{N}_2+\boldsymbol{N}_3+\boldsymbol{N}_4\right).
\end{align*}
So we have 
\begin{equation}\label{A.30}
\left|\widetilde{\boldsymbol{\beta}}_S-\boldsymbol{\beta}^*_S\right|_{\infty}\leq \left\|\boldsymbol{I}_{S \times S}^{-1}\right\|_{\infty}\left(\left|\boldsymbol{N}_1\right|_{\infty}+\left|\boldsymbol{N}_2\right|_{\infty}+\left|\boldsymbol{N}_3\right|_{\infty}+\left|\boldsymbol{N}_4\right|_{\infty}\right) . 
\end{equation}

Recall that $\widetilde{\boldsymbol{Z}}_S\leq 1$, we have 
\begin{equation}\label{A.31}
\left|\boldsymbol{N}_1\right|_{\infty}\leq \lambda_N.  
\end{equation}
Due to fact that $|\boldsymbol{N}_2|_{\infty}\leq |\boldsymbol{N}_2|_2\leq \|\widehat{\boldsymbol{D}}_{1,b,S\times S}-\boldsymbol{I}_{S\times S}\|_{op}|\widetilde{\boldsymbol{\beta}}_S-\boldsymbol{\beta}^*_S|_2$ and Lemma \ref{lem2}, we have 
\begin{align*}
\left|\boldsymbol{N}_2\right|_{\infty}\leq C_2\left(\frac{s\log n}{n}\right)^{1/3}\left|\widetilde{\boldsymbol{\beta}}_S-\boldsymbol{\beta}^*_S\right|_2,
\end{align*}
where $C_2$ is a constant. Together with the assumption $s^3\log n/n=o(1)$, we have 
\begin{equation}\label{A.32}
\left\|\boldsymbol{I}_{S \times S}^{-1}\right\|_{\infty}\left|\boldsymbol{N}_2\right|_{\infty}\leq \frac{1}{2}\left|\widetilde{\boldsymbol{\beta}}_S-\boldsymbol{\beta}^*_S\right|_{\infty},  
\end{equation}
with probability tending to one.
From the proof of Theorem \ref{thm1}, we can obtain that 
\begin{equation}\label{A.33}
\left|\boldsymbol{N}_3\right|_{\infty}\leq C_3a_N\eta, 
\end{equation}
where $C_3$ is a constant, and $\eta=\max\left\{\left(\frac{s\log n}{n}\right)^{1/3},\left(\frac{s\log N}{n}\right)^{1/2}\right\}$.

Furthermore, according to Proposition \ref{prop2}, we have 
\begin{equation}\label{A.34}
\left|\boldsymbol{N}_4\right|_{\infty}=\mathcal{O}_{\mathbb{P}}\left(\sqrt{\frac{\log N}{N}}+a_N\eta\right),
\end{equation}

Combine \eqref{A.30}-\eqref{A.34}, we can conclude that 
\begin{align*}
\left|\widetilde{\boldsymbol{\beta}}_S-\boldsymbol{\beta}^*_S\right|_{\infty}\leq C\left\|\boldsymbol{I}_{S \times S}^{-1}\right\|_{\infty}\left(\sqrt{\frac{\log N}{N}}+a_N\eta\right).
\end{align*}
Then the second part of Theorem \ref{thm3} follows from the above and together with the lower bound
condition on $\min_{j \in S}|\beta_j^*|$.
Then we complete the proof.
\end{proof}

We use Theorem \ref{thm2} and Theorem \ref{thm3} iteratively to prove Theorem \ref{thm4}.

\begin{proof}
By iteratively using Theorem \ref{thm2} and Theorem \ref{thm3}, we have
\begin{align*}
\left|\widetilde{\boldsymbol{\beta}}_{t,S}-\boldsymbol{\beta}^*_S\right|_{\infty}&\leq C\left\|\boldsymbol{I}_{S \times S}^{-1}\right\|_{\infty}\left(\sqrt{\frac{\log N}{N}}+\eta\left|\widehat{\boldsymbol{\beta}}_{t-1,S}-\boldsymbol{\beta}^*\right|_2\right)\\
&=C\left\|\boldsymbol{I}_{S \times S}^{-1}\right\|_{\infty}\left(\sqrt{\frac{\log N}{N}}+\max\left\{s^{5t/6}\left(\frac{\log n}{n}\right)^{(2t+3)/6},s^{t}\left(\frac{\log N}{n}\right)^{(t+1)/2}\right\}\right).
\end{align*}
Then we complete the proof.
\end{proof}

\section{Comparison to the Competitors}

We compare our DHSQR method with other three competitors, DREL, DPQR, and Avg-DC methods, whose definitions are provided in the simulation part of the main text.

\textbf{Space Aspects.} The space complexity of the DHSQR method is provided in Section 2.3 of the main text. Similarly, all other algorithms exhibit the same spatial complexity of $\mathcal{O}\left(np+p^2\right)$. While our algorithm demonstrates strong performance, it remains comparable to other methods in terms of space complexity. We prioritize efficient storage utilization, avoiding unnecessary memory overhead compared to competing algorithms.

\textbf{Computational Aspects.} Compared to the Avg-DC method, our approach eliminates the need to solve non-smoothness quantile optimization on each machine. Instead, each local machine only calculates and communicates a $p$-dimensional vector (rather than a $p \times p$ matrix), while the central machine performs linear optimization with a Lasso penalty using a coordinate descent algorithm. Our method incurs lower communication costs. For instance, DREL\citep{chen2020distributed} requires an additional round of communication for calculating and broadcasting the density function $\widehat{f}^{g, k}(0)$. Simulation results demonstrate that DHSQR outperforms DREL in terms of speed. Leveraging a Newton-type iteration, our method is a second-order algorithm, requiring fewer iterations to achieve convergence compared to gradient-based first-order algorithms like DPQR\citep{tan2022communication}.

\textbf{Theoretical Aspects.} In addition, we wish to reiterate the theoretical innovations of our method, DHSQR. Unlike the DREL method, we relax the homoscedasticity assumption of the error term. In contrast to the Avg-DC and DPQR methods, we provide theoretical guarantees to support recovery.

\begin{table}[H]
\centering
\caption{Comparison of theoretical properties of different methods.}
\begin{tabular}{cccc}
\hline
Method &Statistical convergence & Support recovery & Heterogeneity \\ \hline
DHSQR  & $\checkmark$           & $\checkmark$          &  $\checkmark$                \\
DREL   & $\checkmark$           & $\checkmark$    & $\times$                  \\
DPQR   & $\checkmark$           & $\times$    & $\checkmark$                  \\
Avg-DC & $\checkmark$           & $\times$      & $\times$                  \\ \hline
\end{tabular}
\end{table}

\section{Additional Experiments}\label{sec2}

\subsection{Real Data Analysis}
In this section, we employ the proposed DHSQR algorithm to analyze the drug sensitivity data of the Human Immunodeficiency Virus (HIV) \citep{rhee2003human, hu2021communication}. This data is sourced from the Stanford University HIV Drug Resistance Database (\url{http://hivdb.stanford.edu}). Efavirenz (EFV) is the preferred first-line antiretroviral drug for HIV and belongs to the category of selective non-nucleoside reverse transcriptase inhibitors (NNRTIs) for subtype 1 HIV. We investigate the impact of mutations at different positions on EFV drug sensitivity. After excluding some missing records, our initial dataset comprises $N=2046$ HIV isolates and $p=201$ viral mutation positions. We define the response variable $y$ for regression as the drug sensitivity of the samples, which quantifies the fold reduction in susceptibility of a single virus isolate compared to a control isolate \citep{hu2021communication}. The covariate $x$ indicates whether the virus has mutated at different positions ($x=1$ for mutation, $x=0$ for no mutation). Due to the heavy-tailed distribution of the response variable, we apply a logarithmic transformation, transforming it into $\log_{10} y$. Histograms of both the initial response variable and the transformed response variable are shown in Figure \ref{fig2.6}. Even after the transformation, the response variable remains non-normally distributed. Therefore, we employ quantile regression for data analysis, which is more robust than linear regression. Additionally, it's important to note that while the dataset is not large in size, sensitivity data are typically distributed across different hospitals in practical scenarios, making data aggregation challenging. In such cases, distributed quantile regression algorithms become a suitable choice.

\begin{figure}[H]
  \centering
  \includegraphics[width=0.3\textwidth]
   {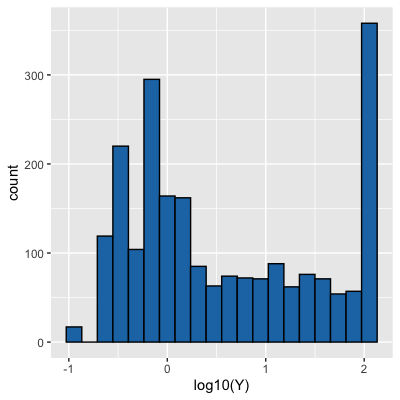}\label{y_plot}
  \includegraphics[width=0.3\textwidth]{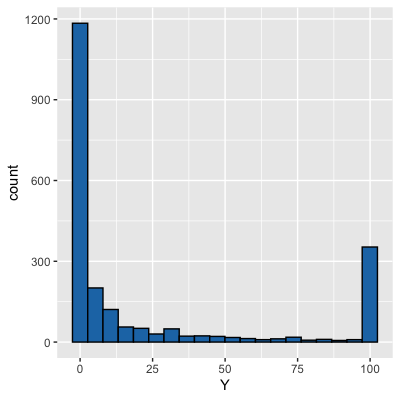}\label{logy_plot}
   \caption{The left figure represents the histogram of the initial drug sensitivity variable distribution, while the right figure represents the histogram of the drug sensitivity variable distribution after undergoing a logarithmic transformation.}\label{fig2.6}
\end{figure}  

In the experimental setup, we randomly selected a training dataset with a sample size of $N_{tr}=1500$, a validation dataset with a sample size of $N_{va}=300$ to select the optimal penalty parameter $\lambda$, and the remaining data served as the test dataset, which had a sample size of $N_{te}=246$. Typically, our interest lies in drug sensitivity at higher quantile levels since it is associated with stronger resistance \citep{hu2021communication}, enabling the development of better treatment strategies. Therefore, in the experiments, we chose quantile levels $\tau \in {0.5,0.75,0.9}$. We employed the same comparative methods as in Section \ref{sec:sim} to evaluate our algorithm. To assess the performance of each estimator, we computed the Predicted Quantile Error (PQE) on the prediction dataset, defined as
$$
\text{PQE}= \sum_{i=1}^{n_{te}}\rho_{\tau}(y_i-\widehat{y}_i) ,\quad i=1,\ldots,246,
$$
where $\widehat{y}_i$ represents the predictions made by each model. The results of the five different methods at quantile levels $\tau\in \{0.5,0.75,0.9\}$ are presented in Table \ref{tab.2.8}. From the results, it can be observed that our DHSQR estimator exhibits smaller prediction errors compared to other distributed estimators and is very close to the performance of the single machine estimator, Pooled DHSQR. This further underscores the method's excellent performance on HIV data. It is noteworthy that the performance of the DREL estimator is even worse than the one-step Avg-DC estimator, indicating the presence of heteroscedasticity in the data. Clearly, the DREL algorithm is not suitable for analyzing this type of data.

\begin{table}[H]
\footnotesize
\centering
\caption{The results of the Predicted Quantile Error (PQE) for different estimators at quantile levels $\tau\in\{0.5,0.75,0.9\}$ .}
\begin{tabular}{c|ccccc}
\hline
Quantile Level & DHSQR & Pooled DHSQR & DREL  & DPQR  & Avg-DC \\ \hline
$\tau=0.5$   & 0.205 & 0.201        & 0.241 & 0.226 & 0.225  \\ \hline
$\tau=0.75$  & 0.193 & 0.189        & 0.278 & 0.213 & 0.211  \\ \hline
$\tau=0.9$   & 0.176 & 0.172        & 0.329 & 0.186 & 0.232  \\ \hline
\end{tabular}
\label{tab.2.8}
\end{table}

Furthermore, to analyze the performance differences of various distributed algorithms in variable selection, we present the results of their predictions of non-zero coefficients at quantile levels $\tau\in{0.5,0.75,0.9}$ in Table \ref{tab.2.9}. We observed that, apart from the Avg-DC algorithm, all the multi-round communication-based distributed algorithms were able to select fewer variables than the total number, $p=201$. As the quantile level increases, the DHSQR algorithm selects an increasing number of non-zero coefficients, going from 24 to 35. This suggests that more virus positions have an impact on resistance at higher quantile levels.

\begin{table}[H]
\footnotesize
\centering
\caption{The results of the number of selected non-zero coefficients for different estimators at quantile levels $\tau\in\{0.5,0.75,0.9\}$.}
\begin{tabular}{c|ccccc}
\hline
Quantile Level & DHSQR & Pooled DHSQR & DREL  & DPQR  & Avg-DC \\ \hline
$\tau=0.5$   & 24 & 24        & 26 & 24 & 41  \\ \hline
$\tau=0.75$  & 29 & 29        & 31 & 29 & 63  \\ \hline
$\tau=0.9$   & 35 & 35        & 36 & 35 & 81  \\ \hline
\end{tabular}
\label{tab.2.9}
\end{table}

\subsection{Additional Simulated Experiments}
\subsubsection{Sensitivity Analysis for the Bandwidth}
In this section, we study the sensitivity of the scaling constant in the bandwidth of the DHSQR estimator. 
Recall that the global bandwidth is $h= c_h (s\log N/n)^{1/3}$ and the local bandwidth is $b= c_b (s\log n/n)^{1/3}$ with constant $c_h, c_b > 0$ being the scaling constant. 
We fix the quantile level $\tau = 0.5$, the number of local sample size $n = 500$  and the number total sample size $N = 20000$, 
We vary the constant $c_h, c_b$ from 1 to 10 and compute the $F_1$-score and the $\ell_2$-error of the DHSQR estimator under the heteroscedastic error case. Due to space limitations, we report the Normal noise case as an example. The results are shown in Table \ref{tab.5}.
\begin{table}[H]
\scriptsize
\caption{The $F_1$-score and $\ell_2$-error of the DHSQR under different choices of bandwidth constant $c_h, c_b$. (The standard deviation is given in parentheses.)}
\centering
\begin{tabular}{c|c|cc}
\hline
$c_b$ & $c_h$ & $F_1$-score  & $\ell_2$-error 
                       \\ 
                       \hline
1                      & 1                      & 0.991(0.032) & 0.077(0.032)   \\
2                      & 1                      & 0.951(0.068) & 0.260(0.081)   \\
5                      & 1                      & 0.996(0.023) & 1.856(1.084)   \\
10                     & 1                      & 1.000(0.02)  & 5.666(1.399)   \\ \hline
1                      & 2                      & 0.991(0.032) & 0.077(0.032)   \\
2                      & 2                      & 0.951(0.068) & 0.260(0.081)   \\
5                      & 2                      & 0.996(0.023) & 1.856(1.084)   \\
10                     & 2                      & 1.000(0.02)  & 5.666(1.399)   \\ \hline
1                      & 5                      & 0.991(0.032) & 0.077(0.032)   \\
2                      & 5                      & 0.951(0.068) & 0.260(0.081)   \\
5                      & 5                      & 0.996(0.023) & 1.856(1.084)   \\
10                     & 5                      & 1.000(0.02)  & 5.666(1.399)   \\ \hline
1                      & 10                     & 0.991(0.032) & 0.077(0.032)   \\
2                      & 10                     & 0.951(0.068) & 0.260(0.081)   \\
5                      & 10                     & 0.996(0.023) & 1.856(1.084)   \\
10                     & 10                     & 1.000(0.02)  & 5.666(1.399)   \\ \hline
\end{tabular}
\label{tab.5}
\end{table}
Table \ref{tab.5} shows that our estimator DHSQR is only sensitive to local bandwidth $b$ and not to global bandwidth.

\subsubsection{Computation Time Comparison}
We further study the computation efficiency of our proposed estimator. We fix the local sample size $n=500$ and vary the total sample size $N$. In Table \ref{tab.7}, we report the $F_1$ score, $\ell_2$ error, and computation time (per iteration) with different the total sample size $N$ under the heteroscedastic error case.
\begin{table}[H]
\centering
\caption{the $F_1$ score, $\ell_2$ error, and Time computation of five estimators under different total sample sizes $N$. (The standard deviation is given in parentheses.)}
\resizebox{0.6\textwidth}{!}{
\begin{tabular}{c|ccc|ccc}
\hline
\multicolumn{1}{c|}{$N$} & \multicolumn{3}{c|}{DHSQR}                       & \multicolumn{3}{c}{Pooled DHSQR}              \\ \cline{2-7} 
                     & $F_1$-score  & $\ell_2$-error & Time            & $F_1$-score  & $\ell_2$-error & Time         \\ \hline
5000                 & 0.961(0.121) & 0.126(0.018)   & 0.122(0.029)    & 0.851(0.132) & 0.113(0.032)   & 0.877(0.115) \\
10000                & 0.941(0.074) & 0.069(0.021)   & 0.130(0.024)    & 0.991(0.032) & 0.060(0.014)   & 1.671(0.080) \\
15000                & 0.984(0.043) & 0.068(0.030)   & 0.143(0.017)    & 0.981(0.072) & 0.055(0.015)   & 2.500(0.094) \\
20000                & 0.991(0.031) & 0.055(0.024)   & 0.151(0.020)    & 0.991(0.017) & 0.058(0.016)   & 3.393(0.140) \\ \hline
\multicolumn{1}{c|}{$N$} & \multicolumn{3}{c|}{DPQR}                       & \multicolumn{3}{c}{DREL}                     \\ \cline{2-7} 
                     & $F_1$-score  & $\ell_2$-error & Time            & $F_1$-score  & $\ell_2$-error & Time         \\ \hline
5000                 & 0.996(0.001) & 0.126(0.061)   & 0.193(0.033)    & 0.862(0.152) & 0.189(0.065)   & 1.006(0.095) \\
10000                & 0.996(0.002) & 0.097(0.056)   & 0.189(0.010)    & 0.972(0.621) & 0.131(0.072)   & 1.768(0.034) \\
15000                & 0.997(0.002) & 0.088(0.052)   & 0.208(0.020)    & 0.991(0.052) & 0.119(0.053)   & 2.622(0.166) \\
20000                & 1.000(0.001) & 0.081(0.050)   & 0.223(0.030)    & 0.992(0.053) & 0.117(0.042)   & 3.471(0.126) \\ \hline
\multicolumn{1}{c|}{$N$} & \multicolumn{3}{c|}{Avg-DC}                     &              &                &              \\ \cline{2-4}
                     & $F_1$-score  & $\ell_2$-error & Time            &              &                &              \\ \cline{1-4}
5000                 & 0.199(0.052) & 0.126(0.015)   & 31.780(2.884)   &              &                &              \\
10000                & 0.122(0.022) & 0.125(0.011)   & 59.576(5.125)   &              &                &              \\
15000                & 0.094(0.011) & 0.126(0.008)   & 93.553(7.594)   &              &                &              \\
20000                & 0.084(0.011) & 0.121(0.007)   & 117.129(11.709) &              &                &              \\ \cline{1-4}
\end{tabular}
}
\label{tab.7}
\end{table}
As is shown in Table \ref{tab.7}, We can see that our DHSQR has the fastest single iteration time, followed by DPQP, which maintains the same order of magnitude, then DREL and Pooled DHSQR, REL, and Pooled DHSQR maintain the same order of magnitude, and Avg-DC is the slowest.
\subsubsection{Effect of Sample Size and Local Sample Size}

\begin{table}[H]
\scriptsize
\caption{ The $F_1$-score and $\ell_2$-error of the DHSQR, Pooled DHSQR, DPQR, DREL, and Avg-DC estimator (at iteration 10) under different sample size $N$. Noises are generated from a normal distribution for the homoscedastic error case. The local sample size is fixed to $n = 500$.  The quantile level is fixed $\tau = 0.5$ and the iteration is fixed $T= 10$. (The standard deviation is given in parentheses.)}
\centering
\resizebox{0.8\textwidth}{!}{
\begin{tabular}{cc|ccccc}
\hline
\multicolumn{2}{c|}{$N$}                                             & 5000         & 10000        & 15000        & 20000        & 25000        \\ \hline
\multicolumn{1}{c|}{DHSQR}        & $F_1$-score    & 0.871(0.210)   & 0.929(0.072)   & 0.992(0.032)   & 1.000(0.001)      & 1.000(0.001)      \\
\multicolumn{1}{c|}{}                             & $\ell_2$-error & 0.139(0.021) & 0.080(0.053)  & 0.056(0.022) & 0.047(0.018) & 0.045(0.014) \\ \hline
\multicolumn{1}{c|}{Pooled DHSQR} & $F_1$-score    & 0.884(0.143)   & 0.957(0.041)         & 0.963(0.082)   & 0.974(0.057)  & 1.000(0.020) \\
\multicolumn{1}{c|}{}                             & $\ell_2$-error & 0.119(0.018) & 0.072(0.051) & 0.064(0.051) & 0.063(0.026) & 0.060(0.021)  \\ \hline
\multicolumn{1}{c|}{DPQR}        & $F_1$-score    & 1.000(0.019)   & 0.992(0.021)   & 1.000(0.001)      & 0.991(0.062)   & 0.992(0.063)   \\
\multicolumn{1}{c|}{}                             & $\ell_2$-error & 0.155(0.05)  & 0.115(0.083) & 0.116(0.059) & 0.096(0.05)  & 0.093(0.048) \\ \hline
\multicolumn{1}{c|}{DREL}        & $F_1$-score    & 0.869(0.081)   & 0.987(0.032)   & 1.000(0.001)      & 1.000(0.001)      & 1.000(0.001)      \\
\multicolumn{1}{c|}{}                             & $\ell_2$-error & 0.094(0.016) & 0.074(0.055) & 0.066(0.036) & 0.062(0.02)  & 0.059(0.016) \\ \hline
\multicolumn{1}{c|}{Avg-DC}      & $F_1$-score    & 0.191(0.054)   & 0.111(0.022)   & 0.082(0.011)   & 0.068(0.011)   & 0.069(0.012)   \\
\multicolumn{1}{c|}{}                             & $\ell_2$-error & 0.147(0.015) & 0.154(0.086) & 0.143(0.025) & 0.141(0.008) & 0.138(0.006) \\ \hline
\end{tabular}
}
\label{tab.1}
\end{table}
\begin{table}[H]
\scriptsize
\caption{ The $F_1$-score and $\ell_2$-error of the DHSQR, Pooled DHSQR, DPQR, DREL, and Avg-DC estimator (at iteration 10) under different sample size $N$. Noises are generated from a normal distribution for the heteroscedastic error case. The local sample size is fixed to $n = 500$.  The quantile level is fixed $\tau = 0.5$ and the iteration is fixed $T= 10$. (The standard deviation is given in parentheses.)}
\centering
\resizebox{0.8\textwidth}{!}{
\begin{tabular}{cc|ccccc}
\hline
\multicolumn{2}{c|}{$N$}                                             & 5000         & 10000        & 15000        & 20000        & 25000        \\ \hline
\multicolumn{1}{c|}{DHSQR}       & $F_1$-score    & 0.961(0.121)   & 0.942(0.073)   & 0.984(0.042)   & 0.990(0.031)   & 1.000(0.036)   \\
\multicolumn{1}{c|}{}                             & $\ell_2$-error & 0.126(0.018) & 0.069(0.021) & 0.068(0.031) & 0.055(0.024) & 0.052(0.026) \\ \hline
\multicolumn{1}{c|}{Pooled DHSQR} & $F_1$-score    & 0.851(0.132)   & 0.991(0.032)   & 0.981(0.072)   & 0.991(0.017)   & 1.000(0.019)   \\
\multicolumn{1}{c|}{}                             & $\ell_2$-error & 0.113(0.032) & 0.060(0.013) & 0.055(0.015) & 0.058(0.016) & 0.058(0.009) \\ \hline
\multicolumn{1}{c|}{DPQR}        & $F_1$-score    & 0.996(0.002)   & 0.996(0.001)   & 0.997(0.002)   & 1.000(0.001)   & 1.000(0.002)   \\
\multicolumn{1}{c|}{}                             & $\ell_2$-error & 0.126(0.061) & 0.097(0.056) & 0.088(0.052) & 0.081(0.051)  & 0.082(0.047) \\ \hline
\multicolumn{1}{c|}{DREL}        & $F_1$-score    & 0.862(0.152)   & 0.972(0.061)   & 0.991(0.052)   & 0.992(0.053)   & 0.996(0.036)   \\
\multicolumn{1}{c|}{}                             & $\ell_2$-error & 0.189(0.065) & 0.131(0.072)  & 0.119(0.052) & 0.117(0.042) & 0.108(0.041) \\ \hline
\multicolumn{1}{c|}{Avg-DC}      & $F_1$-score    & 0.199(0.052)   & 0.121(0.022)   & 0.094(0.011)   & 0.084(0.012)   & 0.082(0.009)   \\
\multicolumn{1}{c|}{}                             & $\ell_2$-error & 0.126(0.015) & 0.125(0.011) & 0.125(0.009) & 0.121(0.007) & 0.123(0.008) \\ \hline
\end{tabular}
}
\label{tab.2}
\end{table}
\begin{table}[H]
\scriptsize
\caption{ The $\ell_2$-error, precision, and recall of the DHSQR, Pooled DHSQR, DPQR, DREL, and Avg-DC estimator under different sample size $N$ and local sample size $n$. Noises are generated from a normal distribution for the homoscedastic error case. The quantile level is fixed $\tau = 0.5$ and the iteration is fixed $T= 10$. (The standard deviation is given in parentheses.)
}
\centering
\resizebox{\textwidth}{!}{
\begin{tabular}{cc|ccc|ccc|ccc}
\hline
\multicolumn{2}{c|}{$n$}                                           & \multicolumn{3}{c|}{200}                   & \multicolumn{3}{c|}{500}                   & \multicolumn{3}{c}{1000}                   \\ \hline
\multicolumn{2}{c|}{$N$}                                           & 5000         & 10000        & 20000        & 5000         & 10000        & 20000        & 5000         & 10000        & 20000        \\ \hline
\multicolumn{1}{c|}{DHSQR}        & Precision      & 0.901(0.147) & 0.947(0.119) & 0.969(0.077) & 0.935(0.136) & 0.901(0.159) & 0.984(0.079) & 0.913(0.177) & 0.940(0.086) & 0.961(0.101) \\
\multicolumn{1}{c|}{}                             & Recall         & 1.000(0.001) & 1.000(0.001) & 1.000(0.001) & 1.000(0.001) & 1.000(0.001) & 1.000(0.001) & 1.000(0.001) & 1.000(0.001) & 1.000(0.001) \\
\multicolumn{1}{c|}{}                             & $\ell_2$-error & 0.131(0.064) & 0.114(0.074) & 0.070(0.038) & 0.090(0.025) & 0.080(0.023) & 0.050(0.018) & 0.088(0.025) & 0.070(0.016) & 0.050(0.016) \\ \hline
\multicolumn{1}{c|}{Pooled DHSQR} & Precision      & 0.874(0.18)  & 0.943(0.095) & 0.997(0.02)  & 0.903(0.158) & 0.949(0.095) & 0.991(0.034) & 0.903(0.164) & 0.944(0.099) & 0.994(0.028) \\
\multicolumn{1}{c|}{}                             & Recall         & 1.000(0.001) & 1.000(0.001) & 1.000(0.001) & 1.000(0.001) & 1.000(0.001) & 1.000(0.001) & 1.000(0.001) & 1.000(0.001) & 1.000(0.001) \\
\multicolumn{1}{c|}{}                             & $\ell_2$-error & 0.086(0.019) & 0.066(0.016) & 0.046(0.009) & 0.086(0.019) & 0.066(0.016) & 0.046(0.009) & 0.087(0.019) & 0.066(0.016) & 0.045(0.009) \\ \hline
\multicolumn{1}{c|}{DPQR}        & Precision      & 0.929(0.141) & 0.956(0.103) & 0.985(0.057) & 0.935(0.139) & 0.923(0.131) & 0.970(0.080) & 0.957(0.085) & 0.918(0.118) & 0.988(0.058) \\
\multicolumn{1}{c|}{}                             & Recall         & 1.000(0.001) & 1.000(0.001) & 1.000(0.001) & 1.000(0.001) & 1.000(0.001) & 1.000(0.001) & 1.000(0.001) & 1.000(0.001) & 1.000(0.001) \\
\multicolumn{1}{c|}{}                             & $\ell_2$-error & 0.149(0.182) & 0.119(0.118) & 0.104(0.14)  & 0.098(0.045) & 0.092(0.040) & 0.086(0.068) & 0.101(0.044) & 0.077(0.029) & 0.063(0.024) \\ \hline
\multicolumn{1}{c|}{DREL}       & Precision      & 0.922(0.144) & 0.991(0.034) & 0.991(0.034) & 0.853(0.141) & 0.991(0.034) & 0.994(0.028) & 0.842(0.122) & 0.994(0.028) & 1.000(0.001) \\
\multicolumn{1}{c|}{}                             & Recall         & 1.000(0.001) & 1.000(0.001) & 1.000(0.001) & 1.000(0.001) & 0.058(0.009) & 1.000(0.001) & 1.000(0.001) & 1.000(0.001) & 1.000(0.001) \\
\multicolumn{1}{c|}{}                             & $\ell_2$-error & 0.094(0.017) & 0.060(0.011) & 0.044(0.008) & 0.085(0.015) & 0.061(0.011) & 0.040(0.007) & 0.084(0.015) & 0.060(0.011) & 0.040(0.007) \\ \hline
\multicolumn{1}{c|}{Avg-DC}     & Precision      & 0.043(0.005) & 0.029(0.002) & 0.025(0.001) & 0.106(0.028) & 0.058(0.009) & 0.038(0.006) & 0.196(0.075) & 0.115(0.025) & 0.066(0.011) \\
\multicolumn{1}{c|}{}                             & Recall         & 1.000(0.001) & 1.000(0.001) & 1.000(0.001) & 1.000(0.001) & 0.058(0.009) & 1.000(0.001) & 1.000(0.001) & 1.000(0.001) & 1.000(0.001) \\
\multicolumn{1}{c|}{}                             & $\ell_2$-error & 0.219(0.014) & 0.216(0.011) & 0.213(0.007) & 0.148(0.012) & 0.144(0.01)  & 0.140(0.008) & 0.112(0.013) & 0.108(0.009) & 0.102(0.007) \\ \hline
\end{tabular}
}
\label{tab.3}
\end{table}
\begin{table}[H]
\scriptsize
\caption{The $\ell_2$-error, precision, and recall of the DHSQR, Pooled DHSQR, DPQR, DREL, and Avg-DC estimator under different sample size $N$ and local sample size $n$. Noises are generated from a normal distribution for the heteroscedastic error case. The quantile level is fixed $\tau = 0.5$ and the iteration is fixed $T= 10$. (The standard deviation is given in parentheses.)}
\centering
\resizebox{\textwidth}{!}{
\begin{tabular}{cc|ccc|ccc|ccc}
\hline
\multicolumn{2}{c|}{$n$}                                           & \multicolumn{3}{c|}{200}                   & \multicolumn{3}{c|}{500}                   & \multicolumn{3}{c}{1000}                   \\ \hline
\multicolumn{2}{c|}{$N$}                                           & 5000         & 10000        & 20000        & 5000         & 10000        & 20000        & 5000         & 10000        & 20000        \\ \hline
\multicolumn{1}{c|}{DHSQR}        & Precision      & 0.901(0.160) & 0.941(0.118) & 0.979(0.061) & 0.943(0.10)  & 0.936(0.129) & 0.981(0.055) & 0.961(0.100) & 0.947(0.116) & 0.978(0.072) \\
\multicolumn{1}{c|}{}                             & Recall         & 1.000(0.001) & 1.000(0.001) & 1.000(0.001) & 1.000(0.001) & 1.000(0.001) & 1.000(0.001) & 1.000(0.001) & 1.000(0.001) & 1.000(0.001) \\
\multicolumn{1}{c|}{}                             & $\ell_2$-error & 0.180(0.133) & 0.157(0.107) & 0.125(0.075) & 0.092(0.025) & 0.081(0.027) & 0.059(0.018) & 0.088(0.022) & 0.065(0.02)  & 0.050(0.014) \\ \hline
\multicolumn{1}{c|}{Pooled DHSQR} & Precision      & 0.937(0.191) & 0.983(0.047) & 1.000(0.001) & 0.933(0.110) & 0.984(0.052) & 1.000(0.001) & 0.920(0.130) & 0.980(0.050) & 1.000(0.001) \\
\multicolumn{1}{c|}{}                             & Recall         & 1.000(0.001) & 1.000(0.001) & 1.000(0.001) & 1.000(0.001) & 1.000(0.001) & 1.000(0.001) & 1.000(0.001) & 1.000(0.001) & 1.000(0.001) \\
\multicolumn{1}{c|}{}                             & $\ell_2$-error & 0.083(0.018) & 0.058(0.015) & 0.043(0.011) & 0.083(0.018) & 0.058(0.015) & 0.043(0.011) & 0.083(0.018) & 0.057(0.016) & 0.043(0.011) \\ \hline
\multicolumn{1}{c|}{DPQR}        & Precision      & 0.969(0.060) & 0.936(0.128) & 0.969(0.084) & 0.970(0.060) & 0.946(0.087) & 0.978(0.073) & 0.950(0.110) & 0.943(0.105) & 0.994(0.028) \\
\multicolumn{1}{c|}{}                             & Recall         & 1.000(0.001) & 1.000(0.001) & 1.000(0.001) & 1.000(0.001) & 1.000(0.001) & 1.000(0.001) & 1.000(0.001) & 1.000(0.001) & 1.000(0.001) \\
\multicolumn{1}{c|}{}                             & $\ell_2$-error & 0.173(0.185) & 0.165(0.172) & 0.143(0.155) & 0.155(0.112) & 0.093(0.046) & 0.061(0.04)  & 0.135(0.087) & 0.079(0.036) & 0.057(0.025) \\ \hline
\multicolumn{1}{c|}{DREL}        & Precision      & 0.939(0.120) & 0.920(0.107) & 0.967(0.080)  & 0.950(0.100)   & 0.968(0.073) & 0.997(0.02)  & 0.921(0.110) & 0.984(0.052) & 0.997(0.020)  \\
\multicolumn{1}{c|}{}                             & Recall         & 1.000(0.001) & 1.000(0.001) & 1.000(0.001) & 1.000(0.001) & 1.000(0.001) & 1.000(0.001) & 1.000(0.001) & 1.000(0.001) & 1.000(0.001) \\
\multicolumn{1}{c|}{}                             & $\ell_2$-error & 0.192(0.110)  & 0.184(0.100)   & 0.159(0.074) & 0.146(0.037) & 0.127(0.06)  & 0.107(0.036) & 0.119(0.034) & 0.103(0.027) & 0.091(0.031) \\ \hline
\multicolumn{1}{c|}{Avg-DC}      & Precision      & 0.049(0.010) & 0.032(0.002) & 0.025(0.001) & 0.112(0.030) & 0.062(0.011) & 0.04(0.004)  & 0.221(0.080) & 0.125(0.028) & 0.071(0.013) \\
\multicolumn{1}{c|}{}                             & Recall         & 1.000(0.001) & 1.000(0.001) & 1.000(0.001) & 1.000(0.001) & 1.000(0.001) & 1.000(0.001) & 1.000(0.001) & 1.000(0.001) & 1.000(0.001) \\
\multicolumn{1}{c|}{}                             & $\ell_2$-error & 0.203(0.017) & 0.201(0.011) & 0.199(0.008) & 0.126(0.015) & 0.125(0.011) & 0.121(0.007) & 0.092(0.016) & 0.090(0.010) & 0.085(0.008) \\ \hline
\end{tabular}
 }
\label{tab.4}
\end{table}

\subsubsection{Additional Experiments for $\tau=\{0.3,0.7\}$}
In this section, we provide some additional experiment results using quantile level $\tau=0.3$ and $\tau=0.7$.
The results are reported in tables. 
The observations are similar to the case of $\tau = 0.5$ in Section 4.2. 
* indicates failure to converge.
\begin{table}[H]
\scriptsize
\caption{ The $\ell_2$-error, precision, and recall of the DHSQR, Pooled DHSQR, DPQR, DREL, and Avg-DC estimator under different sample size $N$ and local sample size $n$. Noises are generated from a normal distribution for the homoscedastic error case. The quantile level is fixed $\tau = 0.3$ and the iteration is fixed $T= 10$. (The standard deviation is given in parentheses.)}
\centering
\resizebox{\textwidth}{!}{
\begin{tabular}{cc|ccc|ccc|ccc}
\hline
\multicolumn{2}{c|}{$n$}                                           & \multicolumn{3}{c|}{200}                   & \multicolumn{3}{c|}{500}                   & \multicolumn{3}{c}{1000}                   \\ \hline
\multicolumn{2}{c|}{$N$}                                           & 5000         & 10000        & 20000        & 5000         & 10000        & 20000        & 5000         & 10000        & 20000        \\ \hline
\multicolumn{1}{c|}{DHSQR}        & Precision      & 0.951(0.100) & 0.961(0.072) & 0.967(0.082) 
& 0.968(0.100) & 0.966(0.066) & 0.995(0.035) 
& 0.941(0.141) & 0.961(0.108) & 1.000(0.001) \\
\multicolumn{1}{c|}{}                             & Recall         & 1.000(0.001) & 1.000(0.001) & 1.000(0.001) 
& 1.000(0.001) & 1.000(0.001) & 1.000(0.001) 
& 1.000(0.001) & 1.000(0.001) & 1.000(0.001) \\
\multicolumn{1}{c|}{}                             & $\ell_2$-error & 0.126(0.077) & 0.135(0.128) & 0.089(0.067) 
& 0.092(0.017) & 0.077(0.025) & 0.059(0.028) 
& 0.097(0.022) & 0.073(0.022) & 0.054(0.013) \\ \hline
\multicolumn{1}{c|}{Pooled DHSQR} & Precision      
& 0.952(0.102) & 0.974(0.055) & 0.997(0.020) 
& 0.947(0.101) & 0.955(0.083) & 0.997(0.020) 
& 0.939(0.112) & 0.964(0.071) & 0.997(0.020)  \\
\multicolumn{1}{c|}{}                             & Recall         & 1.000(0.001) & 1.000(0.001) & 1.000(0.001) 
& 1.000(0.001) & 1.000(0.001) & 1.000(0.001) 
& 1.000(0.001) & 1.000(0.001) & 1.000(0.001) \\
\multicolumn{1}{c|}{}                             & $\ell_2$-error & 0.092(0.017) & 0.072(0.016) & 0.049(0.012) 
& 0.091(0.016) & 0.073(0.017) & 0.049(0.012) 
& 0.091(0.017) & 0.073(0.022) & 0.049(0.012) \\ \hline
\multicolumn{1}{c|}{DPQR}        & Precision      
& 0.967(0.069) & 0.949(0.097) & 0.978(0.057) 
& 0.917(0.132) & 0.970(0.067) & 0.986(0.048) 
& 0.960(0.087) & 0.986(0.048) & 0.987(0.053) \\
\multicolumn{1}{c|}{}                             & Recall         
& 1.000(0.001) & 1.000(0.001) & 1.000(0.001) 
& 1.000(0.001) & 1.000(0.001) & 1.000(0.001) 
& 1.000(0.001) & 1.000(0.001) & 1.000(0.001) \\
\multicolumn{1}{c|}{}                             & $\ell_2$-error 
& 0.137(0.092) & 0.141(0.119) & 0.094(0.068) 
& 0.101(0.020) & 0.078(0.027) & 0.061(0.030)  
& 0.098(0.017) & 0.076(0.011) & 0.054(0.014) \\ \hline
\multicolumn{1}{c|}{DREL}        & Precision      
& 0.970(0.081) & 0.991(0.034) & 0.994(0.028) 
& 0.934(0.103) & 0.986(0.043) & 0.994(0.028) 
& 0.925(0.105) & 0.986(0.043) & 1.000(0.001) \\
\multicolumn{1}{c|}{}                             & Recall         
& 1.000(0.001) & 1.000(0.001) & 1.000(0.001) 
& 1.000(0.001) & 1.000(0.001) & 1.000(0.001) 
& 1.000(0.001) & 1.000(0.001) & 1.000(0.001) \\
\multicolumn{1}{c|}{}                             & $\ell_2$-error 
& 0.095(0.017) & 0.065(0.016) & 0.049(0.009) 
& 0.091(0.015) & 0.064(0.014) & 0.044(0.008) 
& 0.088(0.016) & 0.064(0.014) & 0.043(0.008) \\ \hline
\multicolumn{1}{c|}{Avg-DC}      & Precision      
& 0.046(0.006) & 0.031(0.002) & 0.025(0.001) 
& 0.117(0.028) & 0.061(0.011) & 0.039(0.004) 
& 0.235(0.086) & 0.118(0.036) & 0.065(0.012) \\
\multicolumn{1}{c|}{}                             & Recall         
& 1.000(0.001) & 1.000(0.001) & 1.000(0.001) 
& 1.000(0.001) & 1.000(0.001) & 1.000(0.001) 
& 1.000(0.001) & 1.000(0.001) & 1.000(0.001) \\
\multicolumn{1}{c|}{}                             & $\ell_2$-error 
& 0.232(0.016) & 0.230(0.011) & 0.227(0.008) 
& 0.156(0.013) & 0.153(0.010) & 0.148(0.007)  
& 0.117(0.014) & 0.114(0.009) & 0.107(0.007) \\ \hline
\end{tabular}
}
\label{tab.8}
\end{table}
\begin{table}[H]
\scriptsize
\caption{The $\ell_2$-error, precision, and recall of the DHSQR, Pooled DHSQR, DPQR, DREL, and Avg-DC estimator under different sample size $N$ and local sample size $n$. Noises are generated from a normal distribution for the heteroscedastic error case. The quantile level is fixed $\tau = 0.3$ and the iteration is fixed $T= 10$. (The standard deviation is given in parentheses.)}
\centering
\resizebox{\textwidth}{!}{
\begin{tabular}{cc|ccc|ccc|ccc}
\hline
\multicolumn{2}{c|}{$n$}                                           & \multicolumn{3}{c|}{200}                                              & \multicolumn{3}{c|}{500}                                              & \multicolumn{3}{c}{1000}                                              \\ \hline
\multicolumn{2}{c|}{$N$}                                           & 5000                  & 10000                 & 20000                 & 5000                  & 10000                 & 20000                 & 5000                  & 10000                 & 20000                 \\ \hline
\multicolumn{1}{c|}{DHSQR}        & Precision      & 0.966(0.072)          & 0.954(0.119)          & 0.971(0.084)          & 0.976(0.079)          & 0.971(0.058)          & 1.000(0.001)          & 0.931(0.119)          & 0.993(0.032)          & 0.993(0.032)          \\
\multicolumn{1}{c|}{}                             & Recall         & 1.000(0.030)          & 1.000(0.001)          & 1.000(0.001)          & 1.000(0.001)          & 1.000(0.001)          & 1.000(0.001)          & 1.000(0.001)          & 1.000(0.001)          & 1.000(0.001)          \\
\multicolumn{1}{c|}{}                             & $\ell_2$-error & 0.193(0.254)          & 0.143(0.167)          & 0.128(0.131)          & 0.075(0.032)          & 0.067(0.026)          & 0.051(0.025)          & 0.070(0.024)          & 0.053(0.017)          & 0.052(0.015)          \\ \hline
\multicolumn{1}{c|}{Pooled DHSQR} & Precision      & 0.979(0.057)          & 0.996(0.023)          & 1.000(0.001)          & 0.973(0.068)          & 0.996(0.023)          & 1.000(0.001)          & 0.979(0.057)          & 0.996(0.023)          & 1.000(0.001)          \\
\multicolumn{1}{c|}{}                             & Recall         & 1.000(0.001)          & 1.000(0.001)          & 1.000(0.001)          & 1.000(0.001)          & 1.000(0.001)          & 1.000(0.001)          & 1.000(0.001)          & 1.000(0.001)          & 1.000(0.001)          \\
\multicolumn{1}{c|}{}                             & $\ell_2$-error & 0.065(0.014)          & 0.057(0.021)          & 0.050(0.011)          & 0.062(0.016)          & 0.057(0.021)          & 0.049(0.011)          & 0.061(0.015)          & 0.057(0.021)          & 0.048(0.011)          \\ \hline
\multicolumn{1}{c|}{DPQR}        & Precision     
& 0.943(0.158)  & 0.948(0.109) & 0.977(0.098) 
& 0.937(0.086)  & 0.997(0.020) & 1.000(0.001) 
& 0.943(0.138)  & 0.978(0.057) & 1.000(0.001) \\
\multicolumn{1}{c|}{}                             & Recall         
&  1.000(0.001) & 1.000(0.001) & 1.000(0.001)
&  1.000(0.001) & 1.000(0.001  & 1.000(0.001) 
&  1.000(0.001) & 1.000(0.001) & 1.000(0.001) \\
\multicolumn{1}{c|}{}                             & $\ell_2$-error 
& 0.278(0.055) & 0.196(0.048) & 0.135(0.045) 
& 0.201(0.012) & 0.139(0.009) & 0.094(0.008)
& 0.114(0.016) & 0.081(0.033) & 0.121(0.007)  \\ \hline
\multicolumn{1}{c|}{DREL}        & Precision     & 
\textbf{*}         & 0.977(0.086)          & 0.974(0.082)          & 0.982(0.048)          & 1.000(0.001)          & 0.996(0.023)          & 0.996(0.023)          & 0.996(0.023)          & 1.000(0.001)          \\
\multicolumn{1}{c|}{}                             & Recall         & \textbf{*}          & 1.000(0.001)          & 1.000(0.001)          & 1.000(0.001)          & 1.000(0.001)          & 1.000(0.001)          & 1.000(0.001)          & 1.000(0.001)          & 1.000(0.001)          \\
\multicolumn{1}{c|}{}                             & $\ell_2$-error 
& \textbf{*}            & 0.201(0.173)          & 0.195(0.168)         
& 0.147(0.069)          & 0.124(0.065)          & 0.108(0.046)          
& 0.115(0.043)          & 0.110(0.037)          & 0.090(0.028)           \\ \hline
\multicolumn{1}{c|}{Avg-DC}      & Precision      
& 0.053(0.009)          & 0.034(0.003)          & 0.026(0.002)          
& 0.130(0.033)          & 0.068(0.012)          & 0.042(0.004)          
& 0.266(0.105)          & 0.133(0.038)          & 0.072(0.013)          \\
\multicolumn{1}{c|}{}                             & Recall         
& 1.000(0.001)          & 1.000(0.001)          & 1.000(0.001)          
& 1.000(0.001)          & 1.000(0.001)          & 1.000(0.001)          
& 1.000(0.001)          & 1.000(0.001)          & 1.000(0.001)          \\
\multicolumn{1}{c|}{}                             & $\ell_2$-error 
& 0.175(0.012)          & 0.174(0.009)          & 0.170(0.008)          
& 0.111(0.010)          & 0.108(0.008)          & 0.103(0.005)          
& 0.086(0.011)          & 0.083(0.009)          & 0.077(0.005)          \\ \hline
\end{tabular}
}
\label{tab.9}
\end{table}

\begin{table}[H]
\scriptsize
\caption{ The $\ell_2$-error, precision, and recall of the DHSQR, Pooled DHSQR, DPQR, DREL, and Avg-DC estimator under different sample size $N$ and local sample size $n$. Noises are generated from a normal distribution for the homoscedastic error case. The quantile level is fixed $\tau = 0.7$ and the iteration is fixed $T= 10$. (The standard deviation is given in parentheses.)}
\centering
\resizebox{\textwidth}{!}{
\begin{tabular}{cc|ccc|ccc|ccc}
\hline
\multicolumn{2}{c|}{$n$}                                           & \multicolumn{3}{c|}{200}                   & \multicolumn{3}{c|}{500}                   & \multicolumn{3}{c}{1000}                   \\ \hline
\multicolumn{2}{c|}{$N$}                                           & 5000         & 10000        & 20000        & 5000         & 10000        & 20000        & 5000         & 10000        & 20000        \\ \hline
\multicolumn{1}{c|}{DHSQR}        & Precision      
& 0.943(0.105) & 0.951(0.118) & 0.978(0.072) 
& 0.981(0.055) & 0.962(0.096) & 0.994(0.028) 
& 0.956(0.11)  & 0.977(0.088) & 0.994(0.028) \\
\multicolumn{1}{c|}{}                             & Recall         
& 1.000(0.001) & 1.000(0.001) & 1.000(0.001) 
& 1.000(0.001) & 1.000(0.001) & 1.000(0.001) 
& 1.000(0.001) & 1.000(0.001) & 1.000(0.001) \\
\multicolumn{1}{c|}{}                             & $\ell_2$-error 
& 0.140(0.092)  & 0.107(0.053) & 0.084(0.063) 
& 0.095(0.021)  & 0.072(0.026) & 0.056(0.014) 
& 0.095(0.020)  & 0.067(0.022) & 0.053(0.011) \\ \hline
\multicolumn{1}{c|}{Pooled DHSQR} & Precision      
& 0.949(0.110) & 0.955(0.102) & 1.000(0.001) 
& 0.956(0.096) & 0.958(0.101) & 1.000(0.001) 
& 0.957(0.098) & 0.959(0.102) & 1.000(0.001) \\
\multicolumn{1}{c|}{}                             & Recall         
& 1.000(0.001) & 1.000(0.001) & 1.000(0.001) 
& 1.000(0.001) & 1.000(0.001) & 1.000(0.001) 
& 1.000(0.001) & 1.000(0.001) & 1.000(0.001) \\
\multicolumn{1}{c|}{}                             & $\ell_2$-error 
& 0.094(0.021) & 0.067(0.018) & 0.050(0.012) 
& 0.094(0.021) & 0.066(0.019) & 0.049(0.011) 
& 0.095(0.02)  & 0.065(0.018) & 0.050(0.012)  \\ \hline
\multicolumn{1}{c|}{DPQR}        & Precision      
& 0.997(0.020) & 1.000(0.001) & 1.000(0.001) 
& 0.994(0.028) & 0.997(0.020) & 1.000(0.001) 
& 0.985(0.057) & 0.986(0.048) & 1.000(0.001) \\
\multicolumn{1}{c|}{}                             & Recall        
& 1.000(0.001) & 1.000(0.001) & 1.000(0.001) 
& 1.000(0.001) & 1.000(0.001) & 1.000(0.001) 
& 1.000(0.001) & 1.000(0.001) & 1.000(0.001) \\
\multicolumn{1}{c|}{}                             & $\ell_2$-error 
& 0.201(0.018)  & 0.144(0.011) & 0.099(0.009)
& 0.140(0.015)  & 0.100(0.010) & 0.068(0.007) 
& 0.101(0.015)  & 0.074(0.013) & 0.051(0.010) \\ \hline
\multicolumn{1}{c|}{DREL}        & Precision      
& 0.951(0.083) & 0.997(0.020) & 0.997(0.020) 
& 0.896(0.129) & 0.997(0.020) & 1.000(0.001) 
& 0.901(0.107) & 0.997(0.020) & 1.000(0.001) \\
\multicolumn{1}{c|}{}                             & Recall        
& 1.000(0.001) & 1.000(0.001) & 1.000(0.001) 
& 1.000(0.001) & 1.000(0.001) & 1.000(0.001) 
& 1.000(0.001) & 1.000(0.001) & 1.000(0.001) \\
\multicolumn{1}{c|}{}                             & $\ell_2$-error 
& 0.098(0.021) & 0.063(0.015) & 0.050(0.011)  
& 0.094(0.019) & 0.062(0.014) & 0.044(0.008) 
& 0.09(0.018)  & 0.061(0.014) & 0.044(0.008) \\ \hline
\multicolumn{1}{c|}{Avg-DC}      & Precision      
& 0.048(0.006) & 0.031(0.002) & 0.025(0.001) 
& 0.105(0.028) & 0.063(0.011) & 0.038(0.005) 
& 0.218(0.069) & 0.116(0.029) & 0.065(0.013) \\
\multicolumn{1}{c|}{}                             & Recall         
& 1.000(0.001) & 1.000(0.001) & 1.000(0.001) 
& 1.000(0.001) & 1.000(0.001) & 1.000(0.001) 
& 1.000(0.001) & 1.000(0.001) & 1.000(0.001) \\
\multicolumn{1}{c|}{}                             & $\ell_2$-error 
& 0.234(0.015) & 0.231(0.011) & 0.226(0.008) 
& 0.155(0.014) & 0.151(0.011) & 0.147(0.008) 
& 0.119(0.015) & 0.112(0.01)  & 0.107(0.007) \\ \hline
\end{tabular}
\label{tab.10}
}
\end{table}
\begin{table}[H]
\scriptsize
\caption{The $\ell_2$-error, precision, and recall of the DHSQR, Pooled DHSQR, DPQR, DREL, and Avg-DC estimator under different sample size $N$ and local sample size $n$. Noises are generated from a normal distribution for the heteroscedastic error case. The quantile level is fixed $\tau = 0.7$ and the iteration is fixed $T= 10$. (The standard deviation is given in parentheses.)}
\centering
\resizebox{\textwidth}{!}{
\begin{tabular}{cc|ccc|ccc|ccc}
\hline
\multicolumn{2}{c|}{$n$}                                           & \multicolumn{3}{c|}{200}                   & \multicolumn{3}{c|}{500}                   & \multicolumn{3}{c}{1000}                   \\ \hline
\multicolumn{2}{c|}{$N$}                                           & 5000         & 10000        & 20000        & 5000         & 10000        & 20000        & 5000         & 10000        & 20000        \\ \hline
\multicolumn{1}{c|}{DHSQR}        & Precision      
& 0.947(0.116) & 0.948(0.100) & 0.974(0.069) 
& 0.965(0.098) & 0.974(0.069) & 0.997(0.020)  
& 0.926(0.15)  & 0.989(0.081) & 0.997(0.020)  \\
\multicolumn{1}{c|}{}                             & Recall         
& 1.000(0.001) & 1.000(0.001) & 1.000(0.001) 
& 1.000(0.001) & 1.000(0.001) & 1.000(0.001) 
& 1.000(0.001) & 1.000(0.001) & 1.000(0.001) \\
\multicolumn{1}{c|}{}                             & $\ell_2$-error 
& 0.239(0.144) & 0.174(0.054) & 0.177(0.102) 
& 0.169(0.033) & 0.134(0.032) & 0.117(0.021) 
& 0.159(0.028) & 0.122(0.022) & 0.113(0.019) \\ \hline
\multicolumn{1}{c|}{Pooled DHSQR} & Precision      
& 0.964(0.102) & 0.985(0.057) & 1.000(0.001) 
& 0.959(0.107) & 0.983(0.064) & 1.000(0.001) 
& 0.974(0.069) & 0.985(0.057) & 1.000(0.001) \\
\multicolumn{1}{c|}{}                             & Recall         
& 1.000(0.001) & 1.000(0.001) & 1.000(0.001) 
& 1.000(0.001) & 1.000(0.001) & 1.000(0.001) 
& 1.000(0.001) & 1.000(0.001) & 1.000(0.001) \\
\multicolumn{1}{c|}{}                             & $\ell_2$-error 
& 0.152(0.024) & 0.118(0.017) & 0.107(0.017) 
& 0.152(0.025) & 0.118(0.017) & 0.106(0.017) 
& 0.153(0.024) & 0.118(0.018) & 0.107(0.017) \\ \hline
\multicolumn{1}{c|}{DPQR}        & Precision      
& 0.994(0.028) & 1.000(0.001) & 1.000(0.001) 
& 0.991(0.034) & 1.000(0.001) & 1.000(0.001)
& 0.980(0.050) & 1.000(0.001) & 1.000(0.001) \\
\multicolumn{1}{c|}{}                             & Recall         
& 1.000(0.001) & 1.000(0.001) & 1.000(0.001) 
& 1.000(0.001) & 1.000(0.001) & 1.000(0.001) 
& 1.000(0.001) & 1.000(0.001) & 1.000(0.001) \\
\multicolumn{1}{c|}{}                             & $\ell_2$-error 
& 0.329(0.013) & 0.238(0.007) & 0.171(0.007)  
& 0.291(0.014) & 0.210(0.007) & 0.152(0.009) 
& 0.222(0.012) & 0.159(0.008) & 0.118(0.011)  \\ \hline
\multicolumn{1}{c|}{DREL}        & Precision      
& 0.923(0.125) & 0.976(0.067) & 0.983(0.047) 
& 0.934(0.083) & 0.977(0.053) & 1.000(0.001) 
& 0.962(0.073) & 0.977(0.053) & 1.000(0.001) \\
\multicolumn{1}{c|}{}                             & Recall         
& 1.000(0.001) & 1.000(0.001) & 1.000(0.001) 
& 1.000(0.001) & 1.000(0.001) & 1.000(0.001) 
& 1.000(0.001) & 1.000(0.001) & 1.000(0.001) \\
\multicolumn{1}{c|}{}                             & $\ell_2$-error 
& 0.229(0.066) & 0.207(0.068) & 0.198(0.052) 
& 0.184(0.032) & 0.166(0.04)  & 0.153(0.036) 
& 0.159(0.023) & 0.15(0.024)  & 0.141(0.025) \\ \hline
\multicolumn{1}{c|}{Avg-DC}      & Precision      
& 0.05(0.006)  & 0.032(0.002) & 0.025(0.001) 
& 0.114(0.025) & 0.065(0.012) & 0.040(0.005)  
& 0.207(0.07)  & 0.117(0.029) & 0.066(0.011) \\
\multicolumn{1}{c|}{}                             & Recall         
& 1.000(0.001) & 1.000(0.001) & 1.000(0.001) 
& 1.000(0.001) & 1.000(0.001) & 1.000(0.001) 
& 1.000(0.001) & 1.000(0.001) & 1.000(0.001) \\
\multicolumn{1}{c|}{}                             & $\ell_2$-error 
& 0.278(0.02)  & 0.275(0.015) & 0.271(0.008) 
& 0.187(0.019) & 0.183(0.016) & 0.181(0.008) 
& 0.150(0.020) & 0.144(0.016) & 0.141(0.008) \\ \hline
\end{tabular}
\label{tab.11}
}
\end{table}

\subsubsection{Additional Experiments for the Decaying Sequence Setting of Nonzero Parameters}
In this section, we provide some additional experiment results using the decaying sequence setting of nonzero parameters. 
We consider the heteroscedastic error case with normal distribution, and $\tau = 0.7, n = 500, T = 10$. Here, we set the true parameter as 
$$ \boldsymbol{\beta}^{*} = (1,2^1,2^0,2^{-1},2^{-2},2^{-3},\boldsymbol{0}_{p-5})^\mathrm{T}. $$
Other settings align with those in Section \ref{sec:sim}. The average results from 100 replicates are summarized in Table \ref{tab.12}.
\begin{table}[h]
\centering
\caption{The $\ell_2$-error, precision, and recall of the DHSQR, pooled DHSQR, DPQR, DREL, and Avg-DC estimator under different sample size $N$ and local sample size $n$. Noises are generated from a normal distribution for the heteroscedastic error case. The quantile level is fixed $\tau=0.7$ and the iteration is fixed $T=10$. (The standard deviation is given in parentheses.)}
\centering
\resizebox{0.5\textwidth}{!}{
\begin{tabular}{cc|ccc}
\hline
\multicolumn{2}{c|}{N}                                              & 5000 & 10000        & 20000 \\ \hline
\multicolumn{1}{c|}{\multirow{2}{*}{DHSQR}}        & $F_1$-score    & 0.881(0.082)      & 0.893(0.037) &  0.904(0.029)     \\
\multicolumn{1}{c|}{}                              & $\ell_2$-error &0.159(0.021)      & 0.136(0.032) &  0.125(0.031)     \\ \hline
\multicolumn{1}{c|}{\multirow{2}{*}{Pooled DHSQR}} & $F_1$-score    & 0.897(0.034)     & 0.901(0.023) &  0.909(0.001)     \\
\multicolumn{1}{c|}{}                              & $\ell_2$-error &0.144(0.020)      & 0.120(0.017) &  0.102(0.016)     \\ \hline
\multicolumn{1}{c|}{\multirow{2}{*}{DREL}}         & $F_1$-score    & 0.892(0.045)    & 0.906(0.015) &  0.906(0.020)     \\
\multicolumn{1}{c|}{}                              & $\ell_2$-error &0.181(0.035)      & 0.158(0.034) &  0.152(0.039)     \\ \hline
\multicolumn{1}{c|}{\multirow{2}{*}{DPQR}}         & $F_1$-score    & 0.878(0.071)     & 0.909(0.009) &  0.909(0.001)     \\
\multicolumn{1}{c|}{}                              & $\ell_2$-error &0.176(0.011)      & 0.146(0.002) &  0.137(0.007)     \\ \hline
\multicolumn{1}{c|}{\multirow{2}{*}{Avg-DC}}       & $F_1$-score    & 0.179( 0.047)     & 0.105(0.023) & 0.066(0.007)      \\
\multicolumn{1}{c|}{}                              & $\ell_2$-error &0.184(0.018)      & 0.182(0.013) &  0.180(0.009)     \\ \hline
\end{tabular}
}
\label{tab.12}
\end{table}

As depicted in Table \ref{tab.12}, our DHSQR method consistently outperforms other methods across all sample sizes under the decaying sequence setting. DHSQR demonstrates superior performance in terms of $\ell_2$-error, indicating its capability to provide more accurate estimations compared to alternative methods. Additionally, the DHSQR method achieves an indistinguishable $F_1$-score when compared to other iteration methods, reaffirming its strength and reliability.

\end{document}